\documentclass{article}

\usepackage{microtype}
\usepackage{graphicx}
\usepackage{subfigure}
\usepackage{booktabs} % for professional tables

\usepackage{hyperref}
\usepackage{url}

\usepackage{amsmath}
\usepackage{cleveref}
\usepackage{wrapfig}
\usepackage{caption}
\usepackage{booktabs}
\usepackage{hyperref}
\usepackage{placeins}
\usepackage{array}
\newcolumntype{H}{>{\setbox0=\hbox\bgroup}c<{\egroup}@{}}

\usepackage{comment}
\usepackage{todonotes}

\usepackage{amsbsy}
\usepackage{color}
\usepackage{amsmath}
\usepackage{thm-restate}
\usepackage{Definitions}

\newcommand{\tIcal}{{\tilde \Ical}}
\newcommand{\bdelta}{{\bar \delta}}
\newcommand{\bD}{{\Dcal}}
\newcommand{\hA}{{\hat A}}

\newcommand{\ha}{{\hat a}}
\newcommand{\hT}{{\hat T}}
\newcommand{\hIcal}{{\hat \Ical}}

\newcommand{\hXi}{{\varpi}}

\newcommand{\tbeta}{{\tilde \beta}}
\newcommand{\tYcal}{{\tilde \Ycal}}

\newcommand{\Arm}{{\mathrm{A}}}
\newcommand{\Brm}{{\mathrm{B}}}
\newcommand{\Crm}{{\mathrm{C}}}
\newcommand{\Drm}{{\mathrm{D}}}

\newcommand{\Grm}{{\mathrm{G}}}
\newcommand{\Lrm}{{\mathrm{L}}}
\newcommand{\cX}{{\mathcal{X}}}
\newcommand{\cY}{{\mathcal{Y}}}
\newcommand{\cZ}{{\mathcal{Z}}}
\newcommand{\tZ}{{\tilde Z}}
\newcommand{\tZcal}{{\tilde \Zcal}}
\newcommand{\tf}{{\tilde f}}
\newcommand{\tphi}{\tilde \phi}
\newcommand{\hL}{{\hat L}}
\newcommand{\hQ}{{\hat Q}}
\newcommand{\tOcal}{{\tilde \Ocal}}

\usepackage{thm-restate}

\newcommand{\reve}{\color{black}}

\usepackage[accepted]{icml2023}

\usepackage{amsmath}
\usepackage{amssymb}
\usepackage{mathtools}
\usepackage{amsthm}

\icmltitlerunning{How Does Information Bottleneck Help Deep Learning?}

\begin{document}

\twocolumn[

\icmltitle{How Does Information Bottleneck Help Deep Learning?}



\icmlsetsymbol{equal}{*}

\begin{icmlauthorlist}
\icmlauthor{Kenji Kawaguchi}{equal,xxx}
\icmlauthor{Zhun Deng}{equal,yyy}
\icmlauthor{Xu Ji}{equal,comp}
\icmlauthor{Jiaoyang Huang}{sch}
\end{icmlauthorlist}

\icmlaffiliation{xxx}{NUS}
\icmlaffiliation{yyy}{Columbia University}
\icmlaffiliation{comp}{Mila}
\icmlaffiliation{sch}{University of Pennsylvania}

\icmlcorrespondingauthor{Kenji Kawaguchi}{kenji@nus.edu.sg}

\icmlkeywords{Machine Learning, ICML}

\vskip 0.3in
]



\printAffiliationsAndNotice{\icmlEqualContribution} % otherwise use the standard text.

\begin{abstract}
Numerous deep learning algorithms have been inspired by and understood via the notion of information bottleneck, where unnecessary information is (often implicitly) minimized while task-relevant information is maximized. However, a rigorous argument for justifying why it is desirable to control information bottlenecks has been elusive. In this paper, we provide the first rigorous learning theory for justifying the benefit of information bottleneck in deep learning by mathematically relating information bottleneck to generalization errors. Our theory proves that controlling information bottleneck is one way to control generalization errors in deep learning, although it is not the only or necessary way. We investigate the merit of our new mathematical findings with experiments across a range of architectures and learning settings. In many cases, generalization errors are shown to correlate with the degree of information bottleneck: i.e., the amount of the unnecessary information at hidden layers. This paper provides a theoretical foundation for current and future methods through the lens of information bottleneck. Our new generalization bounds scale with the degree of information bottleneck, unlike the previous bounds that scale with the number of parameters, VC dimension, Rademacher complexity, stability or robustness. Our code is publicly available at: \href{https://github.com/xu-ji/information-bottleneck}{https://github.com/xu-ji/information-bottleneck}
\end{abstract}

\section{Introduction}\label{sec:intro}

The information bottleneck principle \citep{tishby1999information,slonim2000document} has been a great concept in balancing the trade-off between the complexity of representation and the power of predicting. It is based on the notion of minimal sufficient statistics for extracting information about target $Y\in \cY$ into representation $Z=\phi(X)\in\cZ$ from input $X\in\cX$. An information bottleneck imposes regularization at representation $Z$ by minimizing the mutual information between $X$ and $Z$, $I(X;Z)$, while  maximizing  the mutual information between   $Y$ and   $Z$, $I(Y;Z)$. 

In practice $I(X; Z)$ is often minimized implicitly, e.g. as a result of stochastic gradient descent (SGD) or an architecture choice \citep{shwartz2017opening}. An explicit minimization of $I(X; Z)$ has been also adopted in the machine learning literature as a regularization technique  \citep{alemi2016deep,alemi2018fixing}, where the mutual information is either estimated by averaging log probabilities of latent representations over empirical samples  or replaced by a tractable upper bound \citep{kirsch2020unpacking,kolchinsky2017estimating,alemi2016deep}. More generally, the notion of bottlenecks on representation expressivity has been used in work on structural inductive biases~\citep{goyal2022inductive}.

\begin{figure}[t!]
    \centering
    \includegraphics[width=0.4\textwidth]{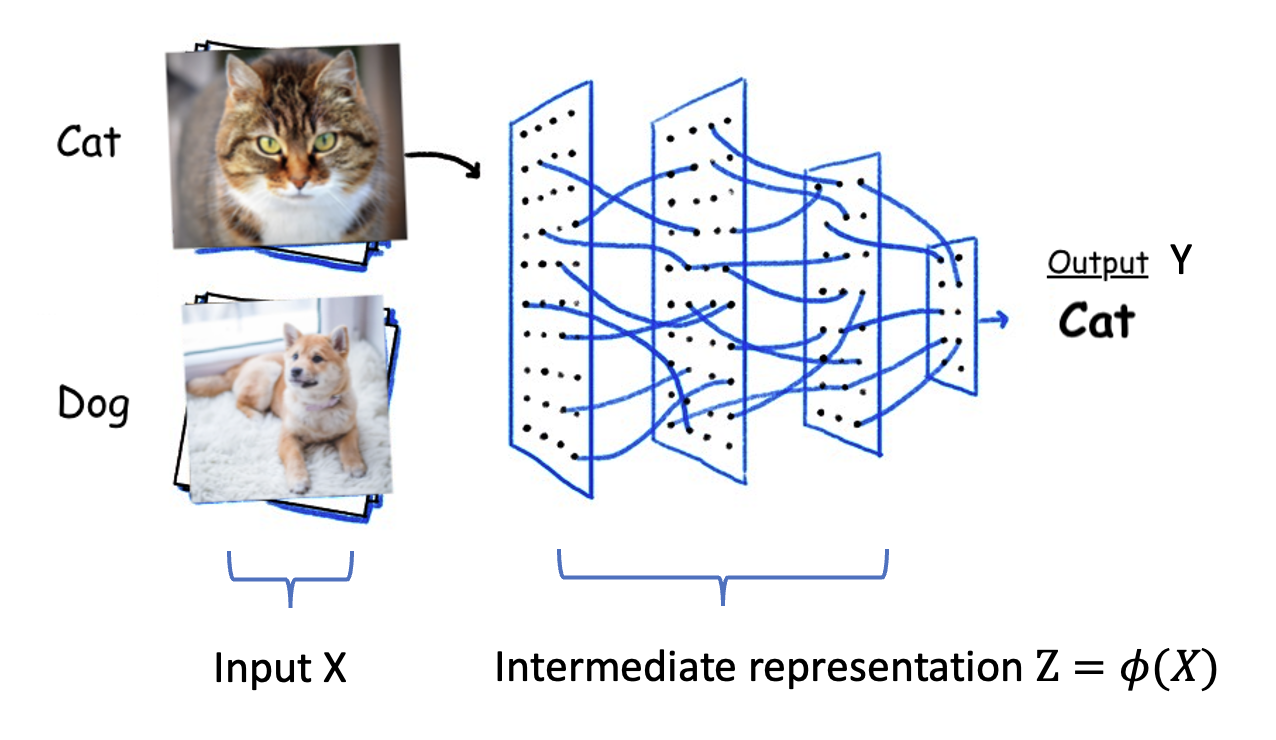}%
    \caption{Illustration of $X$, $Y$ and $Z$. This paper studies the relationship between performances of deep neural networks and the mutual information between $X$ and $Z$. Our theory proves that controlling this mutual information is one way to control performances in deep learning, although it is not the necessary way.}
    \label{fig:ill}
\end{figure}

Consequently, understanding the connection between the information bottleneck regularizer $I(X;Z)$ and the generalization ability of machine learning models has become an active area of research. Given its importance, \citet{shwartz2019representation} provided the following conjecture:
\begin{conjecture} \label{thm:shwartz}
(\textbf{Informal version} (\citealp{shwartz2019representation}))
With probability at least $1-\delta$ over the training data $s=\{(x_i,y_i)\}_{i=1}^n$ drawn from the same distribution as a random variable pair $(X,Y)$, for the generalization error $\Delta(s)=\EE_{X,Y}[\ell(f^s(X),Y)] - \frac{1}{n} \sum_{i=1}^n \ell(f^s(x_{i}), y_i)$, there is a bound obeys the following form:
 \begin{align} \label{eq:new:5}
\Delta(s)\leq  \sqrt{\frac{2^{I(X;Z_l^s)} + \log \frac{2}{\delta}}{2n}},
\end{align}
where $f^s$ is the full model obtained by training and $Z^s_l=\phi^s_l(X)$ is the output of the an intermediate $l$-layer encoder $\phi^s_l$ of the model, i.e. representation obtained after passing through the first $l$ layers.
\end{conjecture} 
However, this appealing conjecture cannot be applied to explain the success of information bottleneck principle in practice. First, the proof of the bound in this conjecture is \textit{incomplete}. More importantly, as pointed out by \citet{hafez2020sample}, there is a critical drawback in the formulation of this conjecture: \citet{shwartz2019representation} implicitly assumes the independence of $Z_l^s$ and $s$ in the arguments of this conjecture, which means that they treated the encoder $\phi_l^s$ as  fixed and independent of training data $s$.
Indeed, \citet{hafez2020sample} constructed a counterexample to show that the conjecture is  invalid  
when the encoder $\phi_l^s$ is  also learned with the training data $s$. This is because minimizing $I(X; Z^s_l)$ only does not sufficiently constrain the complexity of $\phi_l^s$, allowing it  to  arbitrarily overfit to the training data with a large generalization gap, in contradiction to the inequality \eqref{eq:new:5}.
In other words, when selecting the encoder's parameters is part of the learning problem, measuring compression via $I(X; Z^s_l)$ does not capture the degree of overfitting of the encoder's parameters.

Accordingly, as a first step towards proving a sample complexity bound via information bottleneck,  \Citet{hafez2020sample} focused on the \textit{input} layer and proved the following input compression bound for binary classification: if  $\Ycal=\{0,1\}$ and $\ell$ is the 0--1 loss, then
for any  $\delta > 0$, with probability at least $1 - \delta$ over the training dataset $s$, they roughly prove
\begin{align} \label{thm:hafez}
\Delta(s)=\tilde \Ocal\left(\sqrt{\frac{2^{H(X)}}{n}}\right),
\end{align} %\vskip -1em 
up to a factor $2^{1/\epsilon}$ for some constant $\epsilon$ satisfying  $
\epsilon =\Omega \big(\sqrt{(2^{6 H(X)/\epsilon} + \log (1/\delta) + 2)/n}\big).$
However, despite the popularity of the of information bottleneck principle and active usage in practice, this is still far from being a valid sample complexity bound, as noted by \citet{hafez2020sample}. 

To the best of our knowledge, in the current literature, much of the work on information bottleneck assumes its benefits, but no rigorous and valid sample complexity bounds have been proposed to justify why it is desirable to control information bottlenecks. In this paper, we make the first step to fill in this gap and provide an answer to the following open problem:
\begin{center}
``\textit{How does information bottleneck help deep learning from the perspective of statistical learning theory?}"
\end{center}

As our \textbf{\textit{first contribution}}, we resolve this open question by providing  novel and complete  proofs for end-to-end learning of  intermediate representations (Theorem~\ref{thm:2}). To the best of our knowledge, we provide the first rigorous generalization bound for information bottleneck in the case of learning representations, showing that simplicity in  both the representation and representation function are factors that support generalization.

As our \textbf{\textit{second contribution}}, an intermediate step and byproduct of our novel proof for Theorem~\ref{thm:2} not only completes the proof of Conjecture~\ref{thm:shwartz}, where $Z$ is treated as fixed random variable and independent of the training data $s$, it also significantly improve the previous bound in the conjecture. We show the  generalization error  roughly (with high probability) as 
 \begin{align*} 
\tilde \Ocal\left(\sqrt{\frac{I(X;Z^s_l| Y)+1}{n}}\right) \quad \text{as $n\rightarrow \infty$}.
\end{align*} 
This not only improves the numerator of the bound from an exponential dependence to linear dependence on mutual information, but also improves $I(X;Z^s_l)$ to a smaller quantity $I(X;Z^s_l| Y)$. More importantly, it is of independent interest and applicable to cases in transfer learning and unsupervised learning.

Finally, in Section~\ref{sec:experiments}, we consolidate our theoretical findings by comprehensive experiments on our bounds and related generalization prediction metrics, finding that empirical estimates of the main factors in our bounds are strong predictors of the generalization gap.

\section{Preliminaries}\label{sec:pre}
In this section, we describe the notations we use and settings we mainly consider.

\paragraph{Notation.}

We are given a training dataset $s=((x_{i},y_{i}))_{i=1}^n \sim \Pcal^{\otimes n}$ of  $n$ samples where $x_{i} \in \Xcal$ and $y_i \in \Ycal$ are i.i.d. drawn from a joint distribution $\Pcal$ over $\cX\times\cY$. We want to analyze the generalization  gap, i.e., the gap between the expected loss and the training loss, which is defined as
 \begin{align*}
\Delta(s):=\EE_{(X,Y) \sim \Pcal}[\ell(f^{s}(X),Y)] - \frac{1}{n} \sum_{i=1}^n \ell(f^{s}(x_{i}), y_i), 
\end{align*} 
where $\ell:\RR^{m_y} \times\Ycal \rightarrow \RR_{\ge 0}$ is a bounded per-sample loss, and $f^s:\cX\mapsto \RR^{m_y}$ represents a deep neural network learned with a given training dataset $s$. Here, $X$ and $Y$ are the corresponding random variables for $x_i$ and $y_i$ with $(X,Y) \sim \Pcal$.
 We use symbol $\circ$ to represent the composition of functions and the notation of $[D+1]=\{1,2,\dots,D+1\}$. 
We define the random variable of the output of the $l$-th layer by 
\begin{align} \label{eq:new:7}
Z_l^s=\phi_{l}^{s}(X), 
\end{align} 
where $\phi_l^s$ is the map for the first $l$ layer with with $\phi_{l}^{s}(x) \in\Zcal_{l}^s$. 
That is, for any layer index $l \in [D+1]$, we can decompose the neural network  $f^s$ by
 \begin{align}
f^{s}= g_{l} ^{s}\circ \phi_{l}^{s},
\end{align} 
where  
$g_{l}^s$ is the map for the rest of the layers after $l$ layers. For convenience, we refer to $\phi_l^s$  as the encoder and to $g_{l}^s$ as the decoder, though it is unnecessary to have an explicit structure of an encoder and a decoder.  Here, the case of $l=1$ corresponds to the input layer where $\phi_{1}^{s}(x)=x$ and $g_1^{s}(x)=f^{s}(x)$. The case of $l=D+1$ corresponds to the output layer  where $\phi_{D+1}^{s}(x)=f^s(x)$ and $g_{D+1}^{s}(q)=q$. 
$f^s$ can also be decomposed as 
$$f^s = h_{D+1} ^{s}\circ h_{D} ^{s}\circ h_{D-1} ^{s}\circ \cdots \circ h_1^s,$$ 
where $h_l^s$ represents the computation of the $l$-th layer; i.e., $\phi_{l}^s=h_{l} ^{s}\circ h_{l-1} ^{s}\circ \cdots \circ h_1^s$ and $g_{l} ^s=h_{D+1} ^{s}\circ h_{D} ^{s}\circ \cdots \circ h_{l+1}^s$. 

We use $\Acal$ to denote the learning algorithm that returns the output functions of each layer; i.e.,  $\Acal(s)=\{h_{l}^{s}\}_{l=1}^{D+1} $. Then, by taking a subset of the output coordinates, we define $\tilde  \Acal_l(s)=\{h_{k}^{s}\}_{k=1}^l$. Finally, by composing the outputs of $\tilde \Acal_l$, we define $\Acal_l(s)=h_{l} ^{s}\circ h_{l-1} ^{s}\circ \cdots \circ h^s_1=\phi_{l}^s  \in \Mcal_l$ (where $\{\Mcal_l\}_l$'s are families of functions). Define the maximum loss  
$$\Rcal(f^{s})= \sup_{(x,y) \in \Xcal \times \Ycal} \ell(f^{s}(x), y.)$$
We then define the random variable of the encoder of the $l$-th layer by
\vspace{-0.2cm}
 \begin{align} \label{eq:new:8}
\phi_{l} ^{s}=\Acal_l (s).
\end{align} 
\noindent{\textbf{Presumption.}} Following the previous work   \citep{shwartz2019representation}, we consider  the setting of $|\Xcal| < \infty$ and  $|\Mcal_l | < \infty$. This  is the natural setting with digital computers (e.g., using floating point). In this setting,  the mutual information that we consider are all finite and thus all bounds are nontrivial.  Similar restrictions are commonly considered in theory work involving mutual information to avoid the issue of infinite mutual information; for example, in \citet{xu2017information}, they consider countable hypothesis space. We follow the above setting in our main results, but we also show that those requirements can be relaxed: see more details in Section~\ref{sec:ext} and Appendix~\ref{app:10}.

\section{Main Results} \label{sec:analysis}

In this section, we establish sample complexity bounds to connect information bottlenecks and generalization errors. We start with completing and improving the previous results in the setting where the encoder $\phi_l^s$ is treated as fixed and independent of training data $s$ \citep{shwartz2019representation,hafez2020sample} in Section \ref{s:independent_encoder}. It will serve as an important intermediate step toward our final result, where we extend the argument to deal with the main case of our interest -- learning the encoder $\phi_l^s$ with  $s$ in Section \ref{s:dependent_encoder}.

\subsection{Encoder independent with the training data} \label{s:independent_encoder}

The following theorem  shows  that we can indeed optimize the control of expected loss by minimizing the conditional mutual information $I(X;Z_{l}^s|Y)$ and the training loss if the encoder $\phi_l^s$ is fixed and independent of training data $s$.\footnote{Here, we still use the superscript $s$ for various quantities to maintain notation consistency. Theorem \ref{thm:1} considers encoders that are independent of $s$, while Theorem \ref{thm:2} and the rest of this paper consider encoders that are dependent of $s$.} Even though this simpliefied case is just an intermediate step towards our final results, Theorem \ref{thm:1} is still useful and of independent interest. For example, it is applicable when the encoder is learned with data independent of $s$, such as in certain cases in transfer learning and unsupervised learning. 
\begin{theorem} \label{thm:1}
Let $l \in \{1,\dots,D\}$. Suppose that $\phi^s_l$ is fixed independently of the training dataset $s$. Then, for any $\delta>0$, with probability at least $1-\delta$ over training data $s$, the following holds:
\begin{align}  \label{eq:new:6}
\Delta(s)
\le  G_3^l  \sqrt{ \frac{ I(X;Z_{l}^s|Y)\ln(2) +\Gcal_2^l}{n}} + \frac{G_1^l(0)}{\sqrt{n}},
\end{align}
where  $G_1^{l}(0)=\tOcal(1)$, $\Gcal_2^l=\tOcal(1)$, and $G_3^{l}=\tOcal(1)$, as $n\rightarrow \infty$. The formula of $G_1^{l}(0)$, $\Gcal_2^l$ and $G_3^{l}$  are given in Appendix~\ref{app:12}.
\end{theorem}

Theorem~\ref{thm:1} rigorously completes the proof of  Conjecture~\ref{thm:shwartz},  with the significant improvements, which we will provide more detailed explanation in the following paragraph.

\noindent\textbf{Explanation of Theorem~\ref{thm:1}.} There are two significant improvements in our bound \eqref{eq:new:6} when compared with the previous bound \eqref{eq:new:5}. First, we reduce the exponential dependence to the linear dependence by replacing the exponential growth rate $2^{I(X; Z^s_l)}$ with the linear growth rate $I(X; Z^s_l)$. Second, we replace $I(X; Z^s_l)$ with $I(X; Z^s_l|Y)$, 
which is the expected mutual information between $X$ and $Z^s_l$ conditioned on $Y$.
To see why it is an improvement, notice that  $I(X; Z^s_l|Y)\le I(X; Z^s_l)$ since we can decompose $I(X; Z^s_l)$ into two components by using the chain rule as in  \citet{federici2020learning}: $I(X; Z^s_l) = I(X; Z^s_l|Y) +I(Y;Z^s_l).$
 Here, $I(X; Z^s_l|Y)\ge 0$ is the superfluous information that we want to minimize so as to maximize the predictive information $I(Y;Z^s_l)\ge 0$.
Therefore, the spirit of information bottleneck to regularize $I(X; Z^s_l)$ while maximizing $I(Y;Z^s_l)$ is an indirect way to regularize $I(X; Z^s_l|Y)$. Accordingly, instead of regularizing $I(X; Z^s_l)$, recent works have also start to consider regularizing $I(X; Z^s_l|Y)$  \citep{fischer2020conditional,federici2020learning,lee2021compressive}. In terms of theory, replacing $I(X; Z^s_l)$ with  $I(X; Z^s_l|Y)$  is \textit{\textbf{qualitatively significant}} because  $I(X; Z^s_l)$ cannot be zero while maintaining the label-relevant information $I(Y;Z^s_l)$, unlike  $I(X; Z^s_l|Y)$. 

For practical use of our bound, one can think about the case of fixed-feature learning, where the representation is learnt by other dataset independent of $s$. This is widely used in transfer learning and pre-training, where a $l$-layer representation is learnt from a large public available dataset (e.g. ImageNet) and we further use task specific and possibly private dataset $s$ to finetune extra few layers of the neural networks while fixing the representation.

\subsection{Encoder learned with the training data}\label{s:dependent_encoder}

In the previous section, we have proven an improved version of Conjecture~\ref{thm:shwartz} for the setting with a fixed encoder $\phi_l^s$. However, the typical usage of the information bottleneck principle is trying to minimizing  $I(X; Z_{l}^s)=I(X;\phi_l^s(X))$ over the parameters of the encoder $\phi_l^s$ along with a discriminative objective. Thus, to support the typical usage of the information bottleneck principle, we need to extend the results to the setting of learning encoder $\phi_l^s$ with $s$. In this setting, the bound in Conjecture \ref{thm:shwartz} is no longer valid as discussed in Section \ref{sec:intro}. However, the question about proving a sample complexity bound with the information bottleneck in the above typical setting is challenging and \textit{\textbf{remains open}}.

In this section, we present our main theorem that answers this open problem. Our result reconciles the information bottleneck regularizer $I_{}(X;Z_{l}^s|Y)$ with the mutual information of the encoder and the training dataset $I(\phi_{l}^{S}; S)$.

\begin{theorem}[\textbf{Main Theorem}] \label{thm:2}
Let $\bD\subseteq \{1,2,\dots,D+1\}$.    Then,  for any $\delta>0$, with probability at least $1-\delta$ over the training set $s$, the following generalization bound holds: 
\begin{align}  \label{eq:11} 
\Delta(s)
\le \min_{l \in \bD} Q_l,
\end{align}
where for $l \le D$,
$$
\scalebox{1.2}{%
\begin{small}
$
Q_l \hspace{-2pt} = G_3^l \sqrt{\frac{\left( I_{}(X;Z_{l}^s|Y)+I(\phi_{l}^{S}; S) \right) \ln(2)+ \widehat \Gcal_2^l}{n}} + \frac{G_1^l(\zeta)}{\sqrt{n}};
$
\end{small}}
$$ 
and for $l = D+1,$
$$Q_l=\Rcal(f^{s}) \sqrt{\frac{I_{}(\phi_{l}^{S}; S)\ln(2)  + \check \Gcal_2^l }{2n}},$$
Here,   $S \sim\Pcal^{\otimes n}$,     $G_1^l(\zeta)=\tOcal(\sqrt{I(\phi_{l}^{S}; S)+1})$,  $\widehat \Gcal_2^{l}=\tOcal(1)$, $\check \Gcal_2^l=\tOcal(1)$, and $G_3^{l}=\tOcal(1)$ as $n\rightarrow \infty$. The formulas of $G_1^l(\zeta)$, $\widehat \Gcal_2^{l}$, $\check \Gcal_2^l$, and $G_3^l$  are given in Appendix~\ref{app:12}.
\end{theorem}

In Theorem \ref{thm:2}, we  denote by  $S \sim\Pcal^{\otimes n} $ the random variable following the same distribution as that of the training dataset $s\sim\Pcal^{\otimes n}$. This notation is required here because the bound in Theorem \ref{thm:2} is equivalent to  $\PP_{s\sim \Pcal^{\otimes n}}[\Delta(s)
 \le g(I_{}(X;Z_{l}^s|Y),I(\phi_{l}^{S}; S))] \ge 1 -\delta$ for some function $g$. Here,  the probability $\PP_{s\sim \Pcal^{\otimes n}}$ is taken with respect to $s$. The instantiations   $s$ within  this  $\PP_{s\sim \Pcal^{\otimes n}}$ are used in   $\Delta(s)$ and $I(X;Z_{l}^s|Y)$, but not in  $I(\phi_{l}^{S}; S)$ since $I(\phi_{l}^{S}; S)$ only depends the distribution $\Pcal^{\otimes n}$ instead of the instantiations. That is, the reason why we need $S$ here is similar to the same reason why we use $j$ for the expression $\sum_{i=1}^n g_1(i, \sum_{j=1}^n g_2(j))$ (for some functions $g_1,g_2$) while we have   $ \sum_{j=1}^n g_2(j)=  \sum_{i=1}^n g_2(i)$ in terms of its value, where $i$ and $j$ correspond to  $s$ and $S$, respectively. 

In Theorem \ref{thm:2}, the randomness of $I_{}(X;Z_{l}^s|Y)$ and $I(\phi_{l}^{S}; S)$ are different. The mutual information $I_{}(X;Z_{l}^s|Y)$ is calculated over the randomness of the distribution of the input $X$ conditioning on $Y$, after fixing a realization of the training data $s$ (for each fixed draw of $s$ from $\PP_{s\sim \Pcal^{\otimes n}}$).  In contrast, the mutual information  $I(\phi_{l}^{S}; S)$ is computed over the randomness of the distribution of the training dataset  $S \sim\Pcal^{\otimes n} $.

\begin{remark} \label{remark:1}
One can consider the parameterization of the encoder as $\phi_l^{S}=\phi_{l,\theta_l^{S}}$ where $\theta^{S}_l$ is the parameter vector that is learned with $S$ and contains  all parameters of the layers up to $l$-th layer. In that case, Theorem \ref{thm:2} holds when replacing $\phi_{l}^{S}$ with $\theta^{S}_l$. 
\end{remark} 
\noindent\textbf{Explanation of Theorem~\ref{thm:2}.} Theorem \ref{thm:2}  provides the \textbf{\textit{first}} rigorous sample complexity bound for the information bottleneck in the setting of training the encoder $\phi^{s}_{l}$ with the same  training data $s$. Here, $Z_{l}^s$ is the random variable of the $l$-th layer's representation with dependence on the given  training dataset $s$, and   $D$ is the number of  all  layers, including the input layer and excluding the output layer; i.e.,   $Z_{1}^s$  is the input layer,  $Z_{D}^s$ is  the last hidden layer, and  $Z_{D+1}^s$ is  the output layer. Here,   $I(\phi_{l}^{S}; S)$ is measuring the effect of overfitting the encoder, which is necessary to avoid the counter-example  \citep[Example 3.1]{hafez2020sample}.

The main factor in the above theorem is $I_{}(X;Z_{l}^s|Y)+I(\phi_{l}^{S}; S)$. This term captures the novel relationship that has not been studied in any previous sample complexity bounds. Specifically, this captures the relationship between  ``how much information from the input $X$ the trained encoder $\phi_{l}^s$ retains, i.e., $I(X;Z_{l}^s|Y)$'' and ``how much information from the training dataset $S$ is used to train the encoder  $\phi_{l}^{S}$, i.e., $I(\phi_{l}^{S}; S)$''. 

Theorem \ref{thm:2} is applicable when the encoder is trained with $s$ and potentially additional data independent of $s$: e.g., supervised learning, semi-supervised learning, unsupervised learning, and transfer learning. For example, Theorem \ref{thm:2} captures the benefit of transfer learning in both terms of $I(X;Z_{l}^s|Y)$ and $I(\phi_{l}^{S}; S)$ since the encoder $\phi_{l}^{S}$ is expected to have less dependence on $S$ (target data) (for some $l\le D$) in transfer learning, which tends to decrease $I(\phi_{l}^{S}; S)$.

Finally, we note that in the formula of $\widehat \Gcal_2^{l}$,  we have a linear dependence on $H(Z_{l}^s|X_{},Y) \ln(2)$ (see Appendix~\ref{app:12}). However, we have $H(Z_{l}^s|X_{},Y)=0$ if the function $\phi^s_l$ is deterministic, which is the typical case for deep neural networks, because $\phi^s_l$ is the function used at inference or test time as opposed to training time (when dropout for example can be used). When the function $\phi^s_l$ is stochastic at test time, we have $H(Z_{l}^s|X,Y)\approx 0$ when the injected noise is small, and more generally  $H(Z_{l}^s|X,Y) = \Ocal(1)$ as $n\rightarrow \infty$.

\section{Extensions}\label{sec:ext}
Our results thus far focus on the case of $|\Xcal| < \infty$, which is already general enough to cover the realistic implementation on a computer and commonly considered in previous theory work \citep{shwartz2019representation}. Indeed, our presumption makes sure that the mutual information is finite and thus the bounds provided are non-trivial.  In the case of $|\Xcal|= \infty$, the mutual information can be infinite, and thus requires a separate treatment.  In this section, we show how to generalize our arguments to the case of $|\Xcal|= \infty$.

\subsection{Neural networks with ReLU activation functions} \label{subsec:ReLU}
First, we show that finite mutual information can be obtained in some cases even for the case of $|\cX|=\infty$. Specifically, the following proposition shows that  a  (deterministic) neural network can have finite mutual information with ReLU activations with continuous distributions.
\begin{proposition} \label{prop:3}
For a given neural network with ReLU activation functions, there are infinitely many continuous distributions over $\Xcal$ such that the corresponding $I(X,Z|Y)$ is finite.  
\end{proposition}

\subsection{Modification for valid bounds in the case of infinite mutual information} \label{sec:7}

The mutual information for the information bottleneck is finite  for many practical cases including the cases of discrete domains $\Xcal$ with any models  and of continuous domains  $\Xcal$ with stochastic models as well as the case in Proposition  \ref{prop:3} with ReLU. However, it is infinite for some special case, for example, of  continuous domains $\Xcal$ with   deterministic neural networks with certain types of injective activations such as sigmoid  (instead of ReLU) \citep{amjad2019learning}. This subsection demonstrates that our bounds can be modified to produce finite bounds even for any special cases of the 
mutual information being infinite. Our results (Theorems \ref{thm:1}--\ref{thm:2}   with Corollary \ref{coro:1}) also resolve  the known issue of arbitrariness of the mutual information  with different binning methods \citep{saxe2019information}. 

Consider an arbitrary (continuous or discrete) domain  $\Xcal$ and an arbitrary encoder  $\tphi_{l}^s$   such that  $\tphi_l^s(x) \in\tZcal_{l}^s $ and  the set $\tZcal_{l}^s$ is  potentially (uncountably or countably) infinite.
Define the corresponding model $\tf^{s}$  by  $\tf^{s}=g_{l} ^{s}\circ \tphi_{l}^s$ and $\tZ_{l}^s=\tphi_{l}^s \circ X$.  
 We formalize an arbitrary binning method $\Ecal_{l}[\tphi_{l}^s]$ of computing the mutual information \citep{chelombiev2018adaptive}
as follows: for  any $(l,\tphi_{l}^s)$, let   $\Ecal_{l}[\tphi_{l}^s]: \tZcal_{l}^{s}\rightarrow \Zcal_{l}^s \subseteq \tZcal_{l}^s $ be a~function such that $|\Zcal_{l}^{s}|<\infty$. Set  $\phi_{l}^{s} =\Ecal_{l}[\tphi_{l}^s] \circ\tphi_{l}^{s}$; i.e.,  it follows that $Z_l^{s}=\Ecal_{l}[\tphi_{l}^s]\circ \tZ_l^s $ and $f^s= g_{l} ^{s}\circ\Ecal_{l}[\tphi_{l}^s]\circ  \tphi_{l}^{s}$. Let  $\hQ_l$ and $\min_{l \in \bD} Q_l$ be the right-hand side of Eq.~\eqref{eq:new:6}  and Eq.~\eqref{eq:11} in Theorems \ref{thm:1}--\ref{thm:2} with this  choice of  encoder $\phi_{l}^s$; i.e.,  $\hQ_l$ and   $Q_l$ contain  $ I_{}(X;Z_{l}^s|Y)$ instead of $I(X;\tZ_{l}^s|Y)$. 
Here, $I_{}(X;Z_{l}^s|Y)$ is the mutual information computed by the binning method $\Ecal_{l}[\tphi_{l}^s]$ while $I(X;\tZ_{l}^s|Y)$ is the true mutual information of $\tf^s$. Let  $C_{l}$  be a nonnegative  real number such that $\PP(|\ell((g_{l} ^{s}\circ \tphi_{l}^s)(X),Y)- \ell(( g_{l} ^{s}\circ\Ecal_{l}[\tphi_{l}^s]\circ  \tphi_{l}^{s})(X),Y)|\le C_{l} )=1$. 

Corollary \ref{coro:1} shows that even when the mutual information  $I(X;\tZ_{l}^s|Y)$
of the original model $\tf^s$
is infinite, Theorems \ref{thm:1}--\ref{thm:2} provide the finite bounds on the \textit{original} model $\tf^s$
using the finite mutual information $I_{}(X;Z_{l}^s|Y)$ returned by a binning method  $\Ecal_{l}[\tphi_{l}^s]$: 
\setlist[description]{font=\normalfont}
\begin{corollary} \label{coro:1}
 Suppose that  $C_{l} < \infty$. Then, Theorems \ref{thm:1}--\ref{thm:2} hold true also when we
replace\begin{description} \vspace{-8pt}
\item[(Theorem \ref{thm:1})] 
 Eq.~\eqref{eq:new:6}  with $\Delta(s) \le \hQ_{l}+2C_{l}<\infty$, and,
\item[(Theorem \ref{thm:2})] 
 Eq.~\eqref{eq:11}  with $\Delta(s) \le \min_{l \in \bD} Q_l+2C_{l}<\infty$.
\end{description}
\end{corollary}

The assumption on the finiteness of $C_{l}$ is satisfied for common scenarios. For example, let $L$ be the Lipschitz constant of  the function $q\mapsto \ell(g_{l}^s(q),Y)$  w.r.t. some metric $d_{\Ecal}$ almost surely \citep{fazlyab2019efficient,latorre2019lipschitz,aziznejad2020deep,pauli2021training}.  Set  $\Ecal_{l}[\tphi_{l}^s]$ such  that the radius of each bin w.r.t. the metric $d_{\Ecal}$ is at most $\epsilon/\sqrt{nL^{2}}$ for some $\epsilon>0$.  We can then set $C_{l}= \epsilon/\sqrt{n}.$

 In Corollary \ref{coro:1}, the arbitrariness with binning methods is resolved: e.g.,  increasing the bin size $\epsilon$ can  decrease the mutual information, but it also increases the value of $C_{l}= \epsilon/\sqrt{n}$. Thus, there is always a tradeoff and we can't arbitrarily change values of our bounds by choosing different binning methods.  Similarly, for the case of infinite mutual information, we prove the validity of   general methods of computing mutual information, including those of injecting noises and kernel density estimations in Appendix~\ref{app:10}.

\section{Experiments}\label{sec:experiments}

We conduct empirical experiments to investigate the following questions: (i) Does the information bottleneck regularizer $I(X; Z_l^s)$ alone reliably predict generalization when the encoder $\phi_l^s$ is learned with $s$? (ii) Does the main factor in our bound in Theorem \ref{thm:2} $\min_{l \in [D]} I({S}; \theta_l^{S}) + I(X; Z_l^s | Y)$ with Remark \ref{remark:1} predict generalization more accurately than $I(X; Z_l^s)$ alone (or $I(X; Z_l^s | Y)$ alone)? (iii) How does varying layer $l$ within the network affect the values of $I({S}; \theta_l^{S})$ and $I(X; Z_l^s)$ and their predictive ability?

\subsection{On the Representation Compression Bound}\label{s:experiments1}
\begin{table}
\centering
\begin{tabular}{l c}
\toprule
~~~~~~~~~~~~~~~Pearson Correlation \\
 % Renamed from PCC to be consistent with other tables
\midrule
Num. params.  & -0.0294~~~~~~~~~~~~~  \\
$\prod_\lb \lVert \theta_\lb^{s} \rVert_F$ &  -0.0871~~~~~~~~~~~~~ \\
$\breve{I}(X; Z_l^s)$ &  0.3712~~~~~~~~~~~~~ \\
$\breve{I}(X; Z_l^s | Y)$   &   0.3842~~~~~~~~~~~~~ \\

$\breve{I}(S; \theta_{D+1}^\mathbf{S})$  & 0.0091~~~~~~~~~~~~~ \\ 
$\breve{I}(S ; \theta_l^{S})$    & 0.0211~~~~~~~~~~~~~ \\
$\tilde{I}(S; \theta_l^{S}) + \breve{I}(X; Z_l^s)$  & \textbf{0.3928}~~~~~~~~~~~~~ \\
$\tilde{I}(S; \theta_l^{S}) + \breve{I}(X; Z_l^s | Y)$  & \textbf{0.4130}~~~~~~~~~~~~~ \\
\bottomrule
\end{tabular}
%\captionsetup{margin={0.5cm,0cm}} 
\caption{Pearson correlation coefficient between metrics and the generalization gap in loss for constrained models trained for 5 class classification on 2D inputs. Positive values denote positive correlations. $\theta_\lb$ denotes parameters of layer $\lb$ and $\theta_l$ denotes parameters up to layer $l$.}
\label{tab:exp1_summary}
%\vspace{-1.5em}
\end{table}

\begin{table*}
\centering
\begin{tabular}{l H c c c c c c}
\toprule

& Layer summary 
& \multicolumn{2}{c}{Spearman corr.} & \multicolumn{2}{c}{Pearson corr.} & \multicolumn{2}{c}{Kendall corr.} \\
Generalization gap: & & Loss & Error & Loss & Error & Loss & Error  \\
\midrule
$\frac{1}{D}\sum_{l=1}^D \breve{I}(X; Z_l^s)$ & Mean & 0.8481 & 0.7410 & 0.2116 & 0.1831 & 0.6425 & 0.5436 \\
$\min_{l \in [D]} \breve{I}(X; Z_l^s)$ & Min & 0.7145 & 0.5602 & 0.7203 & 0.5719 & 0.4461 & 0.3404 \\
$\frac{1}{D}\sum_{l=1}^D \breve{I}(X; Z_l^s | Y)$ & Mean & 0.8481 & 0.7406 & 0.2140 & 0.1853 & 0.6427 & 0.5435 \\
$\min_{l \in [D]} \breve{I}(X; Z_l^s | Y)$ & Min & 0.7004 & 0.5434 & 0.7062 & 0.5560 & 0.4386 & 0.3305 \\
\midrule
$\breve{I}({S}; \theta_{D+1}^{S})$ & mean & 0.4688 & 0.3112 & 0.2512 & 0.0775 & 0.2121 & 0.1208 \\
$\min_{l \in [D]} \breve{I}({S}; \theta_{l}^{S}) + \breve{I}(X; Z_l^s | Y)$ & min &  0.8434 & 0.7313 & 0.8437 & 0.7195 & 0.6270 & 0.5332 \\
$\bar{I}({S}; \theta_{D+1}^{S})$ & mean & 0.5370 & 0.3800 & 0.2924 & 0.1218 & 0.2442 & 0.1526 \\
$\min_{l \in [D]}  \bar{I}({S}; \theta_l^{S}) + \breve{I}(X; Z_l^s | Y)$ & Min & \textbf{0.8632} & \textbf{0.7576} & \textbf{0.8511} & \textbf{0.7562} & \textbf{0.6626} & \textbf{0.5664} \\

\bottomrule
\end{tabular} 
\caption{Correlation results across metrics for CIFAR10 models. Each value is in [-1, 1] and $>0$ indicates positive correlation. Best metric highlighted. More results can be found in Appendix~\ref{app:exp2}.}
\label{tab:cifar10_main} 
\end{table*}
\begin{table*}
\centering
\begin{tabular}{l H c c c c c c}
\toprule

& Layer  
& \multicolumn{2}{c}{Spearman corr.} & \multicolumn{2}{c}{Pearson corr.} & \multicolumn{2}{c}{Kendall corr.} \\
Generalization gap: & & Loss & Error & Loss & Error & Loss & Error  \\
\midrule

$\frac{1}{D}\sum_{l=1}^D  \bar{I}(S; \theta_l^{S}) + \breve{I}(X; Z_l^s | Y)$ & Mean & 0.4429 & 0.2908 & 0.2783 & 0.1059 & 0.2349 & 0.1426 \\

$\max_{l \in [D]} \bar{I}(S; \theta_l^{S}) + \breve{I}(X; Z_l^s | Y)$ & Max & 0.5711 & 0.4204 & 0.2993 & 0.1311 & 0.2886 & 0.1945 \\

$\min_{l \in [D]}  \bar{I}(S; \theta_l^{S}) + \breve{I}(X; Z_l^s | Y)$ & Min & \textbf{0.8632} & \textbf{0.7576} & \textbf{0.8511} & \textbf{0.7562} & \textbf{0.6626} & \textbf{0.5664} \\

$\bar{I}(S; \theta_1^{S}) + \breve{I}(X; Z_1^s | Y)$ & $l=1$ & 0.6476 & 0.5292 & 0.1557 & 0.1331 & 0.4307 & 0.3504 \\

$\bar{I}(S; \theta_D^{S}) + \breve{I}(X; Z_D^s | Y)$ & $l=D$ & 0.5711 & 0.4204 & 0.2993 & 0.1311 & 0.2886 & 0.1945 \\
\bottomrule
\end{tabular}
\caption{Correlation results for $\bar{I}({S}; \theta_l^{S}) + \breve{I}(X; Z_l^s | Y)$ for CIFAR10 models across different layer summarization methods.}
\label{tab:cifar10_layers} 
%\vspace*{-1em}
\end{table*}
As discussed in Sections \ref{sec:analysis}, $I(X; Z_l^s)$ is generally not a reliable predictor of generalization because feature compression does not prevent the overfitting of the representation function's parameters. We investigate this further by designing a learning algorithm that trains models under various hyperparameter settings with the constraint that estimated $I(X; Z_l^s)$ is approximately constant. Following previous work such as \citet{galloway2022bounding}, we use correlation analysis to empirically evaluate how strongly metrics predict generalization.

The inference problem studied was 5 class classification on clustered 2D inputs (Fig.~\ref{fig:data}). 
The model architecture was a 5 layer MLP with deterministic weights and feature layer $l$ was fixed to the penultimate layer. 
Given training dataset $s$, each model $q_{\theta^s}$ was optimized with the cross-entropy loss $\min_{\theta^s} - \frac{1}{|s|} \sum_{(x, y) \in s} ( \log (1/k) \sum_{j=1}^k q_{\theta^s}(y | z^j) )$
s.t. $\hat{I}(X; Z_l^s) = \rho$, where features $z^j \sim q_{\theta^s}(Z_l^s | x)$, and $q_{\theta^s}(Z_l^s | x)$ is a multivariate Normal distribution with mean and variance computed by the MLP. Here, 
$$\hat{I}(X; Z_l^s)=\frac{1}{|s|} \sum_{(x, y) \in s}\frac{1}{k}\sum_{j=1}^k \log\frac{q_{\theta^s}(z^j | x)}{\frac{1}{|s|}\sum_{(x', y') \in s} q_{\theta^s}(z^j | x')}$$
is a Monte-Carlo sampling based estimator of $I(X; Z_l^s)$, and constraint $\rho$ was set to 1.5, approximately half the value of $\hat{I}(X; Z_l^s)$ attained without constraining $\hat{I}(X; Z_l^s)$. 
The neural network infers a distribution over a stochastic latent features so that $\hat{I}(X; Z_l^s)$ can be regularized and evaluated directly during training; in Section~\ref{s:experiments2} we consider the case of deterministic features without regularization of $\hat{I}(X; Z_l^s)$.
Whereas $\hat{I}(X; Z_l^s)$ is a sampling based estimator, we also use the   upper-bound based estimator: $\breve{I}(X; Z_l^s) = \frac{1}{|s|} \sum_{(x, y) \in s} \frac{1}{k} \sum_{j=1}^k ( \log q_{\theta^s}(z^j | x) - ( \frac{1}{|s|} \sum_{(x', y') \in s} \log q_{\theta^s}(z^j | x') ) )$. We define  $\hat{I}(X; Z_l^s|Y)$ and  $\breve{I}(X; Z_l^s|Y)$ accordingly by conditioning these quantities on $Y$: $\hat{I}(X; Z_l^s | Y) = \frac{1}{|s|}  \sum_{c \in C}  \sum_{(x, y) \in s_c} \frac{1}{k} \sum_{j=1}^k  \log \frac{q_{\theta^s}(z^j | x)}{\frac{1}{|s_c|} \sum_{(x', y') \in s_c} q_{\theta^s}(z^j | x')}$, and 
$\breve{I} ( X; Z_l^s | Y) \allowbreak = \frac{1}{|s|} \sum_{c \in C} \sum_{(x, y) \in s_c} \frac{1}{k} \sum_{j=1}^k ( \log q_{\theta^s}(\allowbreak z^j | x) - ( \frac{1}{|s_c|} \sum_{(x', y') \in s_c} \log q_{\theta^s}(z^j | x') ) )$. 

For the computation of $I({S}; \theta_{l}^{S})$, the learning algorithm is defined by the posterior distribution over network parameters $\PP(\theta_l^{S} | {S} = s)$, which was modelled using SWAG \citep{maddox2019simple,mandt2017stochastic}, chosen for its popularity and simplicity. We denote the estimator of $I({S}; \theta_l^{S})$ using SWAG by $\breve{I}({S}; \theta_l^{S})$, where datasets $S$ were drawn from the set of training datasets. To account for different scales of different estimation procedures, we tested rescaling $\breve{I}({S}; \theta_l^{S})$ by the average value of $\hat{I}(X; Z_l^s | Y)$, denoting rescaled values by $\tilde{I}({S}; \theta_l^{S})$ (see Appendix~\ref{app:exp1} for more details). 216 models were trained over varying architectures, weight decay rates, dataset draws, and random seeds.
Model parameters were optimized end-to-end using the reparameterization trick \citep{kingma2015variational} with dual gradient descent for MI constraints \citep{bertsekas2014constrained} (Appenix~\ref{app:exp1}). 
For each model, we measured the generalization gap between the test set and train set losses.
We found that combining model compression and representation compression yielded the best predictor of generalization overall, and that this outperformed using representation compression alone (Table~\ref{tab:exp1_summary},~\ref{tab:2_large}).
Additional experimental results on MNIST and Fashion MNIST datasets are given in Appendix~\ref{app:mnist}, showing that this conclusion also holds for stochastic feature networks in cases when $I(X; Z_l^s)$ is unconstrained.

\begin{figure}[h]
    \centering
    \includegraphics[width=0.24\textwidth]{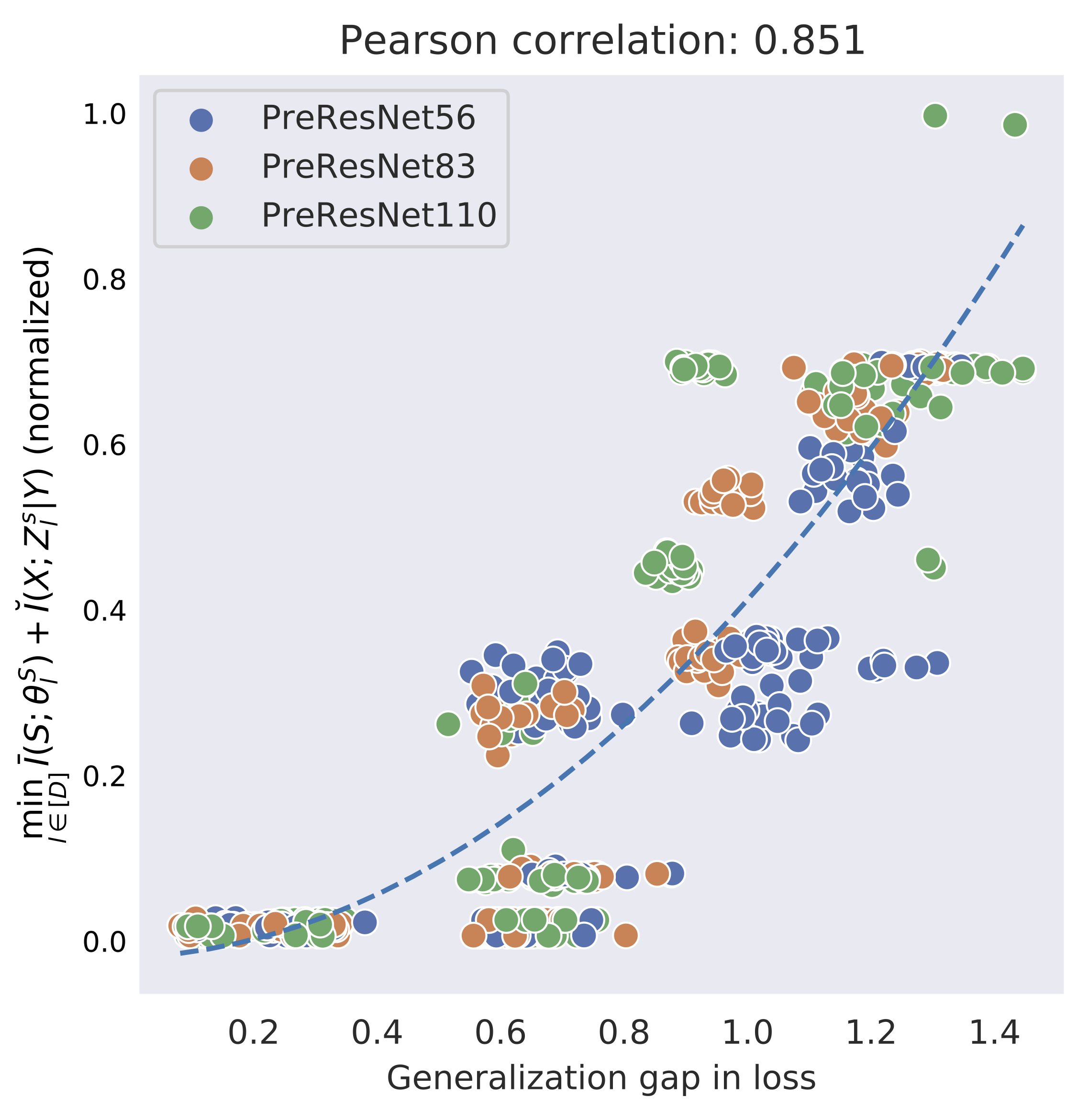}%
    \hspace{0.5em}
    \includegraphics[width=0.22\textwidth]{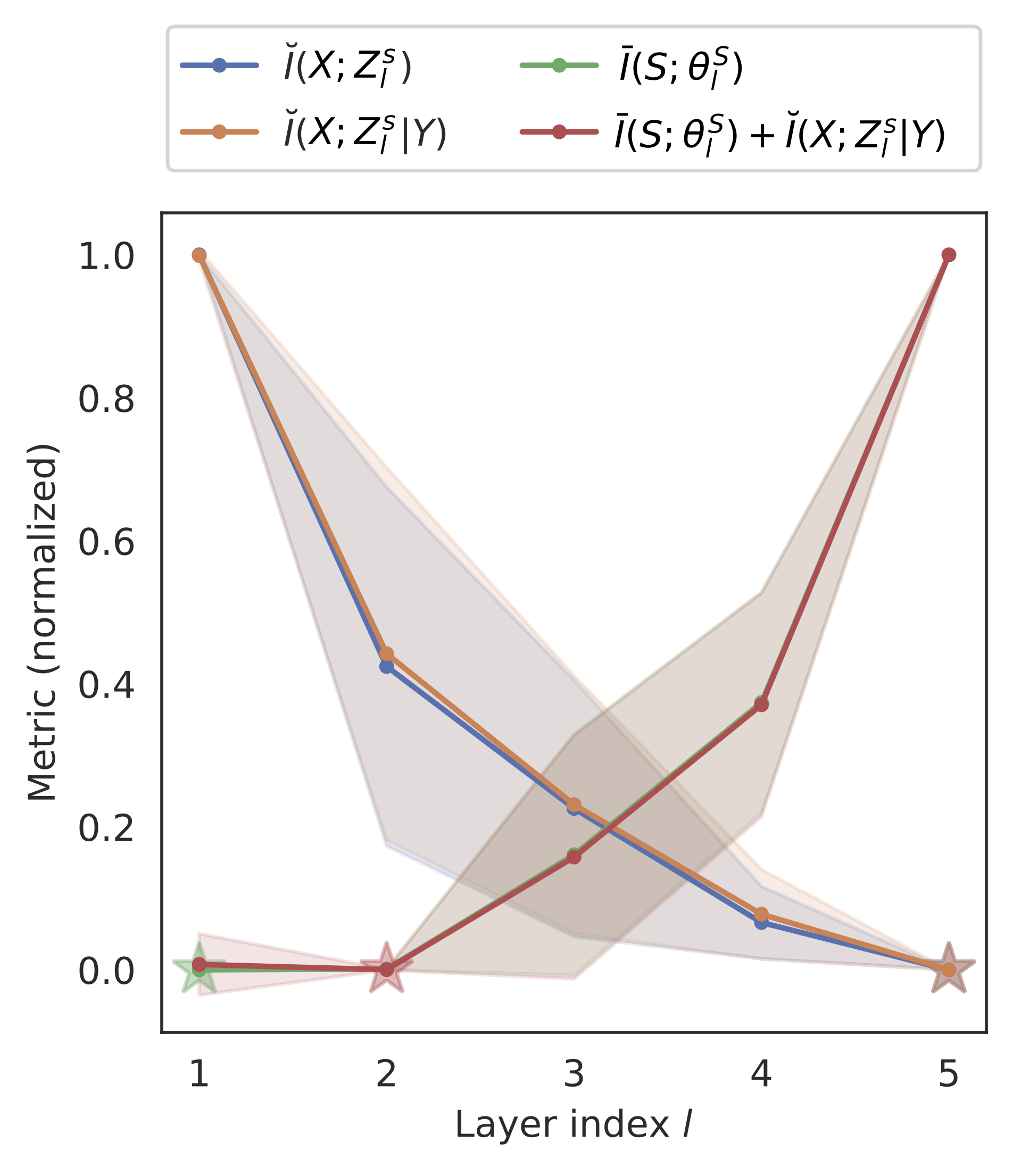}
    \caption{(\textbf{Left}) Results for $\min_{l \in [D]}$ ${\bar{I}({S}; \theta_l^{S}) + \breve{I}(X; Z_l^s | Y)}$ for unconstrained models trained on CIFAR10. Dashed line denotes best polynomial fit with degree 2.  (\textbf{Right}) Metrics averaged over models for each layer index. Star denotes minimum point for each metric. Values are normalized by subtracting the minimum and dividing by the range.}
    \label{fig:haha}
\end{figure}

%\vspace{-0.1in}
\subsection{Image Classification with DNNs}\label{s:experiments2}
To investigate a common setting, we trained 540 deep neural networks on CIFAR10 without explicitly constraining MI, over varying preactivation ResNet architectures \citep{he2016identity}, weight decay rates, batch sizes, dataset draws and random seeds.
To study representation compression by estimating MI with deterministically computed features, noise is customarily injected purely for analysis purposes \citep{saxe2019information}. We tested adaptive kernel density estimation (KDE) \citep{chelombiev2018adaptive}, which models the latent represenation of an input as a unimodal Gaussian centred at the deterministic feature, with variance $\sigma^2_l$ determined by scaling a base value according to maximum observed activation value in the layer. We also tested selecting $\sigma^2_l$ by maximum likelihood estimation (MLE) of observed features under the constraint that estimated MI decreases with layer, which follows from the information processing inequality. We report the results in this section for MLE and in Appendix~\ref{app:new:exp_results} for adaptive KDE. Representations were taken from $D=5$ layers in the model, ranging from the input to the output
 of the penultimate layer. 
 Again, SWAG was used to  model the posterior ${\PP(\theta_l^{S} | {S} = s)}$.
Since SWAG approximates the stationary distribution of SGD from a fixed initialization as a unimodal Gaussian \citep{mandt2017stochastic}, we also tested averaging over initializations to obtain a richer posterior model, and denote the estimator 
of MI from this model as $\bar{I}({S}; \theta_l^{S})$, 
defined in Appendix~\ref{app:exp2_metrics}. To construct 
multiple training datasets, we sampled 5 training sets of size 15K from the CIFAR10 training set, and each test set was the original 10K test set. 

The generalization gap was positively correlated with metrics measuring representation compression, but more correlated with metrics that combined both representation and model compression (Table~\ref{tab:cifar10_main}).
By increasing the value of layer index $l$ of the encoder, MI between the encoder and training dataset increased, while MI between the representation and input decreased (Fig.~\ref{fig:haha}), capturing a trade-off between these two measures of compression.
For selection of hyperparameters $\sigma_l^2$, MLE (Fig.~\ref{fig:haha}, Table~\ref{tab:cifar10_main},~\ref{tab:cifar10_layers},~\ref{tab:exp2_mle},~\ref{tab:exp2_mle_ends}) outperformed adaptive KDE  (Table~\ref{tab:exp2_kde},~\ref{tab:exp2_kde_ends}), however regardless of which scheme was used, 
the best metric that combined representation compression and model compression outperformed the best metric for representation compression or model compression individually. 
$\min_{l \in [D]} \bar{I}({S}; \theta_l^{S}) + \breve{I}(X; Z_l^s | Y)$ performed best overall (Fig.~\ref{fig:haha}).
Taking the minimum over layers (Theorem~\ref{thm:2})  outperformed other layer summarization methods (Table~\ref{tab:cifar10_layers}).

\FloatBarrier

\subsection{Feature binning experiments}
To test whether the importance of model compression would hold when $I(X; Z_l^s)$ was estimated by binning deterministic features into discrete categories (Section~\ref{sec:7}), we conducted toy clustering experiments (Appendix~\ref{s:experiments1}) on deterministic feature models. 
The binning implementation of \citet{saxe2019information} was reused to discretize the activity of each node into 10 buckets and the information bottleneck was implemented using the surrogate objective of \citet[Eq. 1]{kirsch2020unpacking}.
216 models were trained across MLP architectures and hyperparameters and SWAG was used to estimate the posterior. Model details and results are given in \cref{app:binning}.
We found that imposing the information bottleneck regularizer decreased the generalization prediction performance of feature compression metrics $\hat{I}(X; Z_l^s)$ and $\hat{I}(X; Z_l^s | Y)$, and combining model compression with feature compression metrics increased performance.

\section{Proof Sketch}\label{sec:proofsketch}
%\vspace{-4pt}

In this section, we further provide proof sketches of Theorems \ref{thm:1}-\ref{thm:2} for readers interested in having an overview of the proofs. The complete proofs are in Appendix~\ref{app:proofs}.

\subsection{Proof Sketch of Theorem \ref{thm:1}}
%\vspace{-4pt}

Fix  $l \in [D]$ and  $\phi^s_l$   independently of the training dataset $s$. Let $T$ be the standard \textit{typical set} of 
$Z_{l}^s$ in information theory. As we will see in the later sketch, we combine  deterministic decompositions and probabilistic bounds with respect to the randomness of new fresh samples $X$  and datasets  $s$.  The usages of probabilistic bounds for different sample spaces enable the exponential improvement over the previous bounds.

\noindent{\textbf{Step 1.}} Decompose the generalization gap into two terms as $\Delta(s)=A+B,$
where $A$ corresponds to the case of $Z_{l}^s \in T$, while  $B$ is for the case of $Z_{l}^s \notin T$. We will bound $A$ and $B$ separately.

\noindent{\textbf{Step 2.}}  By standard argument from information theory, we have $\PP(Z_{l}^s  \notin T\mid Y=y )\le \Ocal(\frac{1}{\sqrt{n}})$ 
where the probability is with respect to $X$, with which we can roughly argue \footnote{Although this requires a refinement of the standard argument. In Appendix~\ref{app:proofs}, we refine  the  argument using  the McDiarmid's inequality and a further decomposition of $B$.} that $B \le\Ocal(\frac{1}{\sqrt{n}}).$

\noindent{\textbf{Step 3.}} To bound $A$,  we argue that $A$ is bounded by a concentration gap of a special multinomial distribution over the elements of $T$, which is bounded roughly  as $A=\Ocal(\sqrt{(\ln |T|) /n})$ (with high probability with respect to $S$), by using a recent statistical result on multinomial distributions \citep[Lemma 3 \& Proposition 3]{kawaguchi2022robust}.
Then, the standard argument from information theory approximately bounds the size  of the typical set as $|T| \leq 2^{I(X;Z_{l}^s|Y)+C_{T}}$ for some $C_{T}>0$, approximately resulting in: with high probability
%\vspace{-4pt}
$$
A=\Ocal\left(\sqrt{\frac{\ln |T|}{n}}\right)=\tilde \Ocal\left(\sqrt{\frac{(I(X;Z_{l}^s|Y)+1) }{n}}\right).
$$ 

\noindent{\textbf{Step 4.}} Finally, By combining the above bounds on $A$ and $B$, we approximately conclude the result.

\subsection{Proof Sketch of Theorem \ref{thm:2}}
%\vspace{-4pt}

Based on the result  of Theorem \ref{thm:1} for  fixed  $l \in [D]$ and  $\phi^s_l$, we further generalize it for flexible $l$ and learnable $\phi^s_l$. Our proof is based on the tricky usage of probabilistic bounds for different sample spaces in Theorem \ref{thm:1}'s proof. 

\noindent{\textbf{Step 1.}} Let $l \in [D]$. We first find a hypothesis space $\Phi^{l}_\delta$  such that  $\PP(\phi_{l}^{S}\notin \Phi^{l}_\delta) \le \delta$ and $|\Phi_{\delta}^l| \leq 2^{I(\phi_{l}^{S}; S)+C_{\delta}}$ for some $C_{\delta}\ge0$. 
We then construct  the corresponding  hypothesis space $\Hcal$ by $\Hcal=\cup_{\phi_l \in \Phi_{\delta}^l} \Hcal_{\phi_l}$ where $\Hcal_{\phi_l}=\{g_l \circ \phi_l \mid  g_{l}: \Zcal_l \rightarrow \RR^{m_y}
\}$. 

\noindent{\textbf{Step 2.}} We then obtain the sample complexity bound for each $\Hcal_{\phi_l}$ (for each $\phi_l \in \Phi_{\delta}^l$) by 
using the result of  Theorem \ref{thm:1}. For each $\phi_l\in \Phi_{\delta}^l$ that is fixed independently of $s$;  i.e.,  $\PP(\forall f  \in \Hcal_{\phi_l}, \Bcal(f)\le  J_{l}(\delta))\ge 1-\delta$ where $\Bcal(f)=\EE_{X,Y}[\ell(f(X),Y)] - \frac{1}{n} \sum_{i=1}^n \ell(f(x_{i}), y_i)$ and $J_l(\delta)$ is  the right-hand side of Eq.~\eqref{eq:new:6}. Then, by taking union bound with the equal weighting over all $\phi_l \in \Phi_{\delta}^l$, we have $\PP(\forall f  \in \Hcal_{}, \Bcal(f)\le  J_{\delta,l})\ge 1-\delta,$
where $J_{\delta,l}=J_{l}(\delta/(2^{I(\phi_{l}^{S}; S)+C_{\delta}}))$. 

\noindent{\textbf{Step 3.}} We now want to show that this bound holds for  $\Bcal(f^{s})$ instead of $\Bcal(f)$ for $f \in \Hcal$. This is achieved if $f^{s} \in \Hcal$. Since  $\PP(f^{S}\in \Hcal)\ge1-\delta$ from the construction of $\Hcal$ and  $\PP(A \cap B) \le \PP(B)$ for any events $A$ and $B$, the following holds: 
%\vspace{-4pt}
\begin{align*}
&\PP(\Bcal(f^{S})\le J_{\delta,l})\\
%&\ge \PP_{}(f^{S}\in \Hcal \ \bigcap \ \Bcal(f^{S})\le J_{\delta,l})\\
& \ge\PP_{}(f^{S}\in \Hcal)\PP_{}(\Bcal(f^{S})\le J_{\delta,l}\mid f^{S}\in \Hcal)
\\ & \ge \PP(f^{S}\in \Hcal)(1-\delta)\ge1-2\delta.
\end{align*}  
%\vskip -1em
Therefore, 
by replacing $\delta$ with $\delta/2$, we have $\PP(\Bcal(f^{S})\le J_{\delta/2,l})\ge1-\delta.$

\noindent{\textbf{Step 4.}} For the case of $l=D+1$,
the proof is significantly simplified because an entire model is an encoder as $f=\phi_{D+1}$; i.e., we  replace the result of Theorem \ref{thm:1}  with  Hoeffding's inequality to conclude
that   $\PP(\forall f  \in \Hcal_{\phi_{D+1}} , \Bcal(f)\le  J_{D+1}(\delta))\ge1-\delta$ where $J_{D+1}(\delta)=\Rcal(f) \sqrt{(\ln(1/\delta))/(2n)}$. Using the same steps as the case of $l \in [D]$, we prove that $\PP( \Bcal(f^{S})\le J_{\delta/2,D+1})\ge1-\delta$, where  
$J_{\delta,D+1}=J_{D+1}(\delta/(2^{I(\phi_{l}^{S}; S)+C_{\delta}}))$. By taking union bounds over $l \in \bD\subseteq \{1,2,\dots,D+1\}$, we conclude $\PP(\forall l \in\bD , \Bcal(f^{S})\le  J_{\delta/(2|\bD|),l})=\PP(\Bcal(f^{S})\le \min_{l \in\bD} J_{\delta/(2|\bD|),l})\ge 1-\delta.$

\noindent{\textbf{Step 5.}} Organizing  the expression of $J_{\delta/(2|\bD|),l}$   yields  the right-hand side of \cref{eq:11}, which proves Theorem \ref{thm:2}.

%
%\vspace*{-0.25em}

%\vspace*{-0.25em}
\section{Related Works}
%\vspace*{-0.25em}
The implicit minimization of mutual information $I(X;Z)$ has been studied with the motivation of understanding why deep learning works through the lens of information bottlenecks \citep{shwartz2017opening}. The previous work assumes the benefit of minimizing $I(X;Z)$ and questioned whether the training of deep neural networks implicitly result in the minimization of $I(X;Z)$. In contrast, we studied the benefit of (implicitly or explicitly) controlling $I(X;Z)$. 

In this paper, we consider \textit{the generalization gap in deep learning} \citep{nagarajan2019uniform, zhang2021understanding, zhangdoes,kawaguchi2022robust,kawaguchi2022generalization,hu2022extended}, which is different from generalization gaps studied in the field of information bottleneck. In the field of information bottleneck, previous studies have analyzed a generalization gap between the true mutual information and its empirical estimate \citep{shamir2010learning, tishby2015deep} and the generalization gap on $\Ccal(q_\theta(Z|X), q(Y|Z))$ \citep{vera2018role}  where $\Ccal(p, q)$ is the cross entropy of $q$ relative to $p$, $q_\theta(Z|X)$ is a randomized encoder with learnable parameters $\theta$, and $q(Y|Z)$ is a simple count-based decoder \textit{with no learnable parameter}. Unlike ours, their bounds on $\Ccal(q_\theta(Z|X), q(Y|Z))$ scale with $\frac{|\Xcal| \ln n}{\sqrt{n}}+\frac{|\Zcal|}{\sqrt{n}}$ (due to their dependence on $\frac{1}{P_X(x_{\min})} \ge |\Xcal|$ and $|\Ucal|=|\Zcal|$ in their notation). This dependence on $|\Xcal|$ makes their bounds inapplicable as it requires the number of samples $n \gg |\Xcal|^2$. Here, the cross entropy $\Ccal(q_\theta(Z|X), q(Y|Z))$ studied in the previous work is also different from the cross-entropy loss of deep learning, $\Ccal(p(Y|X), q_\theta(Y|X))$, where $p(Y|X)$ is a target distribution and $q_\theta(Y|X)$ represents a deep neural network with learnable parameters $\theta$. Therefore, we could not rely on any of these previous results from the field of information bottleneck.

Another related yet different topic of information theory  is to use  $I(f^{S}; S)$  to compute generalization bounds \citep{xu2017information,bassily2018learners,hellstrom2020generalization,steinke2020reasoning}. However, these previous bounds are  \textbf{\textit{not}}  about information bottleneck as these do \textbf{\textit{not}} utilize  $I(X;Z_{l}^s|Y)$ (or $I(X;Z_{l}^s)$) and only uses $I(f^{S}; S)$, the mutual information between the training dataset $S$ and the \textbf{\textit{entire}} model $f^{S}=\phi^{S}_{D+1}$. Thus, the previous bounds cannot provide insights or justifications on the information bottleneck principle unlike our bounds. Moreover, in Section \ref{sec:experiments}, we demonstrate the advantage of $I(X;Z_{l}^s|Y)+I(\phi_{l}^{S}; S)$ in our bound over $I(f^{S}; S)$ in the previous bounds. Here, notice that $I(\phi_{l}^{S}; S) \neq I(f^{S}; S)$ for any $l\neq D+1$, and we always have $I(\phi_{1}^{S}; S)\le \cdots \le I(\phi_{D}^{S}; S) \le I(\phi_{D+1}^{S}; S) = I(f^{S}; S)$ (e.g., see Fig. \ref{fig:haha}). See Appendix~\ref{app:11} for more comparison with these previous bounds, which are \textbf{\textit{not}} of information bottlenecks although they use mutual information.

The various interesting properties of information bottleneck are discussed in \citep{achille2018emergence}. However, they do not provide generalization bounds with information bottlenecks. Among others, they discuss a connection between $\mathbf{I}(X;Z_{l}^s)$ and $\mathbf{I}(\tilde \theta; S)$, where  $\mathbf{I}$ represents an over-simplified version of  mutual information in which everything is ignored except an artificially added noise; i.e., $\mathbf{I}(\tilde \theta ; S)\triangleq-\frac{1}{2} \sum_{i=1}^d \log \alpha_i$ (this definition is given in Remark 4.2 of \citealp{achille2018emergence})\ where $\alpha_i$ is the variance of the Gaussian noise $\epsilon$ that is artificially multiplied to the  learned weights $\theta$ after training as $\tilde \theta = \epsilon \theta$. In other words, $\mathbf{I}$ completely ignores the dependence of $\theta$ on $S$, although it is the main factor in the question of generalization. For example, if we set $\alpha_i=1$, we have $\mathbf{I}(\tilde \theta; S)=0$ always, despite the fact that $I(\theta; S) \neq 0$ when $\theta$ is learned from $S$. This shows that the previous paper did not  consider the mutual information between $\theta$ and $S$; it used the entropy of the artificially multiplied Gaussian noise $\epsilon$ as mutual information.
 Thus the meaningful factors of mutual information are ignored in the previous work. %In contrast, we study the connection between generalization and $I(X;Z_{l}^s)$ without ignoring  any factors.

There are generalization bounds via different approaches including VC-dimension \citep{vapnik1999nature,bartlett2019nearly}, Rademacher  complexity \citep{bartlett2002rademacher,truong2022rademacher,trauble2023discrete},  stability \citep{bousquet2002stability,deng2021toward}, robustness \citep{xu2012robustness,liudiscrete,kawaguchi2022robust,liu2023adaptive}, and PAC-Bayes \citep{dziugaite2021role, lotfi2022pac}. Unlike these previous studies, we provided the first generalization bound via the information bottleneck.

%\vspace*{-0.25em}
\section{Conclusion} 
%\vspace*{-0.45em}
This study completed the proof of the previous conjecture with near-exponential improvements for the setting of fixed representations, proved the first rigorous generalization bound for the setting of learning representations,
and empirically strengthened the findings with supporting experiments. This paper makes a contribution on technical aspects relevant for current and future methods of deep learning through the lens of information bottlenecks. Whereas information bottleneck is explicit in various algorithms (e.g., \citealp{federici2020learning, sun2022graph,li2022invariant,li2022explanation,su2023vision}), it is also interesting to motivate future methods based on implicit effects of architectures (e.g., transformer and convolution) and training (e.g., SGD and self-supervised objectives) on information bottlenecks.

\bibliography{all}
\bibliographystyle{icml2023}

%%%%%%%%%%%%%%%%%%%%%%%%%%%%%%%%%%%%%%%%%%%%%%%%%%%%%%%%%%%%%%%%%%%%%%%%%%%%%%%
%%%%%%%%%%%%%%%%%%%%%%%%%%%%%%%%%%%%%%%%%%%%%%%%%%%%%%%%%%%%%%%%%%%%%%%%%%%%%%%
% APPENDIX
%%%%%%%%%%%%%%%%%%%%%%%%%%%%%%%%%%%%%%%%%%%%%%%%%%%%%%%%%%%%%%%%%%%%%%%%%%%%%%%
%%%%%%%%%%%%%%%%%%%%%%%%%%%%%%%%%%%%%%%%%%%%%%%%%%%%%%%%%%%%%%%%%%%%%%%%%%%%%%%

\appendix
\onecolumn

\appendix

\allowdisplaybreaks

\subsubsection*{Reproducibility Statement}
For the theoretical results, complete proofs are provided. For the empirical experiments, source code is available at: \href{https://github.com/xu-ji/information-bottleneck}{https://github.com/xu-ji/information-bottleneck}

%\clearpage

%\newpage

\section{Experimental details for \cref{s:experiments1}}\label{app:exp1}

\subsection{Training}

\paragraph{Data.} The dataset was 5-way classification on 2D clustered inputs (\cref{fig:data}). Each dataset draw contained 50 training points and 250 test points.

\begin{figure}[ht]
    \centering
\includegraphics[width=0.5\textwidth]{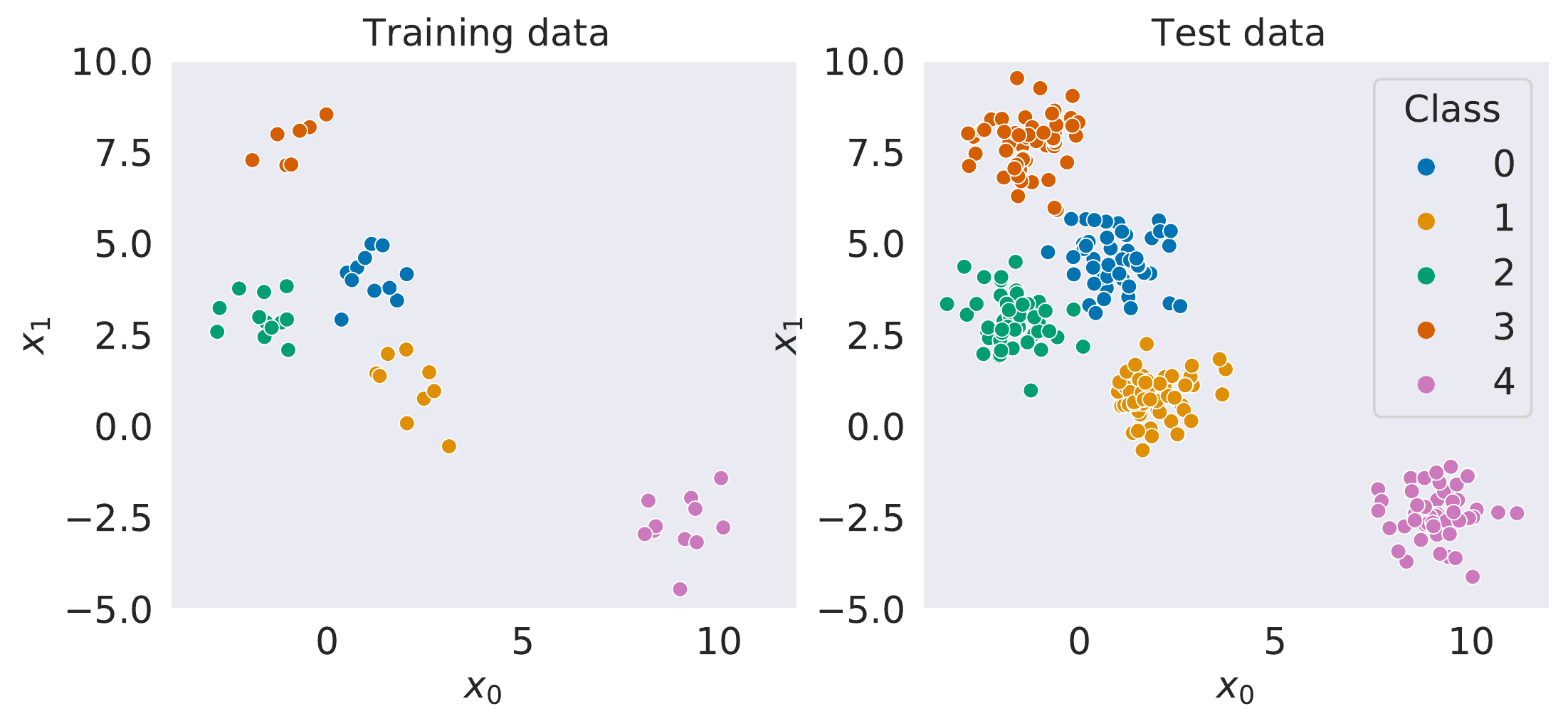}
    \caption{Example draw for 2D classification dataset.}
    \label{fig:data}
\end{figure}

\FloatBarrier

\paragraph{Training.}\label{app:training} 216 models were trained for all combinations of options: 4 ReLU-activated MLP architectures (per-layer widths of [256, 256, 128, 128],
    [128, 128, 64, 64],
    [64, 64, 32, 32],
    [32, 32, 16, 16]), 3 weight decay rates (0, 0.01, 0.1), 3 dataset draws, 3 random seeds, and 2 sample set sizes for evaluating $I(X; Z_l^s)$ and $I(X; Z_l^s | Y)$. 
Final features were sampled from deterministically computed mean and standard deviation vectors and mapped to class probabilities with a softmax-activated linear layer. 
The expectation in MI over $\PP(Z_l^s| x)$ depends on neural network parameters, so the
reparameterization trick was used to optimize the expectation with respect to neural network parameters by rewriting the expectation over $Z_l^s$ as an expectation over random noise \citep{kingma2015variational}.
Models were trained for 300 iterations with a learning rate of $\eta_\theta = 1e-2$. Out of all settings, 36 models with training set accuracy $< 85\%$  were discarded. Statistics for accepted models are given in \cref{tab:stats}. Of 180 accepted models, 9 (5\%) had a small negative generalization gap in loss ($-0.0258 \pm 0.0161$). These models were not screened out before evaluating the metrics because the generalization gap is not estimatable without access to labelled test data. 

\begin{table}[ht] % {wraptable}{r}{9cm}
    \centering
    
    \begin{tabular}{l c c c c}
    \toprule
     & Mean & Standard deviation & Max & Min  \\
    \midrule
Train loss & 0.1265 & 0.1603 & 0.5757 & 0.0018  \\
Train accuracy & 0.9680 & 0.0479 & 1.0000 & 0.8600  \\
Test loss & 0.1984 & 0.1593 & 0.5487 & 0.0247  \\
Test accuracy & 0.9356 & 0.0568 & 0.9960 & 0.7880 \\
\bottomrule
    \end{tabular}
    \caption{Performance statistics of 180 accepted models.}
    \label{tab:stats}
\end{table}

\paragraph{Constrained optimization.} 
In each learning iteration, the gradient of the relaxed problem $\theta^s \leftarrow \theta^s - \eta_{\theta^s} \nabla_{\theta^s} \big [ \big( - \frac{1}{|s|} \sum_{(x, y) \in s} \big( \log \frac{1}{k} \sum_{j=1}^k q_{\theta^{s}}(y | z^j) \big) \big)
+ \lambda (\rho - \hat{I}(X; Z_l^s)) \big ] $ was applied to update the model and $\lambda \leftarrow \lambda + \eta_\lambda (\rho - \hat{I}(X; Z_l^s))$ was applied to update the multiplier $\lambda$, where $\eta_{\theta^s}$ and $\eta_\lambda$ are learning rates.
For a similar use case for dual gradient descent, see \citet{eysenbach2021robust}. Example plots showing the change in $\hat{I}(X; Z_l^s)$ and $\lambda$ during training are given in \cref{fig:constrained_opt}. Note that the gradient of $\lambda$ is a term in the gradient of $\theta^s$, thus updating $\lambda$ incurs negligible additional cost.

\begin{figure}[h]
    \centering
    \includegraphics[width=0.3\textwidth]{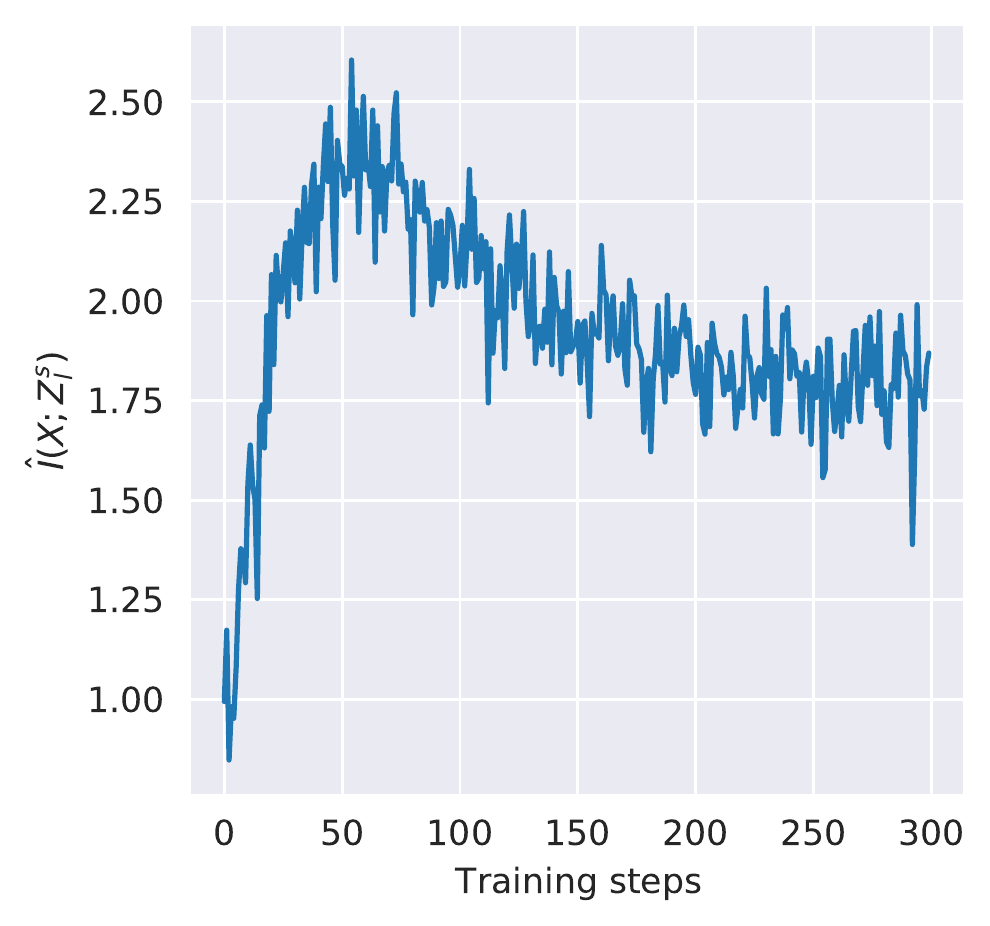}%
    \includegraphics[width=0.3\textwidth]{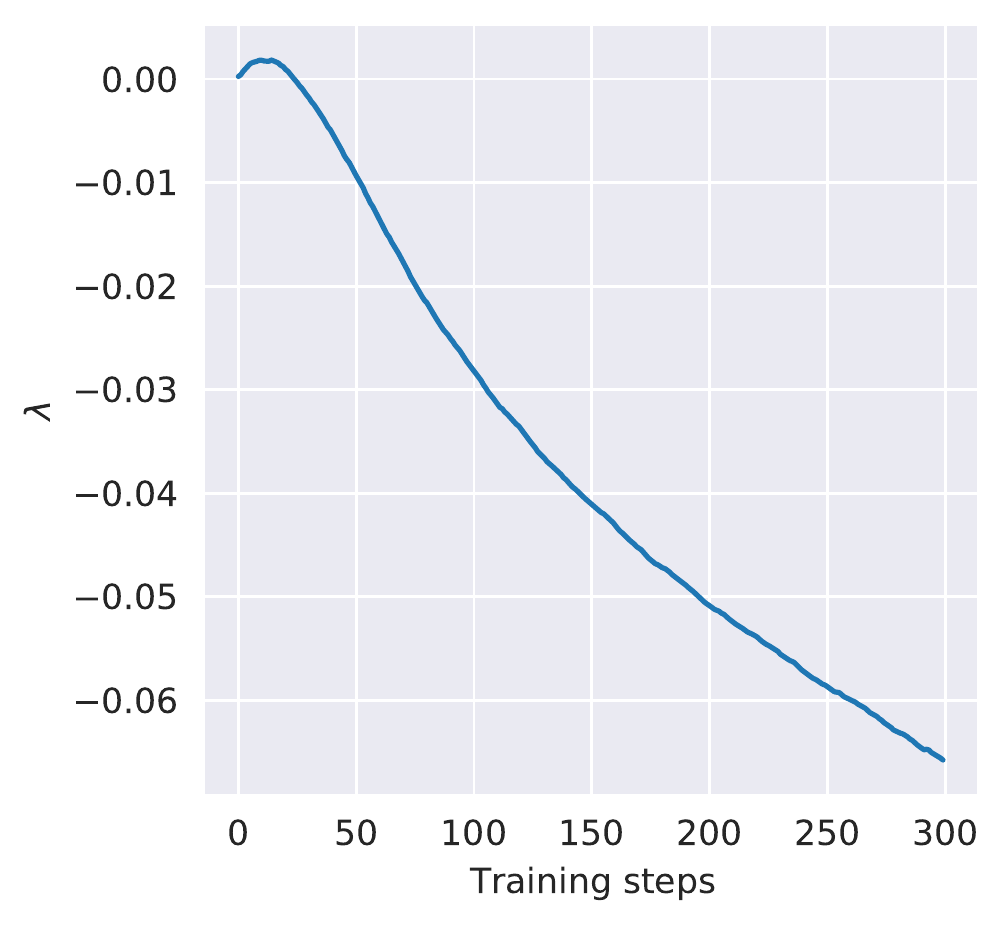}
    \caption{Example plots of $\hat{I}(X; Z_l^s)$ and $\lambda$ during constrained optimized of a neural network model with $\rho = 1.5$.}
    \label{fig:constrained_opt}
\end{figure}

\subsection{Metrics}\label{app:experiments1_metrics}
SWAG provides an estimate of the posterior as a multivariate Gaussian by averaging gradient updates across training epochs. SWAG was used in the estimator 
$
\breve{I}({S}; \theta_l^{S}) = (1/|D|) \sum_{s \in D} (1/k) \sum_{j=1}^k (\log p(w^j | s)) - ((1/|D|) \sum_{s' \in D} \log p(w^j | s')) \ge (1/|D|) \allowbreak \sum_{s \in D} (1/k) \sum_{j=1}^k \log (p(w^j | s)/((1/|D|)\sum_{s' \in D} p(w^j | s')))
$,
where $w^j \sim \PP( \allowbreak \theta_l^{S} \allowbreak | {S} = s)$ for all $j$ and the upper bound is 
obtained via Jensen's inequality. We found that averaging in the log domain by using the upper bound improved numerical  stability compared to averaging in the probability domain due to large magnitudes of $\log p(w^j | s)$.

Mutual information between variables is a measure of their statistical dependence and is defined in our setting as:
\begin{align}
I(X; Z_l^s) &= \mathbb{E}_{X, Z_l^s} \log \frac{q_{\theta^{s}}(Z_l^s | X)}{\mathbb{E}_{X'} q_{\theta^{s}}(Z_l^s| X')}, \\
I(X; Z_l^s | Y) &= \mathbb{E}_{Y} \mathbb{E}_{X_Y, Z_l^s} \log \frac{q_{\theta^{s}}(Z_l^s | X_Y)}{\mathbb{E}_{X'_Y} q_{\theta^{s}}(Z_l^s | X'_Y)}, \\
I({S}; \theta_l^{S}) &= \mathbb{E}_{{S}} \mathbb{E}_{\theta_l^{S} | S} \log \frac{\PP(\theta_l^{S} | {S})}{\mathbb{E}_{{S}''} \PP(\theta_l^{S''} | {S}'')}, \\
\end{align}
where $X'$, $X'_Y$, ${S''}$ are independent copies of variables $X$, $X_Y$, ${S}$ respectively.
Let $C$ be the set of classes and $s_c$ denote dataset samples for class $c$. 
We use $\hat{I}$ to denote estimation by Monte-Carlo sampling and $\breve{I}$ to denote upper bounding via the Jensen inequality:

\begin{align}
\hat{I}(X; Z_l^s) &= \frac{1}{|s|} \sum_{(x, y) \in s} \frac{1}{k} \sum_{j=1}^k \log \frac{q_{\theta^{s}}(z^j | x)}{\frac{1}{|s|} \sum_{(x', y') \in s} q_{\theta^{s}}(z^j | x')}, \label{eq:exp2_MI_mc} \\
\hat{I}(X; Z_l^s | Y) &= \frac{1}{|s|} \sum_{c \in C} \sum_{(x, y) \in s_c} \frac{1}{k} \sum_{j=1}^k  \log \frac{q_{\theta^{s}}(z^j | x)}{\frac{1}{|s_c|} \sum_{(x', y') \in s_c} q_{\theta^{s}}(z^j | x')},  \\
\breve{I}(X; Z_l^s) &= \frac{1}{|s|} \sum_{(x, y) \in s} \frac{1}{k} \sum_{j=1}^k \left( \log q_{\theta^{s}}(z^j | x) - \bigg( \frac{1}{|s|} \sum_{(x', y') \in s} \log q_{\theta^{s}}(z^j | x') \bigg) \right), \label{eq:exp2_MI_jensen} \\
\breve{I}(X; Z_l^s | Y) &= \frac{1}{|s|} \sum_{c \in C} \sum_{(x, y) \in s_c} \frac{1}{k} \sum_{j=1}^k \left( \log q_{\theta^{s}}(z^j | x) - \bigg( \frac{1}{|s_c|} \sum_{(x', y') \in s_c} \log q_{\theta^{s}}(z^j | x') \bigg) \right),
\end{align}
where $z^j \sim q_{\theta^{s}}(Z_l^s | x)$ for all $j$.

For mutual information between the model and training dataset, we compute:
\begin{align}\label{eq:model_compr_jensen}
    \breve{I}({S}; \theta_l^{S}) &= \frac{1}{|D|} \sum_{s \in D} \frac{1}{k} \sum_{j=1}^k \left( \log p(w^j | s) - \bigg( \frac{1}{|D|} \sum_{s' \in D} \log p(w^j | s') \bigg) \right),
\end{align}
where $w^j \sim \PP(\theta_l^{S} | {S} = s)$. 
The learning algorithm is defined by the variables excluding the training dataset, i.e. architecture, weight decay, multiplier learning rate, seed.
Denote the average values of $\breve{I}({S}; \theta_l^{S})$ and $\hat{I}(X; Z_l^s | Y)$ across learning algorithms by $\mu$ and $\mu'$ respectively. 
The rescaled value $\tilde{I}({S}; \theta_l^{S})$ is defined by:
\begin{align}
    \tilde{I}({S}; \theta_l^{S}) = \frac{\mu'}{\mu} \breve{I}({S}; \theta_l^{S}).
\end{align}

Note that MI is measured in universal units, but we test multiple estimation procedures for $I(X; Z_l^s)$ and $I({S}; \theta_l^{S})$. Scaling was tested because of the difference of estimators, rather than of units of the estimated quantity.

\subsection{Additional results}

As the true data generator is available for the toy dataset, 2 sample set sizes were considered for computing estimators of $I(X; Z_l^s)$ and $I(X; Z_l^s | Y)$ during evaluation of the metrics (\cref{app:training}): using the training dataset (50 data points), and using a sample 10 times larger drawn from the generator (500 data points). 
Using the larger sample size (\cref{tab:2_large,tab:exp1_summary}) improved the predictive ability of the baseline representation compression metrics compared to using the small sample size (\cref{tab:2_small}). 

\begin{table}[h]
    %\footnotesize
    \centering
   
    \begin{tabular}{l c c c}
\toprule
Metric & Spearman &Pearson & Kendall  \\
\midrule

Num. params. $m$     & -0.0576 & -0.0294 & -0.0402 \\
$m \log m$      & -0.0576 & -0.0287 & -0.0402 \\
$\sum_\lb \theta_\lb^{s}$       & -0.2550 & -0.1366 & -0.1567 \\
$\prod_\lb \theta_\lb^{s}$      & -0.2172 & -0.0871 & -0.1374 \\
\midrule

$\hat{I}(X; Z_l^s)$     & 0.1816 & 0.2878 & 0.1280 \\
$\hat{I}(X; Z_l^s | Y)$         & 0.1749 & 0.3167 & 0.1129 \\
$\breve{I}(X; Z_l^s)$   & 0.1648 & 0.3712 & 0.1223 \\
$\breve{I}(X; Z_l^s | Y)$       & 0.2293 & 0.3842 & 0.1515 \\
\midrule

$\breve{I}({S}; \theta_l^{S})$    & 0.0020 & 0.0211 & 0.0074 \\
$\breve{I}({S}; \theta_{D+1}^{S})$        & -0.0221 & 0.0091 & -0.0090 \\

$\breve{I}({S}; \theta_l^{S}) + \hat{I}(X; Z_l^s)$        & 0.0178 & 0.0211 & 0.0178 \\
$\breve{I}({S}; \theta_l^{S}) + \hat{I}(X; Z_l^s | Y)$    & 0.0163 & 0.0211 & 0.0167 \\
$\breve{I}({S}; \theta_l^{S}) + \breve{I}(X; Z_l^s)$      & 0.0135 & 0.0212 & 0.0162 \\
$\breve{I}({S}; \theta_l^{S}) + \breve{I}(X; Z_l^s | Y)$  & 0.0164 & 0.0211 & 0.0167 \\

$\tilde{I}({S}; \theta_l^{S}) + \hat{I}(X; Z_l^s)$        & 0.1104 & 0.1401 & 0.0794 \\
$\tilde{I}({S}; \theta_l^{S}) + \hat{I}(X; Z_l^s | Y)$    & 0.2253 & 0.3177 & 0.1567 \\
$\tilde{I}({S}; \theta_l^{S}) + \breve{I}(X; Z_l^s)$      & 0.2684 & 0.3928 & 0.1912 \\
$\tilde{I}({S}; \theta_l^{S}) + \breve{I}(X; Z_l^s | Y)$  & \textbf{0.3015} & \textbf{0.4130} & \textbf{0.2085} \\

\bottomrule
    \end{tabular}
    \caption{Correlation coefficients for metrics and the generalization gap in loss, large sample setting for estimation of $I(X; Z_l^s)$ and $I(X; Z_l^s | Y)$. $\theta_\lb^{S}$ denotes parameters of layer $\lb$ and $\theta_l^{S}$ denotes parameters up to layer $l$. Layer $l$ is fixed to the penultimate layer. $>0$ indicates positive correlation. }
    \label{tab:2_large}
\end{table}

\begin{table}[h]
    %\footnotesize
    \centering
    
    \begin{tabular}{l c c c }
\toprule
Metric & Spearman &Pearson & Kendall  \\
\midrule

$\hat{I}(X; Z_l^s)$     & -0.1066 & -0.0972 & -0.0709  \\
$\hat{I}(X; Z_l^s | Y)$         & 0.0868 & 0.0394 & 0.0698   \\
$\breve{I}({S}; \theta_l^{S}) + \hat{I}(X; Z_l^s | Y)$    & 0.2360 & 0.2489 & 0.1611  \\
$\tilde{I}({S}; \theta_l^{S}) + \hat{I}(X; Z_l^s | Y)$    & \textbf{0.3277} & \textbf{0.2888} & \textbf{0.2257} \\

\bottomrule
    \end{tabular}
    \caption{Correlation coefficients for metrics and the generalization gap in loss, small sample setting for estimation of $I(X; Z_l^s)$ and $I(X; Z_l^s | Y)$. Best metric and ablation shown. Layer $l$ is fixed to the penultimate layer. $>0$ indicates positive correlation.}
    \label{tab:2_small}
\end{table}

\begin{figure}[h]
\centering
\begin{subfigure}[$\breve{I}(X; Z_l^s | Y)$]{\includegraphics[width=.28\textwidth]{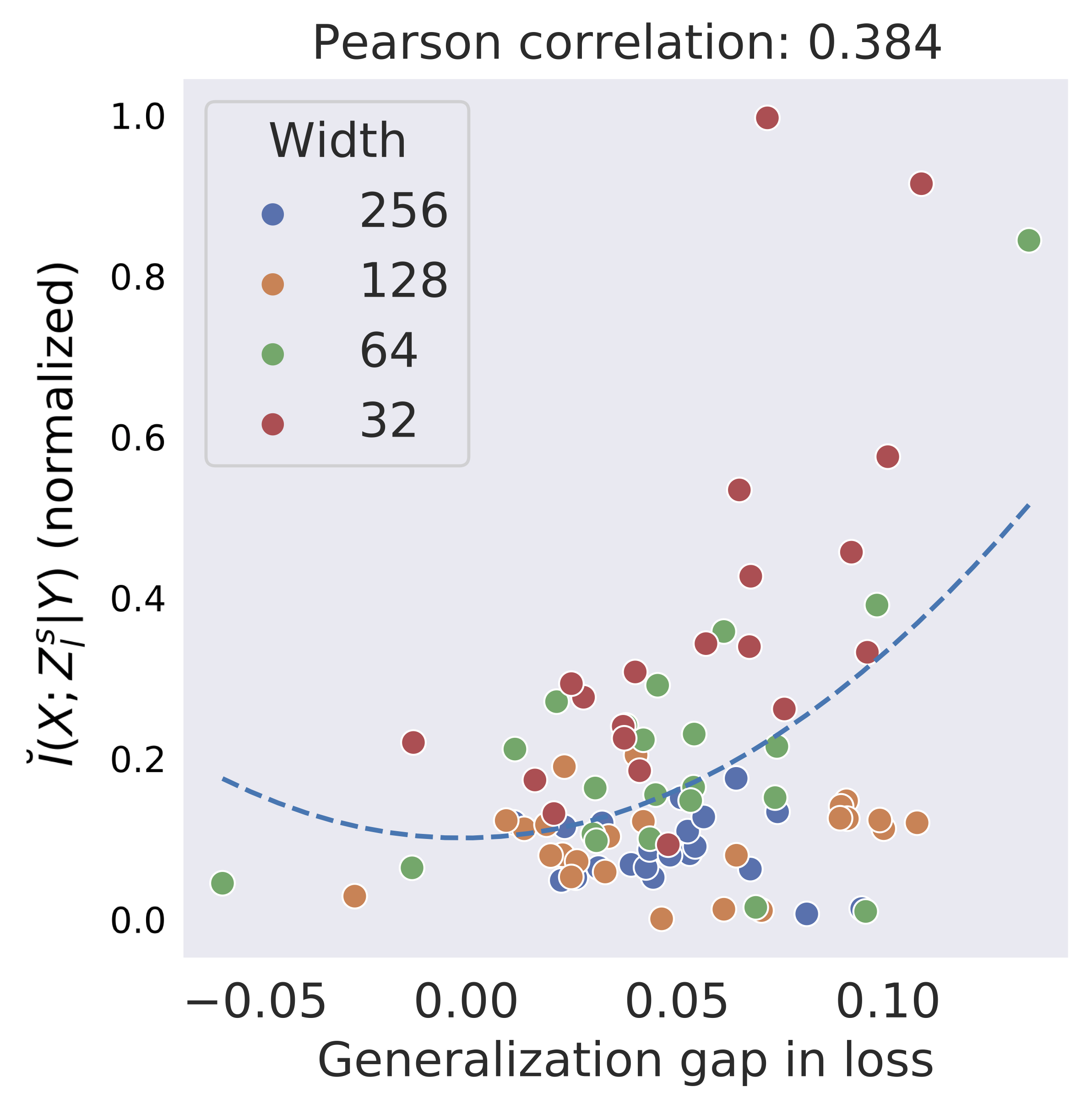}}
  %\label{fig:extra_1a}
\end{subfigure}%
\begin{subfigure}[$\tilde{I}({S}; \theta_l^{S}) + \breve{I}(X; Z_l^s)$]{
  \includegraphics[width=0.29\textwidth]{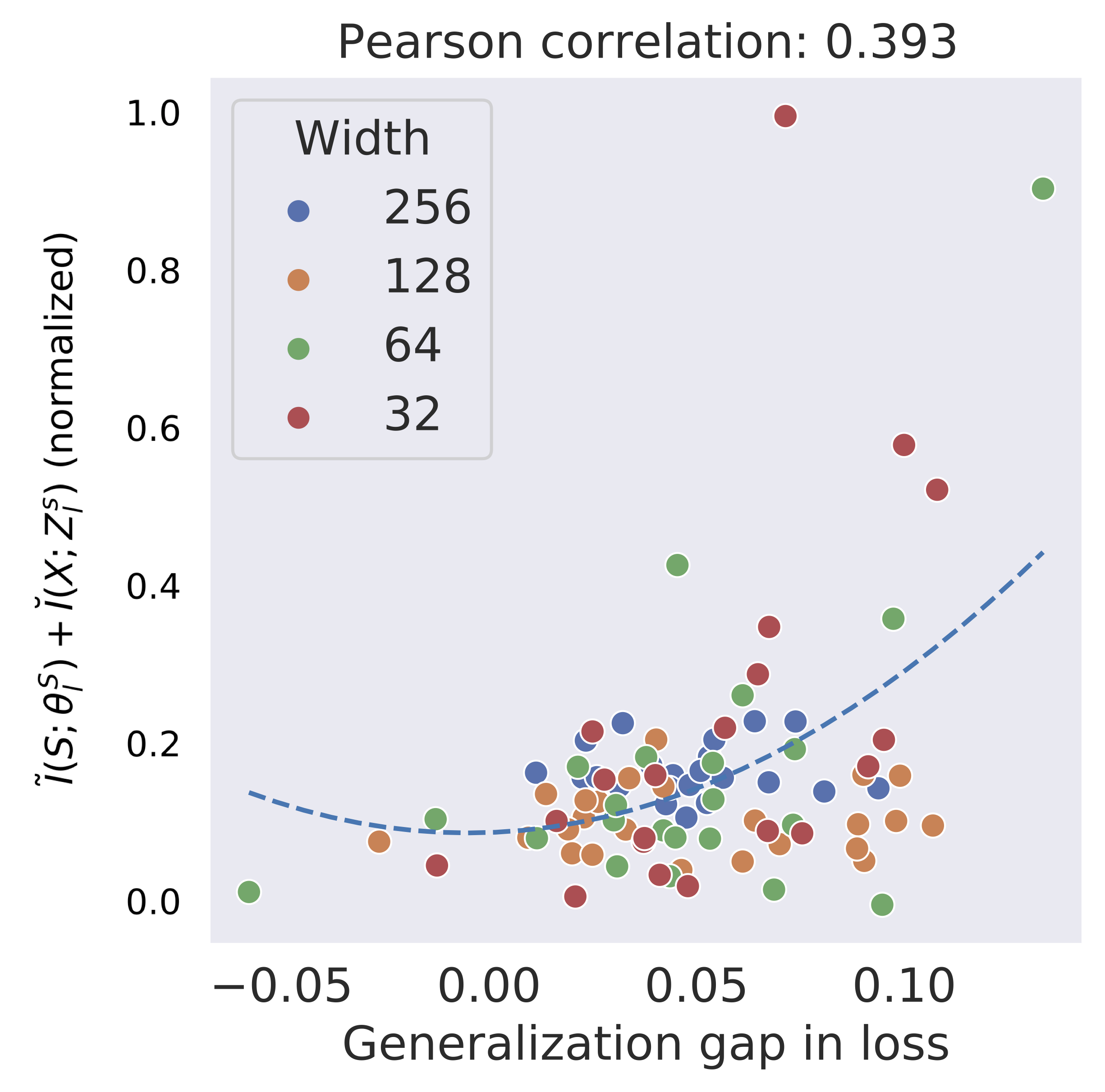}}
  %\label{fig:extra_1b}
\end{subfigure}%
\begin{subfigure}[$\tilde{I}({S}; \theta_l^{S}) + \breve{I}(X; Z_l^s | Y)$ ]{\includegraphics[width=0.28\textwidth]{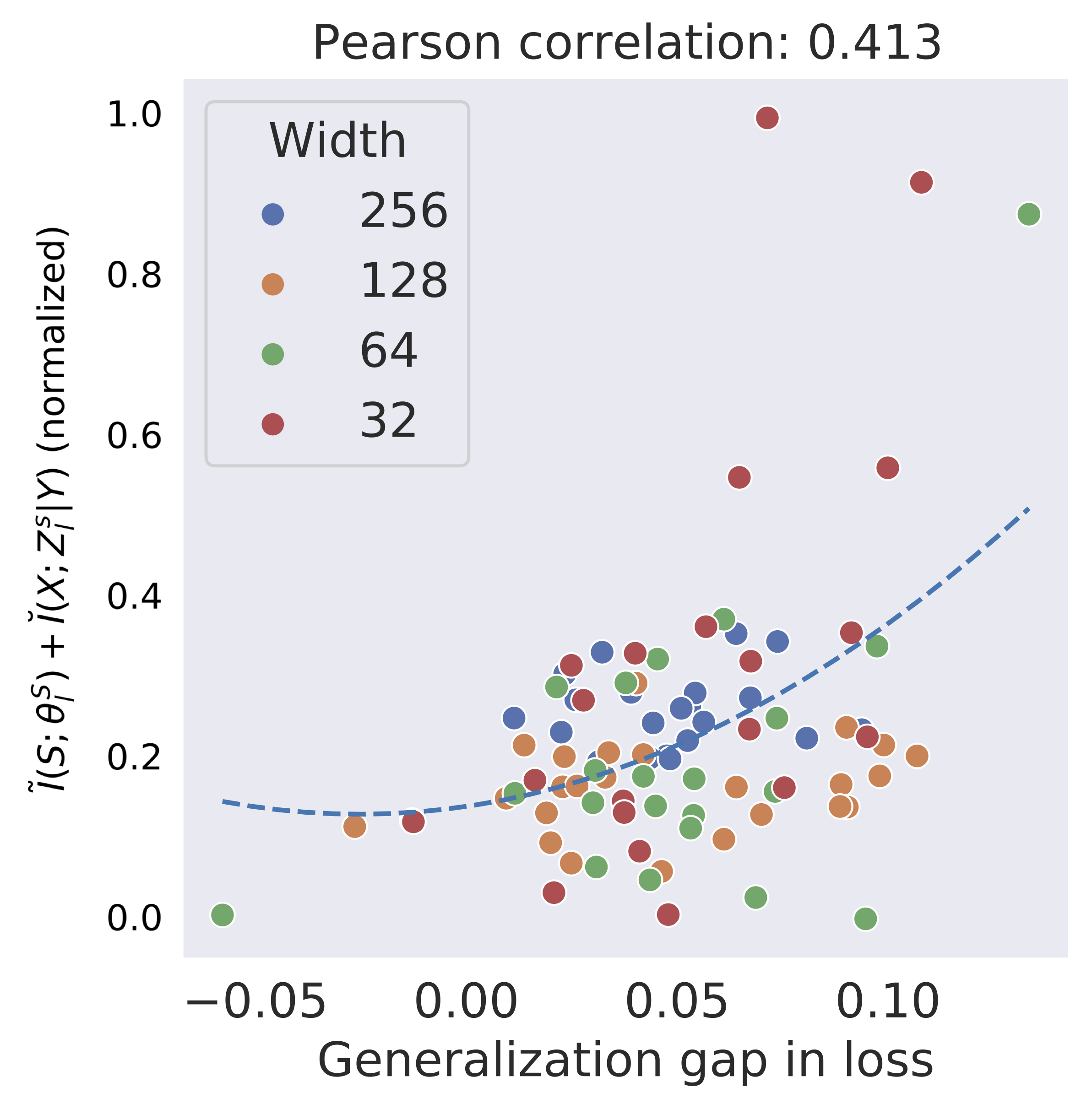}}
  %\label{fig:extra_1c}
\end{subfigure}%
    \caption{
Correlation of best 3 metrics with the generalization gap. Color denotes network width. Dashed line denotes best polynomial fit with degree 2. Values are normalized by subtracting the minimum and dividing by the range.
    } 
    \label{fig:extra_2}
\end{figure}

\FloatBarrier

\section{MNIST and Fashion MNIST}\label{app:mnist}

We conducted experiments on the MNIST and Fashion MNIST datasets. 
These experiments follow the same protocol described in \cref{s:experiments1,app:exp1} except that MI was not constrained, in order to investigate predictive ability of the metrics in the setting of unconstrained stochastic feature models. 
144 models were trained over all combinations of options: 2 datasets, 3 ReLU-activated architectures (1 convolutional layer and 3 linear layers, with per hidden layer channel sizes of 
[64, 512, 512],
[32, 256, 256], 
[16, 128, 128]), 
2 weight decay rates (0, 1e-3), 2 batch sizes (128, 32), 3 dataset draws, 2 random seeds.
As in \cref{s:experiments1,app:exp1}, the penultimate layer of the network infers mean and standard deviation vectors that define a distribution over latent features.  
To construct 
multiple instances of the training dataset, we sampled training datasets of size 8K from the training set, and each test set was the original 10K test set. 
Performance of trained models are given in \cref{tab:fmnist_stats,tab:dmnist_stats}.
In line with \cref{s:experiments1,s:experiments2,app:exp1,app:exp2}, results on MNIST and Fashion MNIST indicate that metrics which combine model compression with representation compression outperform metrics for representation compression alone (\cref{tab:dmnist,tab:fmnist}).

\begin{table}[H] % {wraptable}{r}{9cm}
    \centering
 
    \begin{tabular}{l c c c c}
    \toprule
     & Mean & Standard deviation & Max & Min  \\
    \midrule

Train loss & 0.0071 & 0.0165 & 0.1053 & 0.0001 \\
Train accuracy & 0.9986 & 0.0060 & 1.0000 & 0.9628 \\
Test loss & 0.1355 & 0.0285 & 0.2564 & 0.0915 \\
Test accuracy & 0.9673 & 0.0065 & 0.9737 & 0.9358 \\

\bottomrule
    \end{tabular}
    \caption{Performance statistics for MNIST models.}
    \label{tab:dmnist_stats}
\end{table}

\begin{table}[H] % {wraptable}{r}{9cm}
    \centering
   
    \begin{tabular}{l c c c c}
    \toprule
     & Mean & Standard deviation & Max & Min  \\
\midrule

Train loss & 0.0692 & 0.0640 & 0.1908 & 0.0003 \\
Train accuracy & 0.9765 & 0.0234 & 1.0000 & 0.9329 \\
Test loss & 0.5609 & 0.1682 & 0.8819 & 0.3791 \\
Test accuracy & 0.8721 & 0.0090 & 0.8864 & 0.8262 \\

\bottomrule
    \end{tabular}
    \caption{Performance statistics for Fashion MNIST models.}
    \label{tab:fmnist_stats}
\end{table}

\begin{table}[H]
    %\footnotesize
    \centering
    
    \begin{tabular}{l c c c }
\toprule
Metric & Spearman & Pearson & Kendall  \\
\midrule

$\hat{I}(X; Z_l^s)$     & 0.2738 & -0.1268 & 0.2178 \\

$\hat{I}(X; Z_l^s | Y)$         & 0.4399 & 0.2059 & 0.3243 \\

$\breve{I}(X; Z_l^s)$   & 0.7895 & 0.5467 & 0.6346 \\

$\breve{I}(X; Z_l^s | Y)$       & 0.7931 & 0.5348 & \textbf{0.6416} \\

\midrule

$\breve{I}({S}; \theta_l^{S})$    & 0.6004 & 0.7044 & 0.4193 \\

$\bar{I}({S}; \theta_l^{S})$      & 0.5328 & 0.6906 & 0.3771 \\

$\breve{I}({S}; \theta_{D+1}^{S})$        & 0.5384 & 0.6619 & 0.3768 \\

$\bar{I}({S}; \theta_{D+1}^{S})$  & 0.5328 & 0.6507 & 0.3771 \\

$\tilde{I}({S}; \theta_l^{S}) + \hat{I}(X; Z_l^s)$        & 0.6021 & 0.7043 & 0.4130 \\

$\tilde{I}({S}; \theta_l^{S}) + \hat{I}(X; Z_l^s | Y)$    & 0.5958 & 0.7044 & 0.3955 \\

$\tilde{I}({S}; \theta_l^{S}) + \breve{I}(X; Z_l^s)$      & \textbf{0.8352} & 0.6367 & \textbf{0.6452} \\

$\tilde{I}({S}; \theta_l^{S}) + \breve{I}(X; Z_l^s | Y)$  & \textbf{0.8303} & 0.6242 & 0.6384 \\

$\breve{I}({S}; \theta_l^{S}) + \hat{I}(X; Z_l^s)$        & 0.6021 & 0.7044 & 0.4130 \\

$\breve{I}({S}; \theta_l^{S}) + \hat{I}(X; Z_l^s | Y)$    & 0.5958 & 0.7044 & 0.3955 \\

$\breve{I}({S}; \theta_l^{S}) + \breve{I}(X; Z_l^s)$      & 0.7128 & \textbf{0.7626} & 0.5119 \\

$\breve{I}({S}; \theta_l^{S}) + \breve{I}(X; Z_l^s | Y)$  & 0.6566 & \textbf{0.7329} & 0.4610 \\

\bottomrule
    \end{tabular}
    \caption{MNIST. Correlation coefficients for metrics and the generalization gap in loss. 
    Layer $l$ is fixed to the penultimate layer. $>0$ indicates positive correlation.}
    \label{tab:dmnist}
\end{table}

\begin{table}[H]
    %\footnotesize
    \centering
  
    \begin{tabular}{l c c c }
\toprule
Metric & Spearman & Pearson & Kendall  \\
\midrule

$\hat{I}(X; Z_l^s)$     & -0.0299 & -0.2056 & -0.0261 \\

$\hat{I}(X; Z_l^s | Y)$         & 0.2318 & 0.1185 & 0.1458 \\

$\breve{I}(X; Z_l^s)$   & 0.3861 & 0.5146 & 0.2293 \\

$\breve{I}(X; Z_l^s | Y)$       & 0.3848 & 0.5115 & 0.2308 \\

\midrule

$\breve{I}({S}; \theta_l^{S})$    & 0.3479 & 0.3191 & 0.2682 \\

$\bar{I}({S}; \theta_l^{S})$      & 0.3471 & 0.2804 & 0.2684 \\

$\breve{I}({S}; \theta_{D+1}^{S})$        & 0.3479 & 0.3187 & 0.2682 \\

$\bar{I}({S}; \theta_{D+1}^{S})$  & 0.3928 & 0.3121 & 0.2998 \\

$\tilde{I}({S}; \theta_l^{S}) + \hat{I}(X; Z_l^s)$        & 0.3377 & 0.3188 & 0.2469 \\

$\tilde{I}({S}; \theta_l^{S}) + \hat{I}(X; Z_l^s | Y)$    & 0.3446 & 0.3191 & 0.2660 \\

$\tilde{I}({S}; \theta_l^{S}) + \breve{I}(X; Z_l^s)$      & \textbf{0.5623} & \textbf{0.6488} & \textbf{0.4238} \\

$\tilde{I}({S}; \theta_l^{S}) + \breve{I}(X; Z_l^s | Y)$  & \textbf{0.5637} & \textbf{0.6441} & \textbf{0.4344} \\

$\breve{I}({S}; \theta_l^{S}) + \hat{I}(X; Z_l^s)$        & 0.3377 & 0.3191 & 0.2469 \\

$\breve{I}({S}; \theta_l^{S}) + \hat{I}(X; Z_l^s | Y)$    & 0.3446 & 0.3191 & 0.2660 \\

$\breve{I}({S}; \theta_l^{S}) + \breve{I}(X; Z_l^s)$      & 0.3734 & 0.3280 & 0.2677 \\

$\breve{I}({S}; \theta_l^{S}) + \breve{I}(X; Z_l^s | Y)$  & 0.3679 & 0.3217 & 0.2766 \\

\bottomrule
    \end{tabular}
    \caption{Fashion MNIST. Correlation coefficients for metrics and the generalization gap in loss. Layer $l$ is fixed to the penultimate layer. $>0$ indicates positive correlation.}
    \label{tab:fmnist}
\end{table}

\reve 

\FloatBarrier

\section{Feature binning} \label{app:binning}

Binning is a method for characterizing feature compression in neural networks with deterministic features \citep{shwartz2017opening}. 
To investigate the performance of feature and model compression for predicting generalization, we trained models for the toy clustering problem of \cref{s:experiments1} using MLPs with deterministic features. The binning implementation of \citet{saxe2019information} was reused to discretize the activity of each node into 10 buckets. 
216 models were trained over all combinations of: 
3 ReLU-activated MLP acrhitectures (5 fully connected layers with per hidden layer channel sizes of 
[512, 512, 256, 256],
[256, 256, 128, 128],
[128, 128, 64, 64]
), 
3 weight decay rates (0, 1e-3, 1e-2), 3 dataset draws, 3 random seeds, with and without information bottleneck (IB) feature regularization.
For models trained with IB regularization, we used the surrogate objective of \citet[Eq. 1]{kirsch2020unpacking}.
While feature compression and model compression both performed reasonably in the setting without IB regularization (\cref{tab:binning_noreg}), utilizing IB regularization significantly impaired the performance of metrics based solely on feature compression (\cref{tab:binning_reg}).

\begin{table}[h]
    %\footnotesize
    \centering
    
    \begin{tabular}{l c c c }
\toprule
Metric & Spearman & Pearson & Kendall  \\
\midrule
$\hat{I}(X; Z_l^s)$     & \textbf{0.3271} & \textbf{0.3032} & \textbf{0.2235} \\
$\hat{I}(X; Z_l^s | Y)$         & \textbf{0.2431} & 0.2194 & \textbf{0.1643} \\
Num. param. $m$         & 0.0078 & -0.0259 & 0.0064 \\
$m \log m$      & 0.0078 & -0.0267 & 0.0064 \\
$\sum_l \theta_l^{s}$  & 0.0557 & 0.0770 & 0.0444 \\
$\prod_l \theta_l^{s}$         & -0.0559 & -0.1755 & -0.0191 \\

\midrule
$\breve{I}({S}; \theta_l^{S})$    & 0.1955 & 0.1851 & 0.1335 \\
$\bar{I}({S}; \theta_l^{S})$      & 0.1955 & 0.1940 & 0.1347 \\

$\tilde{I}({S}; \theta_l^{S}) + \hat{I}(X; Z_l^s)$        & 0.2046 & \textbf{0.2419} & 0.1423 \\
$\tilde{I}({S}; \theta_l^{S}) + \hat{I}(X; Z_l^s | Y)$    & 0.1846 & 0.2389 & 0.1272 \\
$\breve{I}({S}; \theta_l^{S}) + \hat{I}(X; Z_l^s)$        & 0.1530 & 0.2041 & 0.1043 \\
$\breve{I}({S}; \theta_l^{S}) + \hat{I}(X; Z_l^s | Y)$    & 0.1377 & 0.2041 & 0.0920 \\

\bottomrule
    \end{tabular}
    \caption{Results for clustering experiments using feature binning for estimation of MI, without information bottleneck regularization. Correlation coefficients for metrics and the generalization gap in loss. Layer $l$ is fixed to the penultimate layer. $>0$ indicates positive correlation.}
    \label{tab:binning_noreg}
\end{table}

\begin{table}[h]
    %\footnotesize
    \centering
    
    \begin{tabular}{l c c c }
\toprule
Metric & Spearman & Pearson & Kendall  \\
\midrule
$\hat{I}(X; Z_l^s)$     & -0.0403 & -0.0440 & -0.0426 \\
$\hat{I}(X; Z_l^s | Y)$         & 0.0252 & 0.0390 & -0.0012 \\

Num. param. $m$         & -0.0103 & 0.0057 & -0.0229 \\
$m \log m$      & -0.0103 & 0.0053 & -0.0229 \\
$\sum_l \theta_l^{s}$  & -0.0579 & 0.0158 & -0.0457 \\
$\prod_l \theta_l^{s}$         & -0.1460 & -0.1640 & -0.1111 \\

\midrule

$\breve{I}({S}; \theta_l^{S})$    & 0.0252 & 0.0567 & 0.0159 \\
$\bar{I}({S}; \theta_l^{S})$      & 0.0370 & 0.0665 & 0.0208 \\

$\tilde{I}({S}; \theta_l^{S}) + \hat{I}(X; Z_l^s)$        & 0.0498 & 0.0545 & 0.0381 \\
$\tilde{I}({S}; \theta_l^{S}) + \hat{I}(X; Z_l^s | Y)$    & 0.0688 & 0.0647 & 0.0512 \\
$\breve{I}({S}; \theta_l^{S}) + \hat{I}(X; Z_l^s)$        & \textbf{0.0757} & \textbf{0.0815} & \textbf{0.0530} \\
$\breve{I}({S}; \theta_l^{S}) + \hat{I}(X; Z_l^s | Y)$    & \textbf{0.0794} & \textbf{0.0815} & \textbf{0.0525} \\

\bottomrule
    \end{tabular}
    \caption{Results for clustering experiments using feature binning for estimation of MI, with information bottleneck regularization. Correlation coefficients for metrics and the generalization gap in loss. Layer $l$ is fixed to the penultimate layer. $>0$ indicates positive correlation.}
    \label{tab:binning_reg}
\end{table}

\section{Experimental details for \cref{s:experiments2}}\label{app:exp2}

\subsection{Training}
540 models were trained for all combinations of the  options: 3 architecutres (PreResNet56, PreResNet83, PreResNet110), 3 weight decay rates (1e-3, 1e-4, 1e-5), 3 batch sizes (64, 128, 1024), 5 dataset draws and 4 random seeds. 
The PreResNet architecture \citep{he2016identity} consists of a convolutional layer, 3 residual blocks and a final linear prediction layer. We consider representations from $D=5$ layers: input layer, after convolutional layer, and after each residual block.
Models were trained for 200 epochs with SGD and a learning rate of $1e-2$. Statistics are given in \cref{tab:exp2_stats}.

\begin{table}[ht]
    \centering
    \begin{tabular}{l c c c c}
    \toprule
     & Mean & Standard deviation & Max & Min \\
     \midrule
Train loss & 0.2157 & 0.2815  & 1.4174 & 0.0005  \\
Train accuracy & 0.9282 & 0.0937  & 1.0000 & 0.5433   \\
Test loss & 0.9704 & 0.2023 & 1.6269 & 0.5989  \\
Test accuracy & 0.8043 & 0.0554 & 0.8770 & 0.5192  \\
\bottomrule
    \end{tabular}
    \caption{Performance statistics of 540 models.}
    \label{tab:exp2_stats}
\end{table}

\subsection{Metrics}\label{app:exp2_metrics}
We used the same metrics as defined in \cref{app:experiments1_metrics} and additionally test excluding the seed from the definition of the learning algorithm by averaging across seeds. Let $G$ denote the set of seeds and let $\Gamma$ denote the seed variable:
\begin{align}
\bar{I}({S}; \theta_l^{S}) = \frac{1}{|D||G|} \sum_{s \in D} \sum_{\gamma \in G} \frac{1}{k} \sum_{j=1}^k \Bigg(& \bigg(  \frac{1}{|G|} \sum_{\gamma' \in G} \log p(w^j | s, \gamma') \bigg)  
\\ \nonumber &-   \bigg(  \frac{1}{|D||G|} \sum_{s'\in D} \sum_{\gamma' \in G} \log p(w^j |s', \gamma') \bigg)\Bigg)
\end{align}
where $w^j \sim \PP(\theta_l^{S} | {S} = s, \Gamma = \gamma)$ is sampled from the estimated posterior produced by SWAG.

\subsection{Kernel Density Estimation}\label{app:exp2_kde}

Without the addition of noise in the hidden representation, mutual information between inputs and deterministic continuous features is ill-defined \citep{saxe2019information}. One way to add noise is to discretize hidden activity into bins \citep{shwartz2017opening, burhanpurkar2021scaffolding}.
Another approach is kernal density estimation \citep{kolchinsky2017estimating, ji2021power,ji2021unconstrained},
which assumes for the purpose of analysis that Gaussian noise with variance $\sigma_l^2$ is added to the representation produced by layer $l$. 
In adaptive KDE \citep{chelombiev2018adaptive} $\sigma_l^2$ is scaled from a base by the maximum observed activation level in the layer, improving on constant $\sigma_l^2$ \citep{saxe2019information} by allowing the level of noise to vary with layers. Following the previous work, we found $1e-3$ to work well as the base value.

As an alternative method for specifying $\sigma_l^2$,  we selected $\sigma^2_l$ from a discrete set by maximum log likelihood of observed features, 
\begin{align}\label{eq:exp2_mle}
    \frac{1}{|s|} \sum_{(x, y) \in s} \frac{1}{k} \sum_{j=1}^k \log \frac{1}{|s|} \sum_{(x', y') \in s} q_{\theta^{s}}(z^j | x') \qquad \text{where } z^j  \sim q_{\theta^{s}}(Z_l^s | x),
\end{align}
under the constraint that estimated MI decreased with layer, which follows from the information processing inequality. This was performed by iterating from layer $D$ to layer $1$ and choosing $\sigma_l^2$ with maximum likelihood such that the estimator of MI was non-decreasing, i.e. $\hat{I}(X; Z_l^s) \geq \hat{I}(X; Z_{l+1}^s)$ for $l < D$. As with estimators in \cref{app:experiments1_metrics}, averaging can be done in the log domain to yield the lower bound:
\begin{align}\label{eq:exp2_mle2}
    \frac{1}{|s|} \sum_{(x, y) \in s} \frac{1}{k} \sum_{j=1}^k \frac{1}{|s|} \sum_{(x', y') \in s} \log q_{\theta^{s}}(z^j | x') \qquad \text{where } z^j  \sim q_{\theta^{s}}(Z_l^s | x).
\end{align}
For consistency, \cref{eq:exp2_mle} was used with $\hat{I}(X; Z_l^s)$ (\cref{eq:exp2_MI_mc}) and \cref{eq:exp2_mle2} with $\breve{I}(X; Z_l^s)$ (\cref{eq:exp2_MI_jensen}).

\subsection{Further results} \label{app:new:exp_results}
We provide further results below.

\begin{table}[b]
    \centering
    \begin{tabular}{l H c c c c c c}
\toprule

& Layer summary 
& \multicolumn{2}{c}{Spearman corr.} & \multicolumn{2}{c}{Pearson corr.} & \multicolumn{2}{c}{Kendall corr.} \\
Generalization gap: & & Loss & Error & Loss & Error & Loss & Error  \\
\midrule

Train loss & mean & -0.8075 & -0.7229 & -0.8376 & -0.8730 & -0.5632 & -0.5307 \\

Test loss & mean & 0.6970 & 0.6095 & 0.6497 & 0.5210 & 0.5938 & 0.4777 \\

Train error & mean & -0.8114 & -0.7269 & -0.8489 & -0.8831 & -0.5707 & -0.5381 \\

Test error & mean & -0.4609 & -0.3834 & -0.5896 & -0.6087 & -0.2652 & -0.2020 \\

Gen. gap error & mean & 0.9443 & 1.0000 & 0.9562 & 1.0000 & 0.8149 & 1.0000 \\

Gen. gap loss & mean & 1.0000 & 0.9443 & 1.0000 & 0.9562 & 1.0000 & 0.8149 \\
\bottomrule
    \end{tabular}
    \caption{Results for prediction with performance metrics.}
    \label{tab:exp2_raw}
\end{table}

\begin{figure}[ht]
    \centering
    \includegraphics[width=0.3\textwidth]{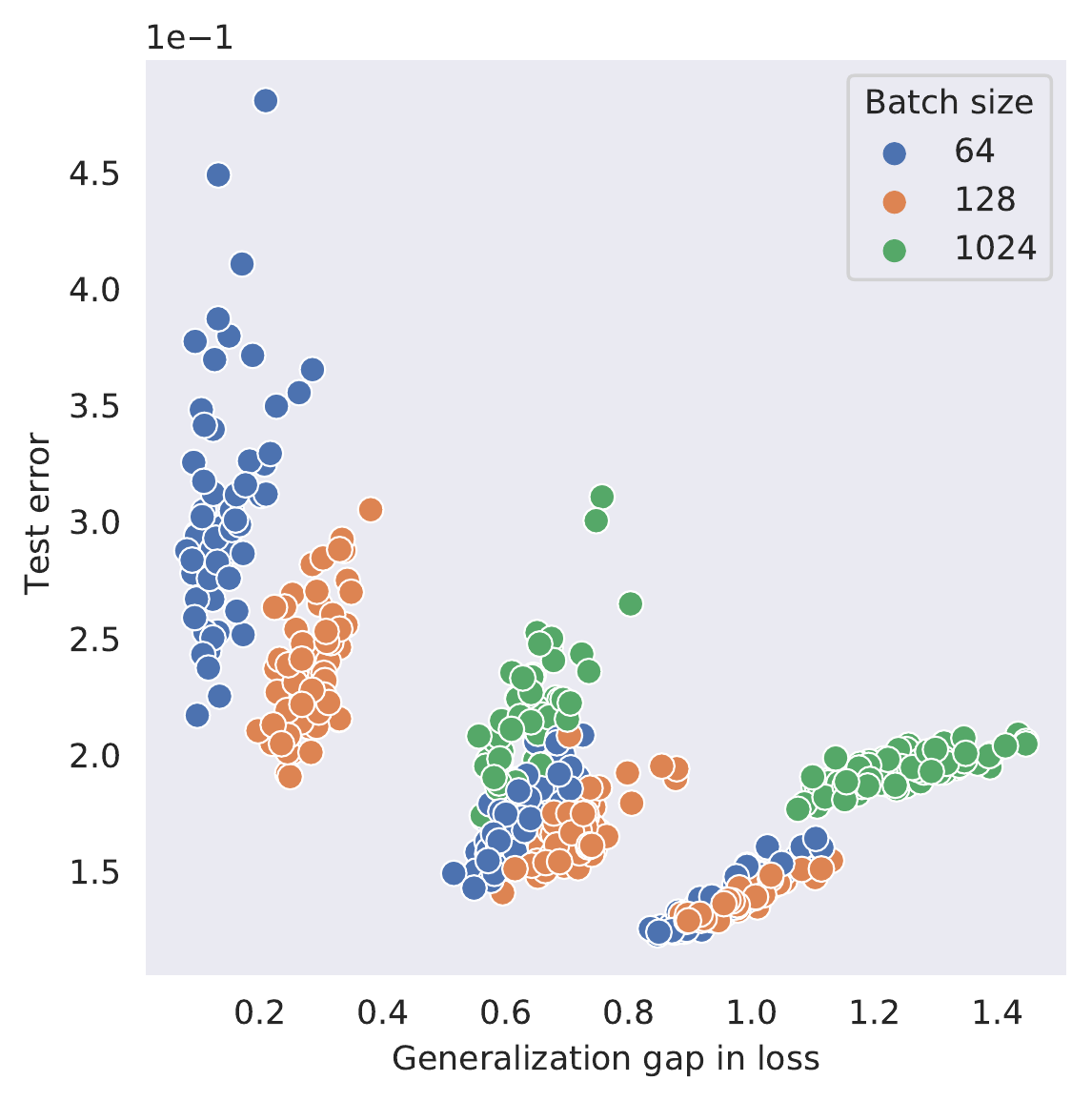}%
    \includegraphics[width=0.3\textwidth]{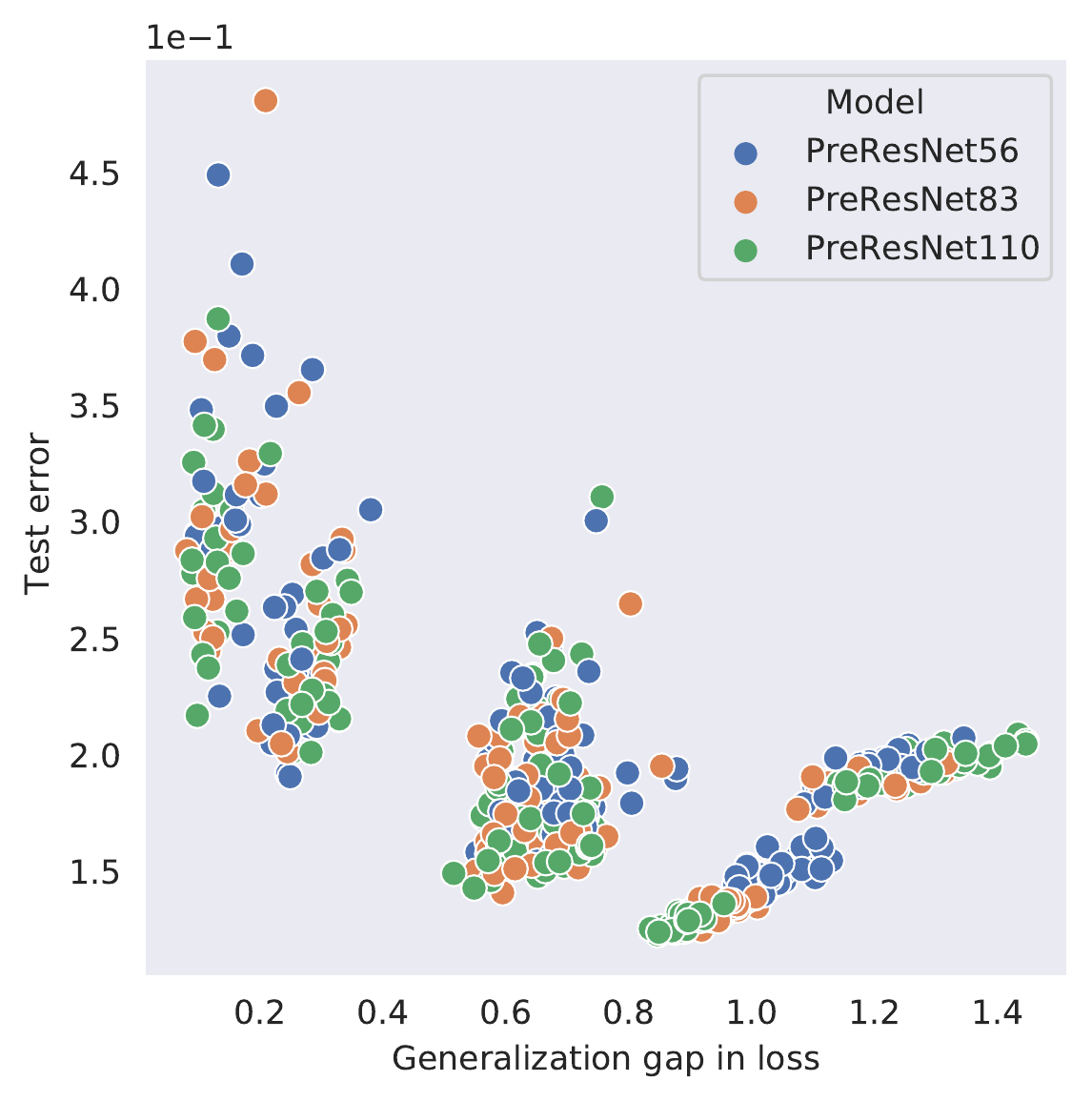}%
    \includegraphics[width=0.3\textwidth]{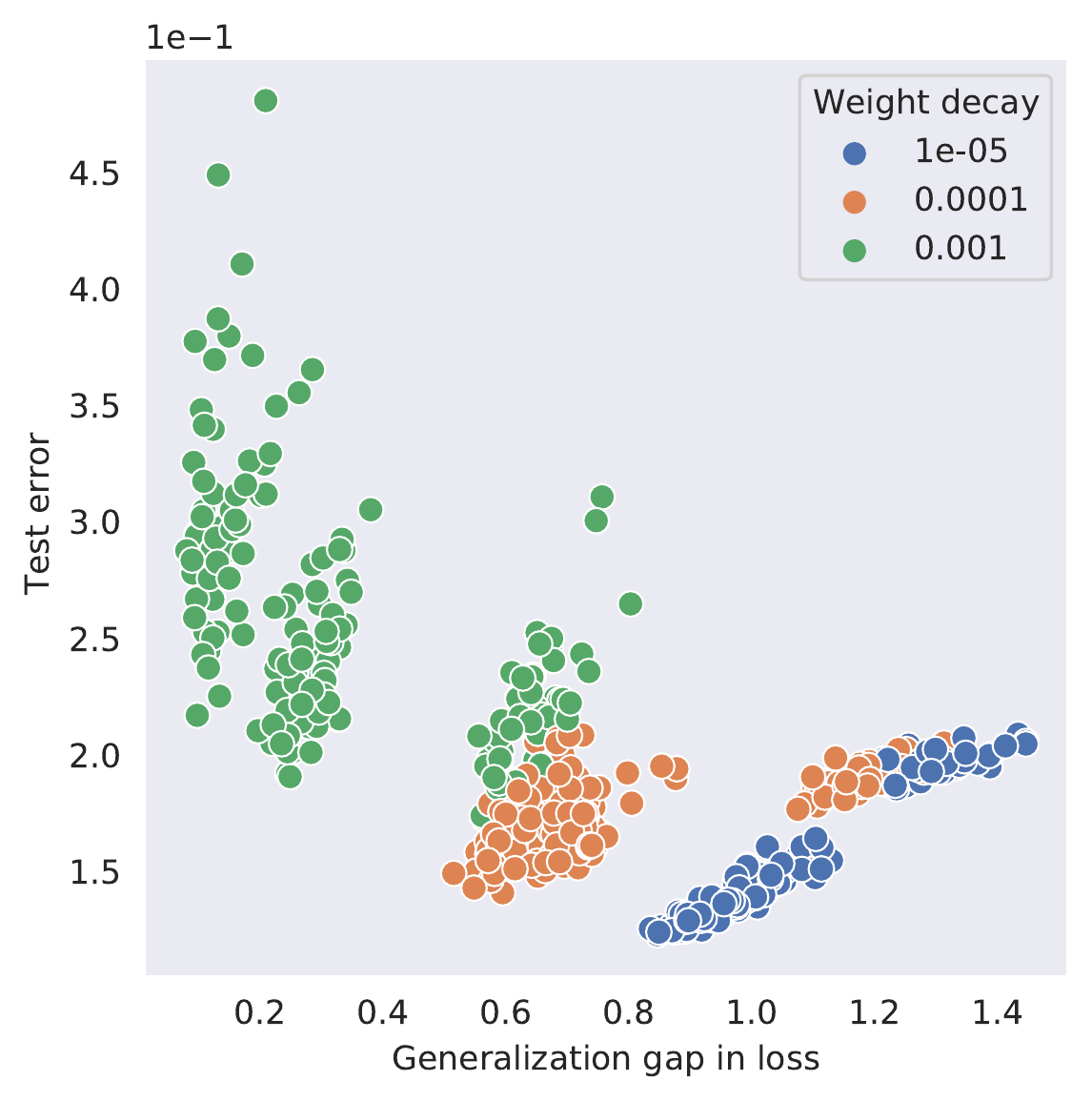}
    
    \includegraphics[width=0.3\textwidth]{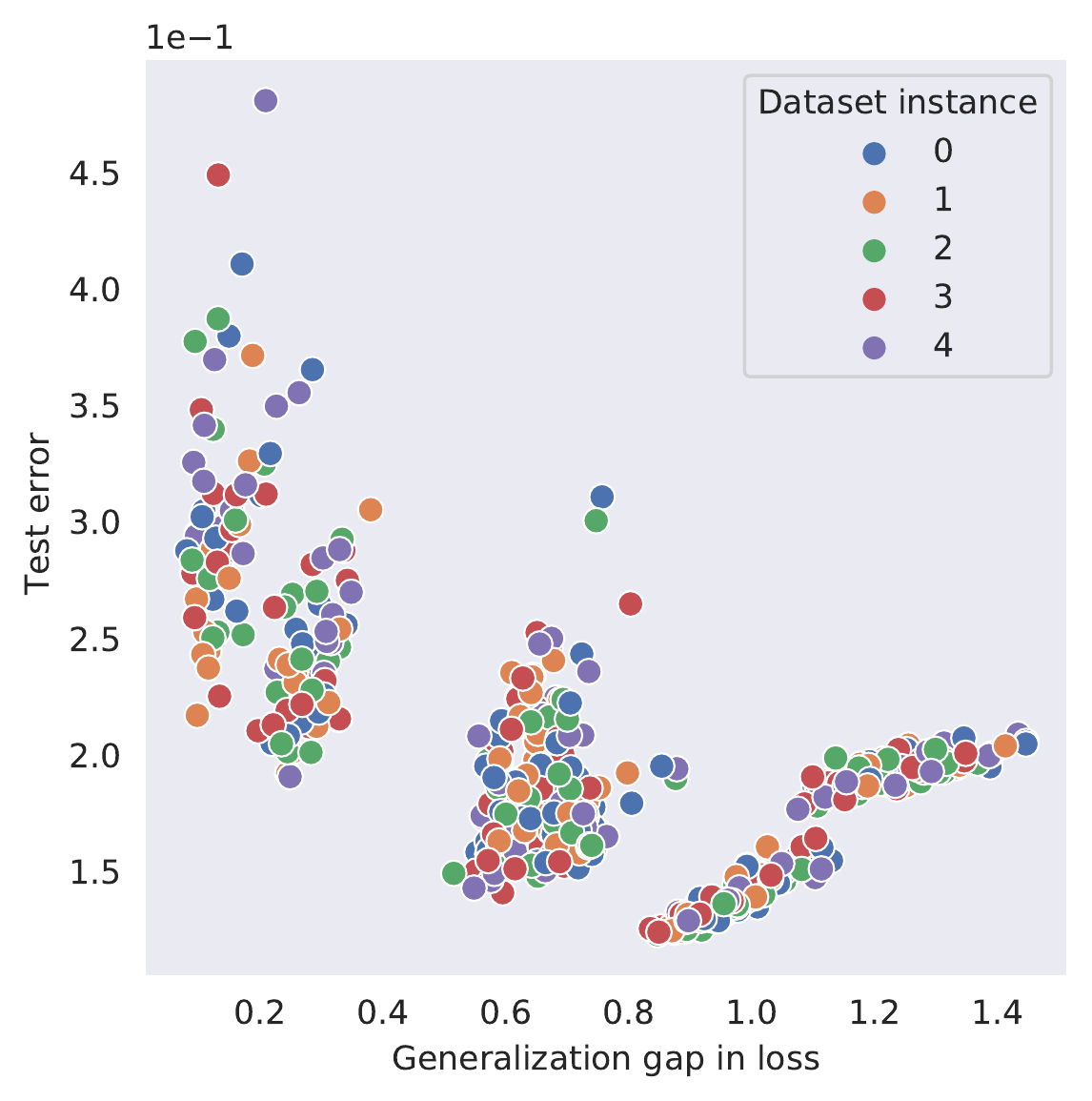}%
    \includegraphics[width=0.3\textwidth]{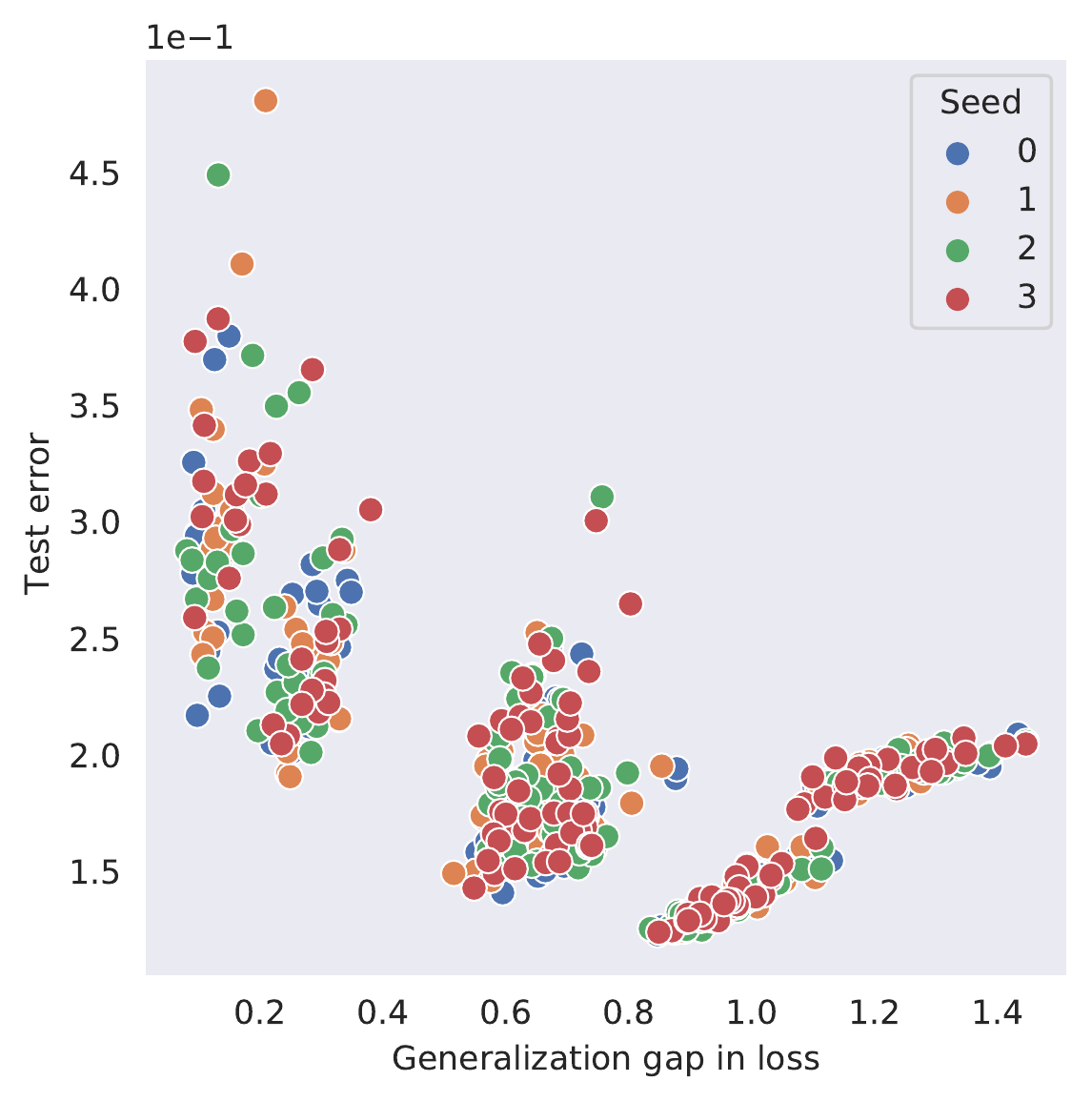}
    
    \caption{Illustration of performance-based clustering behaviour that emerged from training, attributed mostly to batch size and weight decay.}
    \label{fig:exp2_clustering}
\end{figure}

\begin{table}[]
    \centering
    \begin{tabular}{l c c c c c}
\toprule
Layer index $l$: & 1 & 2 & 3 & 4 & 5 \\
\midrule
$\breve{I}(X; Z_l^s)$   & 5.9422E+04& 1.8236E+04& 1.5514E+04& 7.7865E+03& 5.7027E+03 \\
$\breve{I}(X; Z_l^s | Y)$       & 5.7783E+04& 1.7790E+04& 1.5429E+04& 7.6948E+03& 5.3298E+03 \\
$\bar{I}({S}; \theta_l^{S})$      & 0.0000E+00& 3.7136E+03& 9.4890E+05& 2.1247E+06& 5.8244E+06 \\
$\bar{I}({S}; \theta_l^{S}) + \breve{I}(X; Z_l^s | Y)$    & 5.7783E+04& 2.1504E+04& 9.6433E+05& 2.1324E+06& 5.8297E+06 \\
\bottomrule
    \end{tabular}
    \caption{Example values of metrics, for best performing model by test loss (PreResNet56, batch size 128, weight decay 0.001). $l=1$ is the input layer.}
    \label{tab:exp2_values}
\end{table}

\begin{table}[]
    \centering
    \begin{tabular}{l H c c c c c c}
    \toprule
    & Layer  
& \multicolumn{2}{c}{Spearman corr.} & \multicolumn{2}{c}{Pearson corr.} & \multicolumn{2}{c}{Kendall corr.} \\
Generalization gap: & & Loss & Error & Loss & Error & Loss & Error  \\
\midrule

$\breve{I}(S; \theta_{D+1}^{S})$ & - & 0.4688 & 0.3112 & 0.2512 & 0.0775 & 0.2121 & 0.1208 \\

$\bar{I}(S; \theta_{D+1}^{S})$ & - & 0.5370 & 0.3800 & 0.2924 & 0.1218 & 0.2442 & 0.1526 \\
\bottomrule

    \end{tabular}
    \caption{Results for model compression metrics. }
    \label{tab:exp2_model_comp}
\end{table}

\begin{table}[]
    \centering
    \begin{tabular}{l c c c c c c c}
\toprule

& Layer  
& \multicolumn{2}{c}{Spearman corr.} & \multicolumn{2}{c}{Pearson corr.} & \multicolumn{2}{c}{Kendall corr.} \\
Generalization gap: & & Loss & Error & Loss & Error & Loss & Error  \\

\midrule
$\hat{I}(X; Z_l^s)$ & $l=1$ & 0.7345 & 0.6156 & 0.5722 & 0.3853 & 0.5174 & 0.4215 \\
$\hat{I}(X; Z_l^s)$ & $l=D$ & 0.7170 & 0.5740 & 0.7063 & 0.5602 & 0.4784 & 0.3721 \\

$\hat{I}(X; Z_l^s | Y)$ & $l=1$ & 0.7199 & 0.6073 & 0.5724 & 0.3854 & 0.5005 & 0.4111 \\
$\hat{I}(X; Z_l^s | Y)$ & $l=D$ & 0.7126 & 0.5691 & 0.7071 & 0.5616 & 0.4768 & 0.3700 \\
\midrule
$\breve{I}(X; Z_l^s)$ & $l=1$ & 0.6765 & 0.5655 & 0.1554 & 0.1328 & 0.4553 & 0.3781 \\
$\breve{I}(X; Z_l^s)$ & $l=D$ & 0.7145 & 0.5602 & 0.7203 & 0.5719 & 0.4461 & 0.3404 \\

$\breve{I}(X; Z_l^s | Y)$ & $l=1$ & 0.6476 & 0.5292 & 0.1557 & 0.1331 & 0.4307 & 0.3504 \\
$\breve{I}(X; Z_l^s | Y)$ & $l=D$ & 0.7004 & 0.5434 & 0.7062 & 0.5560 & 0.4386 & 0.3305 \\
\midrule

    \end{tabular}
    \caption{Results for representation compression metrics for $l \in \{1, D\}$ summarization over layers (MLE selection of $\sigma_l^2$). }
    \label{tab:exp2_mle_ends}
\end{table}

\begin{table}[]
    \centering
    \begin{tabular}{l c c c c c c c}
\toprule

& Layer  
& \multicolumn{2}{c}{Spearman corr.} & \multicolumn{2}{c}{Pearson corr.} & \multicolumn{2}{c}{Kendall corr.} \\
Generalization gap: & & Loss & Error & Loss & Error & Loss & Error  \\

\midrule
$\hat{I}(X; Z_l^s)$ & $l=1$ & -0.0681 & -0.0659 & -0.0645 & -0.0617 & -0.0464 & -0.0447 \\
$\hat{I}(X; Z_l^s)$ & $l=D$ & 0.7019 & 0.5452 & 0.4596 & 0.2845 & 0.4329 & 0.3293 \\

$\hat{I}(X; Z_l^s | Y)$ & $l=1$ & 0.0050 & 0.0032 & 0.0146 & 0.0191 & 0.0035 & 0.0022 \\
$\hat{I}(X; Z_l^s | Y)$ & $l=D$ & 0.6977 & 0.5403 & 0.4544 & 0.2798 & 0.4302 & 0.3268 \\
\midrule
$\breve{I}(X; Z_l^s)$ & $l=1$ & -0.0196 & -0.0080 & -0.0070 & -0.0029 & -0.0139 & -0.0051 \\
$\breve{I}(X; Z_l^s)$ & $l=D$ & 0.7314 & 0.5848 & 0.5109 & 0.3350 & 0.4645 & 0.3624 \\

$\breve{I}(X; Z_l^s | Y)$ & $l=1$ & 0.0005 & -0.0060 & 0.0036 & -0.0001 & 0.0005 & -0.0038 \\
$\breve{I}(X; Z_l^s | Y)$ & $l=D$ & 0.7033 & 0.5465 & 0.4587 & 0.2846 & 0.4342 & 0.3304 \\
\bottomrule
    \end{tabular}
    \caption{Results for metrics for $l \in \{1, D\}$ summarization over layers (adaptive KDE selection of $\sigma_l^2$). }
    \label{tab:exp2_kde_ends}
\end{table}

\begin{table}[]
    \centering
    \begin{tabular}{l c c c c c c c}
\toprule

& Layer  
& \multicolumn{2}{c}{Spearman corr.} & \multicolumn{2}{c}{Pearson corr.} & \multicolumn{2}{c}{Kendall corr.} \\
Generalization gap: & & Loss & Error & Loss & Error & Loss & Error  \\
\midrule

$\hat{I}(X; Z_l^s)$ & Mean & 0.7598 & 0.6184 & 0.5781 & 0.3911 & 0.4970 & 0.3936 \\
$\hat{I}(X; Z_l^s)$ & Max & 0.7345 & 0.6156 & 0.5722 & 0.3853 & 0.5174 & 0.4215 \\
$\hat{I}(X; Z_l^s)$ & Min & 0.7170 & 0.5740 & 0.7063 & 0.5602 & 0.4784 & 0.3721 \\
$\hat{I}(X; Z_l^s | Y)$ & Mean & 0.7720 & 0.6378 & 0.5790 & 0.3920 & 0.5186 & 0.4145 \\
$\hat{I}(X; Z_l^s | Y)$ & Max & 0.7049 & 0.5814 & 0.5723 & 0.3853 & 0.4712 & 0.3785 \\
$\hat{I}(X; Z_l^s | Y)$ & Min & 0.7137 & 0.5701 & 0.7112 & 0.5679 & 0.4775 & 0.3697 \\
\midrule

$\breve{I}(X; Z_l^s)$ & Mean & 0.8481 & 0.7410 & 0.2116 & 0.1831 & 0.6425 & 0.5436 \\
$\breve{I}(X; Z_l^s)$ & Max & 0.6765 & 0.5655 & 0.1554 & 0.1328 & 0.4553 & 0.3781 \\
$\breve{I}(X; Z_l^s)$ & Min & 0.7145 & 0.5602 & 0.7203 & 0.5719 & 0.4461 & 0.3404 \\
$\breve{I}(X; Z_l^s | Y)$ & Mean & 0.8481 & 0.7406 & 0.2140 & 0.1853 & 0.6427 & 0.5435 \\
$\breve{I}(X; Z_l^s | Y)$ & Max & 0.6486 & 0.5297 & 0.1557 & 0.1331 & 0.4316 & 0.3511 \\
$\breve{I}(X; Z_l^s | Y)$ & Min & 0.7004 & 0.5434 & 0.7062 & 0.5560 & 0.4386 & 0.3305 \\
\midrule

$\breve{I}({S}; \theta_l^{S}) + \hat{I}(X; Z_l^s)$ & Mean & 0.4546 & 0.3055 & 0.1867 & 0.0160 & 0.2304 & 0.1398 \\
$\breve{I}({S}; \theta_l^{S}) + \hat{I}(X; Z_l^s)$ & Max & 0.5111 & 0.3609 & 0.2572 & 0.0858 & 0.2634 & 0.1715 \\
$\breve{I}({S}; \theta_l^{S}) + \hat{I}(X; Z_l^s)$ & Min & 0.8134 & 0.6906 & 0.5363 & 0.3729 & 0.5840 & 0.4870 \\
$\breve{I}({S}; \theta_l^{S}) + \hat{I}(X; Z_l^s | Y)$ & Mean & 0.4543 & 0.3052 & 0.1858 & 0.0152 & 0.2305 & 0.1397 \\
$\breve{I}({S}; \theta_l^{S}) + \hat{I}(X; Z_l^s | Y)$ & Max & 0.5112 & 0.3609 & 0.2572 & 0.0858 & 0.2637 & 0.1718 \\
$\breve{I}({S}; \theta_l^{S}) + \hat{I}(X; Z_l^s | Y)$ & Min & 0.8136 & 0.6949 & 0.5193 & 0.3591 & 0.5817 & 0.4871 \\

\midrule
$\breve{I}({S}; \theta_l^{S}) + \breve{I}(X; Z_l^s)$ & Mean & 0.4513 & 0.3026 & 0.1832 & 0.0132 & 0.2305 & 0.1398 \\
$\breve{I}({S}; \theta_l^{S}) + \breve{I}(X; Z_l^s)$ & Max & 0.5112 & 0.3609 & 0.2572 & 0.0858 & 0.2636 & 0.1715 \\
$\breve{I}({S}; \theta_l^{S}) + \breve{I}(X; Z_l^s)$ & Min & 0.8489 & 0.7353 & 0.8459 & 0.7216 & 0.6354 & 0.5386 \\
$\breve{I}({S}; \theta_l^{S}) + \breve{I}(X; Z_l^s | Y)$ & Mean & 0.4513 & 0.3026 & 0.1832 & 0.0132 & 0.2301 & 0.1397 \\
$\breve{I}({S}; \theta_l^{S}) + \breve{I}(X; Z_l^s | Y)$ & Max & 0.5113 & 0.3609 & 0.2572 & 0.0858 & 0.2638 & 0.1716 \\
$\breve{I}({S}; \theta_l^{S}) + \breve{I}(X; Z_l^s | Y)$ & Min & 0.8434 & 0.7313 & 0.8437 & 0.7195 & 0.6270 & 0.5332 \\

\midrule
$\bar{I}({S}; \theta_l^{S}) + \hat{I}(X; Z_l^s)$ & Mean & 0.4770 & 0.3244 & 0.2901 & 0.1155 & 0.2510 & 0.1568 \\
$\bar{I}({S}; \theta_l^{S}) + \hat{I}(X; Z_l^s)$ & Max & 0.5709 & 0.4205 & 0.2993 & 0.1311 & 0.2878 & 0.1946 \\
$\bar{I}({S}; \theta_l^{S}) + \hat{I}(X; Z_l^s)$ & Min & 0.7898 & 0.6548 & 0.6706 & 0.4929 & 0.5432 & 0.4412 \\
$\bar{I}({S}; \theta_l^{S}) + \hat{I}(X; Z_l^s | Y)$ & Mean & 0.4764 & 0.3241 & 0.2869 & 0.1128 & 0.2489 & 0.1558 \\
$\bar{I}({S}; \theta_l^{S}) + \hat{I}(X; Z_l^s | Y)$ & Max & 0.5709 & 0.4205 & 0.2993 & 0.1311 & 0.2882 & 0.1952 \\
$\bar{I}({S}; \theta_l^{S}) + \hat{I}(X; Z_l^s | Y)$ & Min & 0.7917 & 0.6595 & 0.6490 & 0.4753 & 0.5447 & 0.4450 \\

\midrule

$\bar{I}({S}; \theta_l^{S}) + \breve{I}(X; Z_l^s)$ & Mean & 0.4429 & 0.2910 & 0.2785 & 0.1060 & 0.2353 & 0.1432 \\
$\bar{I}({S}; \theta_l^{S}) + \breve{I}(X; Z_l^s)$ & Max & 0.5707 & 0.4205 & 0.2993 & 0.1311 & 0.2880 & 0.1946 \\
$\bar{I}({S}; \theta_l^{S}) + \breve{I}(X; Z_l^s)$ & Min & \textbf{0.8635} & \textbf{0.7576} & \textbf{0.8493} & \textbf{0.7544} & \textbf{0.6660} & \textbf{0.5684} \\
$\bar{I}({S}; \theta_l^{S}) + \breve{I}(X; Z_l^s | Y)$ & Mean & 0.4429 & 0.2908 & 0.2783 & 0.1059 & 0.2349 & 0.1426 \\
$\bar{I}({S}; \theta_l^{S}) + \breve{I}(X; Z_l^s | Y)$ & Max & 0.5711 & 0.4204 & 0.2993 & 0.1311 & 0.2886 & 0.1945 \\
$\bar{I}({S}; \theta_l^{S}) + \breve{I}(X; Z_l^s | Y)$ & Min & \textbf{0.8632} & \textbf{0.7576} & \textbf{0.8511} & \textbf{0.7562} & \textbf{0.6626} & \textbf{0.5664} \\

\bottomrule

    \end{tabular}
    \caption{Results for metrics for mean, min, max summarization over layers (MLE  selection of $\sigma_l^2$). Best metrics highlighted.}
    \label{tab:exp2_mle}
\end{table}

%%%% linear adaptive KDE

\begin{table}[]
    \centering
    \begin{tabular}{l c c c c c c c}
\toprule

& Layer  
& \multicolumn{2}{c}{Spearman corr.} & \multicolumn{2}{c}{Pearson corr.} & \multicolumn{2}{c}{Kendall corr.} \\
Generalization gap: & & Loss & Error & Loss & Error & Loss & Error  \\
\midrule

$\hat{I}(X; Z_l^s)$ & Mean & 0.7322 & 0.5837 & 0.6823 & 0.5067 & 0.4606 & 0.3605 \\
$\hat{I}(X; Z_l^s)$ & Max & 0.7805 & 0.6521 & 0.6848 & 0.5088 & 0.5299 & 0.4360 \\
$\hat{I}(X; Z_l^s)$ & Min & 0.7726 & 0.6450 & 0.8131 & 0.7004 & 0.5126 & 0.4205 \\
$\hat{I}(X; Z_l^s | Y)$ & Mean & 0.7323 & 0.5833 & 0.6863 & 0.5110 & 0.4604 & 0.3598 \\
$\hat{I}(X; Z_l^s | Y)$ & Max & 0.7865 & 0.6548 & 0.6920 & 0.5160 & 0.5378 & 0.4380 \\
$\hat{I}(X; Z_l^s | Y)$ & Min & 0.7731 & 0.6451 & 0.8175 & 0.7041 & 0.5131 & 0.4221 \\
\midrule
$\breve{I}(X; Z_l^s)$ & Mean & 0.7401 & 0.5933 & 0.7078 & 0.5334 & 0.4739 & 0.3727 \\
$\breve{I}(X; Z_l^s)$ & Max & 0.7452 & 0.6264 & 0.6891 & 0.5116 & 0.4927 & 0.4109 \\
$\breve{I}(X; Z_l^s)$ & Min & 0.7981 & 0.6720 & \textbf{0.8515} & \textbf{0.7262} & 0.5490 & 0.4500 \\
$\breve{I}(X; Z_l^s | Y)$ & Mean & 0.7375 & 0.5894 & 0.7010 & 0.5263 & 0.4701 & 0.3685 \\
$\breve{I}(X; Z_l^s | Y)$ & Max & 0.7449 & 0.6177 & 0.7056 & 0.5289 & 0.4924 & 0.4053 \\
$\breve{I}(X; Z_l^s | Y)$ & Min & 0.7686 & 0.6270 & 0.8130 & 0.6699 & 0.5093 & 0.4055 \\
\midrule

$\breve{I}({S}; \theta_l^{S}) + \hat{I}(X; Z_l^s)$ & Mean & 0.4555 & 0.3070 & 0.1856 & 0.0153 & 0.2334 & 0.1425 \\
$\breve{I}({S}; \theta_l^{S}) + \hat{I}(X; Z_l^s)$ & Max & 0.5133 & 0.3633 & 0.2580 & 0.0865 & 0.2650 & 0.1735 \\
$\breve{I}({S}; \theta_l^{S}) + \hat{I}(X; Z_l^s)$ & Min & 0.8112 & 0.7033 & 0.8272 & 0.7243 & 0.5775 & 0.4914 \\
$\breve{I}({S}; \theta_l^{S}) + \hat{I}(X; Z_l^s | Y)$ & Mean & 0.4542 & 0.3058 & 0.1849 & 0.0146 & 0.2322 & 0.1414 \\
$\breve{I}({S}; \theta_l^{S}) + \hat{I}(X; Z_l^s | Y)$ & Max & 0.5126 & 0.3625 & 0.2578 & 0.0863 & 0.2643 & 0.1730 \\
$\breve{I}({S}; \theta_l^{S}) + \hat{I}(X; Z_l^s | Y)$ & Min & \textbf{0.8205} & \textbf{0.7203} & 0.8281 & \textbf{0.7313} & \textbf{0.5913} & \textbf{0.5108} \\
\midrule

$\breve{I}({S}; \theta_l^{S}) + \breve{I}(X; Z_l^s)$ & Mean & 0.4602 & 0.3115 & 0.1878 & 0.0172 & 0.2378 & 0.1468 \\
$\breve{I}({S}; \theta_l^{S}) + \breve{I}(X; Z_l^s)$ & Max & 0.5146 & 0.3647 & 0.2588 & 0.0871 & 0.2666 & 0.1748 \\
$\breve{I}({S}; \theta_l^{S}) + \breve{I}(X; Z_l^s)$ & Min & 0.8014 & 0.6854 & 0.8423 & 0.7107 & 0.5658 & 0.4760 \\
$\breve{I}({S}; \theta_l^{S}) + \breve{I}(X; Z_l^s | Y)$ & Mean & 0.4598 & 0.3111 & 0.1874 & 0.0168 & 0.2374 & 0.1464 \\
$\breve{I}({S}; \theta_l^{S}) + \breve{I}(X; Z_l^s | Y)$ & Max & 0.5143 & 0.3643 & 0.2584 & 0.0868 & 0.2661 & 0.1743 \\
$\breve{I}({S}; \theta_l^{S}) + \breve{I}(X; Z_l^s | Y)$ & Min & 0.8069 & 0.6913 & \textbf{0.8440} & 0.7135 & 0.5732 & 0.4849 \\
\midrule

$\bar{I}({S}; \theta_l^{S}) + \hat{I}(X; Z_l^s)$ & Mean & 0.4783 & 0.3249 & 0.2868 & 0.1134 & 0.2495 & 0.1577 \\
$\bar{I}({S}; \theta_l^{S}) + \hat{I}(X; Z_l^s)$ & Max & 0.5765 & 0.4250 & 0.3021 & 0.1333 & 0.2971 & 0.2057 \\
$\bar{I}({S}; \theta_l^{S}) + \hat{I}(X; Z_l^s)$ & Min & 0.8031 & 0.6920 & 0.8270 & 0.7149 & 0.5664 & 0.4793 \\
$\bar{I}({S}; \theta_l^{S}) + \hat{I}(X; Z_l^s | Y)$ & Mean & 0.4766 & 0.3242 & 0.2842 & 0.1110 & 0.2468 & 0.1551 \\
$\bar{I}({S}; \theta_l^{S}) + \hat{I}(X; Z_l^s | Y)$ & Max & 0.5759 & 0.4247 & 0.3013 & 0.1326 & 0.2958 & 0.2049 \\
$\bar{I}({S}; \theta_l^{S}) + \hat{I}(X; Z_l^s | Y)$ & Min & \textbf{0.8130} & \textbf{0.7070} & 0.8295 & 0.7203 & \textbf{0.5799} & \textbf{0.4963} \\
\midrule

$\bar{I}({S}; \theta_l^{S}) + \breve{I}(X; Z_l^s)$ & Mean & 0.4864 & 0.3335 & 0.2946 & 0.1205 & 0.2613 & 0.1689 \\
$\bar{I}({S}; \theta_l^{S}) + \breve{I}(X; Z_l^s)$ & Max & 0.5769 & 0.4253 & 0.3048 & 0.1356 & 0.2978 & 0.2067 \\
$\bar{I}({S}; \theta_l^{S}) + \breve{I}(X; Z_l^s)$ & Min & 0.7974 & 0.6749 & 0.8206 & 0.6842 & 0.5608 & 0.4684 \\
$\bar{I}({S}; \theta_l^{S}) + \breve{I}(X; Z_l^s | Y)$ & Mean & 0.4851 & 0.3317 & 0.2933 & 0.1192 & 0.2596 & 0.1674 \\
$\bar{I}({S}; \theta_l^{S}) + \breve{I}(X; Z_l^s | Y)$ & Max & 0.5770 & 0.4253 & 0.3034 & 0.1343 & 0.2981 & 0.2068 \\
$\bar{I}({S}; \theta_l^{S}) + \breve{I}(X; Z_l^s | Y)$ & Min & 0.7989 & 0.6763 & 0.8242 & 0.6891 & 0.5628 & 0.4703 \\
\bottomrule
    \end{tabular}
    \caption{Results for metrics for mean, min, max summarization over layers (adaptive KDE selection of $\sigma_l^2$). Best metrics highlighted.}
    \label{tab:exp2_kde}
\end{table}

\FloatBarrier

\allowdisplaybreaks

\section{Additional Results and Explanations for Theory} \label{app:7}

We present additional results and explanations for theoretical results in \cref{app:7}, full proofs in \cref{app:proofs}, and additional results and details for experimental results in \cref{app:exp1}.

\subsection{On Theorems \ref{thm:1}--\ref{thm:2}} \label{app:12}
The mutual information $I(\phi_{l}^{S}; S)$ in Theorem \ref{thm:2} does not appear in Conjecture \ref{thm:shwartz}. However, the sample complexity bound  in Conjecture \ref{thm:shwartz} is invalid for the setting of learning $\phi^s_l$, because the encoder $\phi_{l} ^{s}$ can overfit to the training data, which was demonstrated  with the counter-example in \citet[Example 3.1]{hafez2020sample}. The mutual information  $I(\phi_{l}^{S}; S)$ is measuring the effect of overfitting the encoder, which is necessary to avoid the counter-example. 

Using additional notation defined in \cref{app:13}, we will discuss the details of  the formulas of  $G_1^{l}(0)$, $\Gcal_2^l$ and $G_3^{l}$ of  Theorems \ref{thm:1} and  $G_1^l(\zeta)$, $\widehat \Gcal_2^{l}$, $\check \Gcal_2^l$, and $G_3^l$  of Theorem \ref{thm:2} in \cref{app:14}.

\subsubsection{Additional Notation} \label{app:13}
We define the variables for the (hidden) input generating process as follows. Each $x\in \Xcal$ is generated with  a hidden function $\chi$  by $
x=\chi(y,\xi^{(y)}), 
$
where $\xi^{(y)}=(\xi_1^{(y)},\dots,\xi_m^{(y)}) \in \hXi_{y} \subseteq\RR^m$ is the nuisance variable. We denote the random variable for  $\xi^{(y)}$ by $\Xi_{y}$; i.e.,  
$
\Xi_{y}(\omega_{y}) = \xi^{(y)}
$
where  $\omega_{y}\in \Omega_{y}$ is the element of the sample space $\Omega_{y}$ of  the nuisance variable, conditioned on $Y=y$. Then, we denote the random variables for $X$ and $Z_l^s$ conditioned on $Y=y$ by $X_{y}$ and $Z_{l,y}^s$: $X_{y}(\omega_{y})=\chi(y, \Xi_y(\omega_{y})) \in \Xcal$ and $Z_{l,y}^s =\phi_l ^{s}\circ X_{y}$. For any $l \in  [D]$ and $y \in \Ycal$, we define the sensitivity $c_{l}^{y}(\phi_l^{s})$ of the trained encoder $\phi_l^s$ with respect to the nuisance variable     $\xi^{(y)}$ by the number such  that for all $i\in [m]$,
$c_{l}^{y}(\phi_l^{s}) \ge \sup_{\xi_{1}^{(y)},\dots,\xi_{i-1}^{(y)},\xi_{i}^{(y)},\tilde \xi_{i}^{(y)},\xi_{i+1}^{(y)},\dots,\xi_{m}^{(y)}}|\log p_y((\phi_l^{s}\circ\chi _{y})(\xi_{1}^{(y)},\dots,\xi_{i-1}^{(y)},\xi_{i}^{(y)},\xi_{i+1}^{(y)},\dots,\xi_{m}^{(y)}))-\log p_y((\phi_l^{s}\circ\chi_{y} )(\xi_{1}^{(y)},\dots,\xi_{i-1}^{(y)},\tilde \xi_{i}^{(y)},\xi_{i+1}^{(y)},\dots,\xi_{m}^{(y)}))|$, where $\chi _{y}(\xi^{(y)})=\chi(y,\xi^{(y)})$ and $p_y(q)= \PP(Z_{l,y}^s=q)$. 

For any $l \in [D+1]$ and $\lambda_{l}>0$, define
\vskip -1.5em \begin{align} \label{e:Cla}
    C_{\lambda_{l},l}=\frac{1}{e^{\lambda_{l} H(\phi_{l} ^{S})}}\sum_{q\in \Mcal_l}(\PP(\phi_{l} ^{S}=q))^{1-\lambda_{l}},
\end{align} \vskip -1em 
where $H(\phi_{l} ^{S})$ is the entropy of the random variable $\phi_{l}^{S}$. We define the set of the latent variable per class by
$
\Zcal_{l,y}^s =\{(\phi_l^{s}\circ \chi_{y})(\xi^{(y)}):\ \xi^{(y)}\in \hXi_y \}.
$
For any $\gamma>0$, we then define a  (typical) subset $\Zcal_{\gamma,l,y}^{s} $ (of the set $\Zcal_{l,y}^s$) by
$
\Zcal_{\gamma,y}^{s,l} =\{z \in \Zcal_{l,y}^s  : - \log \PP_{}(Z_{l,y}^{s}=z)- H_{}(Z_{l,y}^{s})  \le c_{l}^{y}(\phi_l^{s}) \sqrt{\frac{m\ln(\sqrt{n}/\gamma)}{2}} \}.
$
Let us write the element of $\Zcal_{\gamma,y}^{s,l}$ by $\Zcal_{\gamma,y}^{s,l} =\{a_1^{l,y},\dots,a_{T_{y}^l}^{l,y}\}$ where $T_{y}^l =|\Zcal_{\gamma,y}^{s,l} |$. Finally, define maximum training loss $\Lcal(f^{s})= \max_{i \in \{1,\dots,n\}} \ell(f^{s}(x_{i}),y_i)$.

\subsubsection{Details of Other Terms} \label{app:14}

In Theorem \ref{thm:1}, we have that $G_1^{l}(q) =\frac{\Lcal(f^s)\sqrt{2\gamma_{l}|\Ycal|}}{n^{1/4}} \sqrt{q+\ln (2| \Ycal|/\delta)} 
+\gamma_{l} \Rcal(f^s)$,  
$\Gcal_2^l= G_2 ^{l}\ln(2) +\ln(2| \Ycal|/\delta)$,  $G_3 ^{l}= \max\limits_{y \in \Ycal}  \sum_{k=1}^{T_{y}^l}\ell( g_{l}^{s}(a_k^{l,y}),y) \sqrt{2|\Ycal|\PP_{}(Z_{l,y}^s=a_{k}^{l,y})}$, and  $G_2 ^{l}=\EE_y[c_{l}^{y}(\phi_{l}^{s})]\sqrt{\frac{m\ln(\frac{\sqrt{n}}{\gamma_{l}})}{2}}+H_{}(Z_{l}^s|X_{},Y)$.
Theorem \ref{thm:1} holds for any (fixed) $\gamma_l>0$.

In Theorem \ref{thm:2},   the definitions of   $G_1^l(q), G_2^l, G_3^l$  are the same as in  Theorem \ref{thm:1}, and we have that $\zeta=(I_{}(\phi_{l}^{S}; S)+G_{4}^{l})\ln (2)+\ln(2|\bD|)$,  
$\widehat \Gcal_2 ^l=\left(G_2 ^{l}+G_{4}^l\right)\ln(2)+\ln(4| \Ycal||\bD|/\delta)$, $\check \Gcal_2^l=G_{4}^{l} \ln(2)+\ln(2/\delta)$, and $G_4^l =\frac{1}{\lambda_{l}}\ln \frac{C_{\lambda_{l},l}|\bD| }{\delta} +H_{}(\phi_{l}^{S}|S)$. Theorem \ref{thm:2}  holds for any (fixed)  $\gamma_{l}>0$ and $\lambda_l>0$ for all $l \in \bD$. 

Proposition \ref{prop:1} below shows that $G_3^l$ can be bounded by a constant value, which is much smaller than and independent of the size of the  set  $\Zcal_{\gamma,l,y}^{s}$. \begin{proposition} \label{prop:1}
Let $l \in \{1,\dots,D\}$.
Let $v_k(y)=\ell( g_{l}^{s}(a_{j_k}^{l,y}),y)^{2}\PP(Z_{l,y}^{s}=a_{j_k}^{l,y})$ where $k \mapsto j_k$ is a permutation of the index such that $v_1(y)\ge v_2(y) \ge \cdots \ge v_{T_{y}^l}(y)$. If there exist some constants  $\alpha_{y}\ge 1$ and $\beta_{y},C_{y}>0$ such that $v_k(y)\le C _{y}e^{-(k/\beta_{y})^{\alpha_y}}$, then 
\vskip -1.8em \begin{align}
G_3 ^{l}\le \sqrt{2|\Ycal|}\max_{y \in \Ycal} \left(\sqrt{v_1(y)} \ceil{\tbeta_{y}}  + (C _{y}\tbeta_y)/(\alpha _{y}e)\right),
\end{align} \vskip -1em
without  assuming that $\phi^s_l$ is fixed independently of the training dataset $s$, where   $\tbeta_{y}= 2^{1/\alpha_y}\beta_{y}$.
\end{proposition}
\begin{proof} %\vspace{-10pt}
The proof is provided in \cref{app:4}.
\end{proof}  

Proposition \ref{prop:new:1} shows that the value of $\ln C_{\lambda_{l},l}$ (recall from \eqref{e:Cla}) in the formula of $G_4^l$  can be bounded by a constant value independently of $\ln |\Mcal_l|$ and is much smaller than $\ln |\Mcal_l|$ and $H(\phi_{l}^{S})$: 
\begin{proposition} \label{prop:new:1}
Let  $l \in \{1,\dots,D+1\}$. We denote $N=|\Mcal_l|$, and enumerate $\Mcal_l$ as $q_1,q_2,q_3,\cdots, q_N$ with decreasing probability, i.e. $p_i=\PP(\phi_{l}^{S}=q_i)$ and $p_1\geq p_2\geq \cdots\geq p_N$.
\begin{enumerate}
    \item \vskip -0.5em 
    If $p_i$ decays sufficiently fast, i.e., $p_i\leq C/i^\alpha$ with some $\alpha>1$ and $C\geq 1$, then for $0<\lambda_l<1-1/\alpha$, both the entropy $H(\phi_{l}^{S})$ and $C_{\lambda_{l},l}$ are bounded and independent of $N$:    
    \vskip -1.5em \begin{align*}
    &H(\phi_{l}^{S})\leq 1+C\alpha\left(\frac{\ln(2)}{2^\alpha}+\frac{\ln(3)}{3^\alpha}+\frac{3^{1-\alpha}((a-1)\ln(3)+1)}{(\alpha-1)^2}\right),\\
    &C_{\lambda_{l},l}\leq C^{1-\lambda_l}\frac{\alpha(1-\lambda_l)}{\alpha(1-\lambda_l)-1}.
    \end{align*} \vskip -1em
    \item \vskip -0.5em 
    If $p_i$ decays slowly, i.e., $p_i=c_i/(Z i^\alpha )$ with $0\leq \alpha<1$ and $0< c\leq c_i\leq C$
    where $Z$ is the normalization constant, then the entropy $H(\phi_{l}^{S})$ grows as $\ln(N)+\Ocal(1)$ where $\Ocal(1)$ depends only on $\alpha$, but $C_{\lambda_{l},l} $ is bounded and independent of $N$ as:
    \vskip -1.8em \begin{align*}
   C_{\lambda_{l},l}\leq \left(\ln(1-(1-\lambda_l)\alpha)-(1-2\lambda_l)\ln(1-\alpha)\right)+ (2-\lambda_l)\ln(C/c)+\frac{C}{c(1-\alpha)}.
    \end{align*} \vskip -1em
\end{enumerate}
\end{proposition}
\begin{proof}  %\vspace{-20pt}
The proof is provided in \cref{app:5}.
\end{proof} %\vskip -0.8em

We now discuss the factors $G_1^l(q)$ and $G_2^l$ in Theorem \ref{thm:1}.  The formula of $G_1^l(q)$ is simplified as $G_1^l(q)=\gamma_l$ for any $q \in \RR_{\ge 0}$ in a  common scenario of deep learning where we use  the 0-1 loss (to measure generalization) and have zero training error. This is because  $\Lcal(f^s)=0$ and $\Rcal(f^s) \le 1$ for the  scenario. In the formula of $G_2$,  we have $H_{}(Z_{l}^s|X_{},Y)=0$ if the function $\phi^s_l$ is deterministic, which is the typical case for deep neural networks, because $\phi^s_l$ is the function used at inference or test time as opposed to training time (when dropout for example can be used). When the function $\phi^s_l$ is stochastic, we have $H(Z_{l}^s|X,Y) = \Ocal(1)$ as $n\rightarrow \infty$. The networks can be stochastic, for example, with randomization defenses against adversarial attacks \citep{xie2018mitigating,pinot2019theoretical,pinot2020randomization,levine2020robustness} or noise injections \citep{goldfeld2019estimating}. The value of $\EE_{y}[c_{l}^{y}(\phi_l^{s})]$ in the formula of $G_2$ measures the  sensitivity with respect to the nuisance variable $\xi^{(y)}$; i.e., minimizing this value should result in better generalization, which is consistent with Theorem \ref{thm:1}. The sensitivity $c_{l}^{y}(\phi_l^{s})$  is a measure on the single final encoder $\phi_{l}^s$; i.e., increasing the complexity of hypothesis spaces does not imply an increase in this value.

All the discussions and results, including Proposition \ref{prop:1},  on the factors $G_1^l(q), G_2^l$, and $G_3^l$ in Theorem \ref{thm:1} hold true for these factors in Theorem \ref{thm:2} (because we do not assume the use of  the fixed encoder for these). Accordingly, we now discuss the new factor, $G_4$, in Theorem \ref{thm:2}.
The value of $\ln C_{\lambda_{l},l}$ in the formula of $G_4^l$   is analyzed in Proposition \ref{prop:new:1}. In the formula of $G_4$, $H_{}(\phi_{l}^{S}|S)$ measures the randomness of algorithm $\Acal_l$. To understand this, let us consider a real-world experiment with a coin tossing. We can model the coin tossing by a stochastic model (by saying that coin tossing has 50-50 chance of getting heads and tails) or by a deterministic model to predict the exact outcome with an exact initial condition of the physical system. similarly, we can model a single real-world algorithm with a stochastic model $\Acal_l$ (by saying that something has some random chances) or a deterministic model $\Acal_l$ with the exact initial condition, which is the random seed in the numerical experiments. That is, as in any mathematical theories and symbols, $\Acal_l$ is a theoretical placeholder with its mathematical definition; i.e., $\Acal_l$ does \textit{not} have one-to-one correspondence to a real-world algorithm implemented in experiments. In other words, given a single real-world algorithm, there are many different ways to model the real-world algorithm and different ways result in different $\Acal_l$. For example, let us fix the real-world algorithm implemented in experiments to be one with dropout \citep{srivastava2014dropout} and stochastic gradient descent (SGD). At this point, $\Acal_l$ in Theorem \ref{thm:2} is not fully determined yet and we can choose $\Acal_l$ differently for the same real-world algorithm by modeling the real-world differently. For instance, we can model the one with dropout and SGD as a stochastic algorithm or as a deterministic algorithm given a fixed random seed in practice. Thus, we can set $\Acal_l$ in \ref{thm:2} to be either a stochastic algorithm or a deterministic algorithm for the exact same real-world and the same fixed algorithm implemented in experiments. If we set $\Acal_l$ to be a deterministic algorithm with a fixed seed, then we have $H_{}(\phi_{l}^{S}|S)=0$. If we set $\Acal_l$ to be a stochastic algorithm, then we increase $H(\phi_{l}^{S}|S)$ but we can potentially decrease $I(\phi_{l}^{S}; S)$ since the extra randomness can potentially reduce the mutual information of $\phi_{l}^{S}$ and $S$. Thus, there is a trade-off in how we model the real-world via $\Acal_l$ and we cannot reduce the bound arbitrarily. Our theorems allow to instantiate our bounds with both deterministic and stochastic views of the learning algorithms, without changing the real-world algorithms.

More generally, a randomized algorithm can be defined as a deterministic algorithm with an additional input that consists of a sequence of random bits \citep{hromkovivc2004randomized}. Here, the sequence of random bits corresponds to the sequence of random seeds in the numerical experiments with SGD. In other words, on the one hand, we can model SGD as a stochastic process when we analyze a set of experiments with SGD over a set of random seeds that are generated randomly. On the other hand, we can model SGD as a deterministic process when we analyze one experiment with SGD for one random seed. Moreover, as \citet{hromkovivc2004randomized} explains, we can bridge those two cases by modeling SGD as a deterministic algorithm with its additional inputs being the seed; then (1) it is deterministic for each seed, and (2) we can recover the stochastic model by considering a sequence of the randomly generated seeds. If we analyze one experiment of SGD for one random seed, then we have a deterministic algorithm, and a typical worst-case analysis provides a guarantee on the SGD with the worst seed. But, if we analyze a set of experiments of SGD over a set of random seeds that are generated randomly, then we have a stochastic algorithm, and we can analyze its expected performance or high-probability guarantee w.r.t. the random seeds.

To formally treat the learning algorithm $\Acal_l$ as a stochastic one, we replace $S$ with $\tilde {S}$ in \cref{eq:new:8} where $\tilde {S}(\omega,\omega')=(S(\omega),\omega')$ with $\omega$ and $\omega'$ being elements of the sample spaces for $S$ and $\Acal_l$ respectively. Similarly,  when the encoder $\phi_{l}^{s} $ is stochastic, we replace $X$ with $\tilde X$ in \cref{eq:new:7} where $\tilde X(\omega,\omega')=(X(\omega),\omega')$ with $\omega$ and $\omega'$ being elements of the sample spaces for $X$ and $\phi_{l}^s$ respectively. All of the proofs work in any of  these cases.

\subsection{On the application to the case of infinite mutual information} \label{app:10}
Section \ref{sec:7} discusses a way to apply a sample complexity bound with mutual information to the cases of infinite mutual information by using binning methods. This section considers a more general method of computing the mutual information to achieve the following goal: we demonstrate that a theoretical work on a bound with mutual information is an important and sensible research area more generally beyond our paper, even for the cases of infinite mutual information. This section also provides theoretical justifications on previous methods of computing mutual information even for the case of deterministic neural networks with continuous random variables with injective activations  \citep{shwartz2017opening,saxe2019information,chelombiev2018adaptive}. We use the notation of $\phi^{s}=\phi_l^s$ and $g^{s}=g_l^s$ for a fixed $l$ in this subsection.  

This is based on the following simple observation:
we can bound the generalization error of a given encoder, $\Grm[\tphi^s]$, by using the generalization bound of another  encoder, $\Brm_{\delta}[\phi^{s}]$, if we  add the term measuring a distance between the two encoders, $ \Drm(\phi^s,\tphi^s)$. This is formalized in Remark \ref{remark:2}: 
\begin{remark} \label{remark:2}
Define $\Grm[\phi^s]=\EE_{X,Y}[\ell_{g}(\phi^s(X),Y)] - \frac{1}{n} \sum_{i=1}^n \ell_{g}(\phi^s(x_{i}), y_i)$ and $\Lrm[\phi^s]=\ell_{g}(\phi^s(X),Y)$ where $\ell_{g}(\phi^s(X),Y)=\ell((g^{s}\circ\phi^s)(X),Y)$. Suppose that for any $\delta>0$,  $\PP(\Grm[\phi^s]\le \Brm_{\delta}[\phi^s])\ge 1 - \delta$ for some functional $\Brm_{\delta}$ and that $\PP_{X,Y}(|\Lrm[\phi^s]-\Lrm[\tphi^s]|\le \Drm(\phi^s,\tphi^s))=1$ for some functional $ \Drm$. Then, for any $\delta>0$, with at least probability  $1 - \delta$, 
\begin{align}
\Grm[\tphi^s]\le \Brm_{\delta}[\phi^s]+2 \Drm(\phi^s,\tphi^s). 
\end{align}
\end{remark}
\begin{proof}
Since $\PP(|\Lrm[\phi^s]-\Lrm[\tphi^s]|\le \Drm(\phi^s,\tphi^s))=1$, we have with probability one,
$
\Grm[\tphi^s] \le\Grm[\phi^s]+2 \Drm(\phi^s,\tphi^s). 
$
Since  $\PP(\Grm[\phi^s]\le \Brm[\phi^s])\ge 1 - \delta$, we have with at least probability  $1 - \delta$,
$\Grm[\tphi^s] \le\Grm[\phi^s]+2 \Drm(\phi^s,\tphi^s) \le\Brm_\delta [\phi^s]+2 \Drm(\phi^s,\tphi^s)$.
\end{proof}
Here, let us set  $\Brm_{\delta}[\phi^s]$ to be  a generalization bound on $\phi^s$ with mutual information. 
Then, given an original model $\tphi^s$, its direct bound $\Brm_{\delta}[\tphi^s]$ can be infinite since its mutual information can be infinite, for example, for deterministic neural networks $\tphi^s$ with sigmoid activations for continuous random variables. However, instead of using its direct bound  $\Brm_{\delta}[\tphi^s]$, we can bound the generalization error of the original model $\tphi^s$ by invoking Remark \ref{remark:2} to use the bound  $\Brm_{\delta}[\phi^s]$  of another model $\phi^s\neq \tphi^s$ such that $\phi^s$ has finite mutual information and $\Drm(\phi^s,\tphi^s)$ is small. 

Indeed, this is a theoretical formalization of what is implicitly done  in practice when we compute mutual information of deterministic models. That is, in practice, we often compute the mutual information of the original model $\tphi^s$ by computing the mutual information of another model $\phi^s$ where $\phi^s$ is a binning version of $\tphi^s$ or a noise injected version of $\tphi^s$ with kernel density estimation \citep{shwartz2017opening,saxe2019information,chelombiev2018adaptive}. 

Indeed, all of such methods of computing mutual information in experiments are  theoretically valid and meaningful based on  our results in Section \ref{sec:7} and Remark \ref{remark:2} along with  Proposition \ref{prop:4} below,  even for the case of mutual information being infinite for the original model $\tphi^s$. 

As a concrete example,  we now study the case when $\phi^s$ is obtained from $\tilde \phi^s$ by injecting noise, i.e. $ \phi^s(x)=\tilde \phi^s(x)+\lambda \vartheta$, where $\vartheta\sim \mathcal N(0, \mathbb I_d/d)$ is the Gaussian noise ($d$ is the dimension of the intermediate output $\tilde \phi^s(x)$): 

\begin{proposition} \label{prop:4}
Let $ \phi^s(x)=\tilde \phi^s(x)+\lambda\vartheta$, where $\vartheta\sim \mathcal N(0, \mathbb I_d/d)$. Let $L$ be the Lipschitz constant of  the function $q\mapsto \ell_{g}(q,Y)$ y w.r.t. the metric induced by $\|\cdot\|_2$ almost surely. Then, we can take $D(\phi^s,\tilde \phi^s)= \lambda L\|\vartheta\|_2$, and with probability at least $1-2\delta$, 
\begin{align}\label{e:thebb}
\Grm[\tphi^s]\le \Brm_{\delta}[\phi^s]+2\lambda L \sqrt{\log(2/\delta)}. 
\end{align}
\end{proposition}
\begin{proof}
Since the function $q\mapsto \ell(g_{l}^s(q),Y)$ is Lipschitz almost surely, we have that with probability one,
\begin{align*}
|\ell_g(\phi^s(X), Y)-\ell_g(\tilde \phi^s(X), Y)|\leq |\ell_g(\tilde \phi^s(X)+\lambda \vartheta, Y)-\ell_g(\tilde \phi^s(X), Y)|\leq \lambda L\|\vartheta\|_2. 
\end{align*}
Thus, we can take $D(\phi^s,\tilde \phi^s)= \lambda L\|\vartheta\|_2$. Since $\vartheta\sim \mathcal N(0, \mathbb I_d/d)$ is a Gaussian vector, by Bernstein inequality, $\PP(\|\vartheta\|_2\geq t)\leq 2e^{-t^2/2}$. If we take $t=2\sqrt{\log (2/\delta)}$, we get $D(\phi^s,\tilde \phi^s)\leq \lambda L\|\vartheta\|_2\leq 2\lambda L\sqrt{\log (2/\delta)}$ with probability $1-\delta$. Thus,  Proposition \ref{prop:4} follows from Remark \ref{remark:2} by taking union bounds.
\end{proof}

In Proposition \ref{prop:4}, let us set  $\Brm_{\delta}[\phi^s]$ to be  a generalization bound on $\phi^s$ with mutual information  $I(Z;X)$ where $Z=\phi ^{s}\circ X$. Then, by the construction of $ \phi^s(x)=\tilde \phi^s(x)+\lambda \vartheta$, the output is stochastic, and the mutual information $I(Z;X)$ in $\Brm_{\delta}[\phi^s]$ is bounded, although $I(\tZ; X)$ with $\tZ=\tphi ^{s}\circ X$  can be infinite. Moreover, there is a trade-off between the two terms on the righthand side of \eqref{e:thebb}:  injecting more noise by increasing $\lambda$ reduces the mutual information $I(Z;X)$ in $\Brm_{\delta}[\phi^s]$, but increases error $2\lambda L \sqrt{\log(2/\delta)}$. Thus,  we cannot arbitrarily change values of the bounds of the original model $\Grm[\tphi^s]$ by choosing  different   methods of computing the mutual information even for the case of  deterministic neural networks with continuous random variables with injective activations.

\reve 

\subsection{On comparisons with previous information-theoretic bounds} \label{app:11}

We discuss the difference between our bounds  and the previous information-theoretic bounds \citep{xu2017information,bassily2018learners} in Section \ref{s:dependent_encoder}; e.g., the previous bounds do not utilize the information bottleneck term. In this subsection, we provide additional discussion on the relation between them.

We first note that Theorem \ref{thm:2} recovers the previous bounds if we set $\bD=\{D+1\}$. If $\bD=\{D+1\}$, then our bound in Theorem \ref{thm:2} removes  $I(X;Z_{l}^s|Y)$ and only keeps $I(\phi_{D+1}^{S}; {S})=I(f^{S}; S)$, resulting in the previous bounds. This is because the hypothesis space of  the decoder after the output layer is always a singleton (since there is no learnable parameter)  and thus there is no need of ``(information) bottleneck'' to avoid overfitting of such decoder. Indeed, the previous bounds only consider the setting where the hypothesis space of the decoder is a singleton; e.g., $g$ is the identity function. In the previous bounds, since the hypothesis space of the decoder g is singleton, there is no need for the encoder to provide a bottleneck to control the complexity of the hypothesis space of the decoder. In contrast, we consider non-singleton hypothesis spaces of decoders and utilize the information bottleneck of the encoder to control the complexity of the decoder. This also illustrates the difficulty to prove our sample complexity bounds with the information bottleneck where we need to consider the non-singleton hypothesis space for the decoder.

Another challenge of proving our bounds comes from the fact that we need to efficiently utilize different sources of randomness while the previous bounds only consider the single source of randomness; i.e.,  $f^{S}=\Acal_{D+1} \circ S$ is a random variable through the randomness of training data  $S$ (and potentially of  algorithm   $\Acal_{D+1}$) whereas $Z^s_l=\phi_{l}^{s} \circ X$ is a random variable  through the randomness of the new unseen input $X$ (and potentially of encoder $\phi_{l}^s$). Thus, $I(X;Z_{l}^s|Y)$ and  $I(f^{S}; S)$ measures different types of mutual information with the different sources of randomness. Our bound needs to utilize both types of randomness efficiently while the previous bound only uses the randomness of $S$.

The main factor $I_{}(X;Z_{l}^s|Y)+I(\phi_{l}^{S}; S)$ in Theorem \ref{thm:2}  captures the novel tradeoff between the two types of mutual information.  It tells us that as we minimize the information bottleneck $I_{}(X;Z_{l}^s|Y)$ by optimizing $\phi_{l} ^{s} $ based on the training data $s$, we must pay the price of mutual information  $I(\phi_{l}^{S}; S)$. If $\phi_{l}^{S}$ depends more on $S$, then we can more easily minimize the information bottleneck $I_{}(X;Z_{l}^s|Y)$ (while minimizing the training loss for $s$), which   comes at the cost of  increasing $I(\phi_{l}^{S}; S)$. This trade-off is not captured by any of previous bounds.

As a result of utilizing the both types of randomness, we show in Section \ref{sec:experiments} that the main factor $I(X;Z_{l}^s|Y)+I(\phi_{l}^{S}; S)$ in our bound is a better predictor than the main factor $I(f^{S}; S)$ in the previous bounds.

\subsection{On the standard arguments for proving the conjecture}

The previous work \citep{shwartz2019representation} provided the  arguments of using the Probably Approximately Correct (PAC) bound  for a finite hypothesis space $\Hcal$ to obtain  $\tilde \Ocal(\sqrt{(\log |\Hcal|)/n})$ \citep{shalev2014understanding} and bounding its cardinality $|\Hcal|$ via $H(Z_{l}^s)$. However, this argument results in the exponential factor $2^{I(X;Z_{l}^{s})}$ as in Conjecture \ref{thm:shwartz}.

\section{Proofs} \label{app:proofs}

\subsection{Overview of Proofs of Theorems}
Before providing complete proofs, we first provide a overview of the proofs of Theorem \ref{thm:1}--\ref{thm:2}. Let  $l \in [D]$. We first prove two  properties of the typical set of $Z_{l}$, Lemma \ref{lemma:5} and Lemma \ref{lemma:8} (in \cref{app:1}),  by  combining a standard proof  used in  information theory and the McDiarmid's inequality. A typical set is a concept in information theory and we utilize the properties of a typical set to obtain  the information-theoretic bounds. To achieve this, Lemma \ref{lemma:paper:1} (in \cref{app:1}) decomposes the generalization gap into four terms as $\EE_{X,Y}[\ell(f^{s}(X),Y)] - \frac{1}{n} \sum_{i=1}^n \ell(f^{s}(x_{i}), y_i)=\Arm+\Brm+\Crm+\Drm$,  where the one term $\Arm$ corresponds to the case of $X$ being in the typical set, while other three terms $\Brm,\Crm$, and $\Drm$ are for the case of $X$ being outside of  the typical set. The rest of the proof of Theorem \ref{thm:1}  analyzes each of these terms (with Lemma \ref{lemma:5} and Lemma \ref{lemma:8}),  proving that $\Arm$ and  $\Brm+\Crm+\Drm$ are bounded by the first term and the second term on the right-hand side of \cref{eq:new:6}, respectively. That is, we show $\Crm +\Drm \le\frac{\gamma_{l} \Rcal(f^s)}{\sqrt{n}} $ in Lemma \ref{lemma:paper:2} (in \cref{app:1}) by  invoking Lemma \ref{lemma:5}. Lemma \ref{lemma:paper:3}  (in \cref{app:1}) then bounds the terms $\Arm$ and $\Brm$  by recasting the problem into  that of   multinomial distributions and by incorporating Lemma \ref{lemma:8} into the concentration inequality of multinomial distributions. 

Lemmas \ref{lemma:5}--\ref{lemma:paper:2}  (in \cref{app:1})
 are carefully proven for the  trained encoder $\phi_{l}^s$ instead of a hypothesis space of encoders $\phi_{l}$. This is achieved by combining  deterministic decompositions and probabilistic bounds with respect to the randomness of new fresh samples $X$  instead of the training data $S$. In contrast, Lemma \ref{lemma:paper:3}  (in \cref{app:1})
is  proven  for a hypothesis space $\Phi$ of encoders using the randomness of $S$, where  $\Phi$  must be  independent of $s$.  These decompositions and probabilistic bounds for different sample spaces enable the exponential improvement over the previous bounds. Combining Lemmas \ref{lemma:paper:1}-\ref{lemma:paper:3}  (in \cref{app:1}) produces Lemma \ref{lemma:3}, which proves Theorem \ref{thm:1} by setting the hypothesis space as $\Phi=\{\phi_l^s\}$ where $\phi^s_l$ is fixed independently of  $s$.

The standard proof techniques result in the exponential factor $2^{I(X;Z_{l}^{s})}$ as in Conjecture \ref{thm:shwartz} (see \cref{app:7} for more details). This paper  provides a novel proof technique to avoid the exponential factor. Compared to arguments for Conjecture \ref{thm:shwartz}, our proof discards the non-mathematical arguments regarding the typical set, keeps track of all the effects of the approximation and non-typicality rigorously, and discards the assumption of the input dimension approaching infinity with an ergodic Markov random field. 

Another main challenge in proving our main result, Theorem \ref{thm:2}, is avoiding the dependence on the hypothesis space for the value of $I(X; Z_{l}^s|Y)$. That is, with a relatively simpler proof, we could  prove a similar bound with $\sup_{\phi_{l} \in \Phi_{l}} I(X;\phi_{l}^{} \circ X|Y)$ where $\Phi_{l}$ is a fixed hypothesis space of the encoder $\phi_{l}$. However, this dependence on the hypothesis space is not preferred since enlarging the hypothesis space can increase the value, whereas the value of $I(X; Z_l^s)$ in our bound is independent of the hypothesis space given the final hypothesis $\phi_{l}^{s}$. \\ \\ We carefully construct and prove our key lemmas in the following subsection, which enables us to avoid  the dependence over the entire hypothesis space and the exponential factor.

\subsection{Proofs of Key Lemmas} \label{app:1}

We use the notation of $\ln = \log_e$ and $\log=\log_2$.
 Fix $l \in \{1,\dots,D\}$ throughout this section.
For the simplicity of the notation, we write $Z=Z_l^s$ and $Z_y=Z_{l,y}^s$ in the following; we  must to be always aware of the dependence on $s$ for related variables. We recall that
$$
Z_y =\phi_l ^{s}\circ X_{y}. 
$$
We write    $\xi^{(y)} \in \hXi_{y} \subseteq \RR^m$, and define the set of the latent variable per class by
$$
\Zcal_{y} =\left\{(\phi_l^{s}\circ \chi_{y})(\xi^{(y)}):\ \xi^{(y)}\in \hXi_y  \right\}.
$$
For any $\gamma>0$, we then define the typical subset $\Zcal_{\gamma,y}^{s} $ of the set $\Zcal_{y}$ by
$$
\Zcal_{\gamma,y}^{s} =\left\{z \in \Zcal_{y} : - \log \PP_{}(Z_{y}=z)- H_{}(Z_{y})  \le c_{l}^{y}(\phi_l^{s}) \sqrt{\frac{m\ln(\sqrt{n}/\gamma)}{2}}  \right\}.
$$ 
Then, for any set $A$ and any function $\varphi$, 
we have that
\begin{align*}
\PP_{}(Z \in A|Y=y)=\PP_{}(Z_{y} \in A)&=\PP(\{\omega _{y}\in \Omega_{y} :Z_{y}(\omega_{y})\in A  \})
\\ &=\PP(\{\omega_{y}\in \Omega_{y}:(\phi_l ^{s}\circ \chi_{y})(\Xi_y(\omega_{y}))\in A  \})\\ &=\PP_{}((\phi_l ^{s}\circ \chi_{y}\circ\Xi_y) \in A),
\end{align*}
and
\begin{align*}
\PP((\varphi\circ Z) >  0|Y=y)=\PP_{}((\varphi\circ Z_{y} )>0)&=\PP(\{\omega _{y}\in \Omega_{y} :\varphi(Z_{y}(\omega_{y}))> 0  \}) 
\\ &=\PP(\{\omega_{y} \in \Omega _{y}:\varphi((\phi_l ^{s}\circ \chi_{y})(\Xi_y(\omega_{y})))> 0  \})
\\ &=\PP_{}((\varphi\circ\phi_l ^{s}\circ \chi_{y}\circ\Xi_y) >  0 ).
\end{align*}
Thus, for example,   we can write
\begin{align*}
\PP(Z  \notin \Zcal_{\gamma,y}^{s}|Y=y)&=\PP\left((\phi_l ^{s}\circ \chi_{y}\circ\Xi_y)\notin\Zcal_{\gamma,y}^{s}\right)
\\ &=\PP_{}\left( \left\{\omega_{y}\in \Omega_{y} :- \log \PP_{}(Z_{y}=\phi_{l}^{s}(X_{y}(\omega_{y}))_{})-  H_{}(Z_{y})> \epsilon \right\}\right)
\\ & =\PP_{}\left( \left\{\omega_{y}\in \Omega_{y} :- \log\PP\left( \{\omega '_{y}\in \Omega_{y} :Z_{y}(\omega'_{y})=Z_{y}(\omega_{y})_{}\}\right)- H_{}(Z_y)> \epsilon \right\}\right),
\end{align*}
where $\epsilon=c_{l}^{y}(\phi_l^{s}) \sqrt{\frac{m\ln(\sqrt{n}/\gamma)}{2}}$. 

\subsubsection{Probability of Going outside of The Typical Subset} \label{sec:5}
The following lemma shows that the conditional probability of going outside of $\Zcal_{\gamma,y}$ is bounded by $\frac{\gamma}{\sqrt{n}}$:  
\begin{lemma} \label{lemma:5}
For any $\gamma>0$, it holds  that 
$$
\PP(Z  \notin\Zcal_{\gamma,y}^{s}\mid Y=y )\le \frac{\gamma}{\sqrt{n}}.
$$
\end{lemma}  
\begin{proof}
Fix $y\in \Ycal$. We then write $\xi=\xi^{(y)}$ for the simplicity of the notation. We now consider the statistical property of the function $\xi \mapsto -  \log \PP_{}(Z_{y}=\phi_{l}^{s}(\chi(y, \xi)))$.
That is, in the following,  we will apply McDiarmid's inequality w.r.t. the sample space $\omega_{y}\in \Omega_{y}$ to the following function: 
$$
\xi \mapsto - \log \PP(Z_{y}=\phi_{l}^{s}(\chi(y,\xi)))=- \log \PP( \{\omega '_{y}\in \Omega_{y} :Z_{y}(\omega'_{y})=\phi_l^{s}(\chi (y,\xi)_{})\}).
$$
For the simpler notation, define the function $p_{y}$ by $$
p_{y} (q)= \PP_{}(Z_{y}=q).
$$ Then,  we can rewrite the above function of $\xi$ as
$$
\xi=(\xi_1,\dots,\xi_m)\mapsto- \log p_{y}(\phi_l^{s}(\chi (y,\xi))).
$$
We also define 
$$
\tilde \Zcal_{\epsilon,y}^{s}=\left\{z \in \Zcal_{y} : - \log p_{y}(z)- H_{}(Z_{})  \le  \epsilon  \right\}.
$$
For any $(h,\varphi)$ and $t=h(q)$ with a probability mass function $p$, since $p(t)=\sum_{q \in h^{-1}(t)} p(q)$,
\begin{align*}
\EE_{t\sim p}\left[\varphi_{}(t) \right] =  \sum_{t} \varphi(t)p(t) &= \sum_{t}   \varphi(t) \sum_{q \in h^{-1}(t)} p(q)
  \\ &= \sum_{t}    \sum_{q \in h^{-1}(t)} \varphi(t)p(q)
  \\ &= \sum_{t}    \sum_{q \in h^{-1}(t)} \varphi(h(q))p(q)
\\ & =\sum_{q} \varphi(h(q))p(q)=\EE_{q\sim p}[\varphi(h(q))]. 
\end{align*}Thus, by choosing $q=\xi_{y}$, $h(q)= \phi_{l}^{s}(\chi(y,q))$, and $\varphi(t)=- \log p_{y}(t)$, we have that$$
\EE_{\Xi_{y}}\left[- \log p_{y}(\phi_l^{s}(\chi(y,\Xi_{y}))) \right]=\EE_{q}[\varphi(h(q))]=\EE_{t}\left[\varphi_{}(t) \right]=\EE_{Z_{y}}\left[- \log p_{y}(Z_{y}) \right]=H_{}(Z_{y}).
$$  

Thus, by using McDiarmid's inequality, 
$$
\PP_{\Xi_{y}}\left(- \log p_{y}((\phi_l^{s}\circ\chi_y )(\Xi_{y}))- H_{}(Z_{y}) \ge \epsilon \right) \le \exp\left(-\frac{2\epsilon^2}{mc_{l}^{y}(\phi_l^{s})^2}\right).  
$$
By setting $\delta=\exp\left(-\frac{2\epsilon^2}{c_{l}^{y}(\phi_l^{s})^2}\right)$ and solving for $\epsilon$,
we set 
$$
\epsilon=c_{l}^{y}(\phi_l^{s}) \sqrt{\frac{m\ln(1/\delta)}{2}},
$$ with which
\begin{align*}
\PP(Z  \notin \tilde \Zcal_{\epsilon,y}^{s} \mid Y=y )&=\PP_{\Xi_{y}}\left((\phi_l ^{s}\circ \chi_{y})(\Xi_y)\notin\tilde \Zcal_{\epsilon,y}^{s}\right)
\\ & =\PP_{\Xi_{y}}\left(- \log p_{y}((\phi_l^{s}\circ\chi_y )(\Xi_{y}))- H_{}(Z_{y}) > \epsilon \right) 
\\ &\ \le \PP_{\Xi_{y}}\left(- \log p_{y}((\phi_l^{s}\circ\chi_y )(\Xi_{y}))- H_{}(Z_{y})\ge \epsilon \right) \le\delta.
\end{align*}
 Therefore, by setting  $\delta = \frac{\gamma}{\sqrt{n}}$ and accordingly $\epsilon=c_{l}^{y}(\phi_l^{s}) \sqrt{\frac{m\ln(1/\delta)}{2}}=c_{l}^{y}(\phi_l^{s}) \sqrt{\frac{m\ln(\sqrt{n}/\gamma)}{2}}$, we have proven the desired statement, since $\tilde \Zcal_{\epsilon,y}^{s}=\Zcal_{\gamma,y}^s$ when $\epsilon=c_{l}^{y}(\phi_l^{s}) \sqrt{\frac{m\ln(\sqrt{n}/\gamma)}{2}}$.
\end{proof}

\subsubsection{Size of the Typical Subset} \label{sec:6}
The following lemmas bounds the size of the subset $\Zcal_{\gamma,y}^s$:
\begin{lemma} \label{lemma:8}
For any $\gamma>0$, 
$$
|\Zcal_{\gamma,y} ^{s}|\le2^{H_{y}(Z_{y})+c_{l}^{y}(\phi_l^{s})\sqrt{\frac{m\ln(\sqrt{n}/\gamma)}{2}}}.
$$
\end{lemma}
\begin{proof}
Set $\epsilon=c_{l}^{y}(\phi_l^{s}) \sqrt{\frac{m\ln(\sqrt{n}/\gamma)}{2}}$. We define the function $p_{y}$ by $
p_{y} (q)= \PP_{}(Z_{y}=q)$. Then, from the definition of $\Zcal_{\gamma,y}^s$,
we have
that for any $a\in \Zcal_{\gamma,y}^s$,
\begin{align*}
  - \log p_{y}(a_{})- H_{}(Z_{y})\le \epsilon &\Longleftrightarrow  - \log p_{y}(a_{}) \le H_{}(Z_{y})+\epsilon 
\\  &\Longleftrightarrow
-\left(H_{}(Z_{y})+\epsilon \right) \le \log p_{y}(a)
\\  &\Longleftrightarrow
2^{-H_{}(Z_{y})-\epsilon} \le  p_{y}(a).
\end{align*}
Using $2^{-H_{}(Z_{y})-\epsilon} \le  p_{y}(a)= \PP(Z_{y}=a_{})$ for all  $a \in\Zcal_{\gamma,y} ^{s}$, 
$$
1\ge \PP_{}(Z_{y}\in\Zcal_{\gamma,y} ^{s}) = \sum_{a \in\Zcal_{\gamma,y} ^{s}}  \PP_{ }(Z_{y}= a) \ge \sum_{a \in\Zcal_{\gamma,y} ^{s}}  2^{-H_{}(Z_{y})-\epsilon} =|\Zcal_{\gamma,y} ^{s}|2^{-H_{}(Z_{y})-\epsilon}.
$$
This implies that using $\epsilon=c_{l}^{y}(\phi_l^{s}) \sqrt{\frac{m\ln(\sqrt{n}/\gamma)}{2}}$, $$
|\Zcal_{\gamma,y} ^{s}|\le2^{H_{}(Z_{y})+\epsilon}=2^{H_{}(Z_{y})+c_{l}^{y}(\phi_l^{s})\sqrt{\frac{m\ln(\sqrt{n}/\gamma)}{2}}}.
$$
\end{proof}

\subsubsection{Decomposition of Expected Loss using the Typical Subset}
Let us write 
$$
z_{i}=\phi_{l}^{s}(x_{i})\in \Zcal_{l} \subseteq \RR^{m_l},  
$$
and$$
\ell_{l}(q,y) =\ell( g_{l}^{s}(q),y).
$$

Then, by the law of the unconscious statistician,
$$
\EE_{X,Y}[\ell(f^{s}(X),Y)] - \frac{1}{n} \sum_{i=1}^n \ell(f^{s}(x_{i}), y_i)=\EE_{Z,Y}[\ell_{l}(Z,Y)] - \frac{1}{n} \sum_{i=1}^n \ell_{l}(z_{i},y_{i}).
$$
For simplicity of the notation, define $A_{y}=\Zcal_{\gamma,y}^{s}$. 
We now consider a partition of the space $\Zcal_{l} $ as $
\Zcal_{l} = \{z \in A_{y}\} \cup  \{z \notin A_{y}\}$. Fix an order and write the element of $A_{y}$ by $A_{y}=\{a_1^{y},\dots,a_{T_{y}}^{y}\}$ where $T_{y} =|A_y|\le2^{H_{y}(\phi_{l} \circ X_{y})+c_{l}^{y}(\phi_l^{s})\sqrt{\frac{m\ln(\sqrt{n}/\gamma)}{2}}}$ from the Lemma \ref{lemma:8}. We define 
$
\Ical_{y}=\{i \in [n] :y_{i}=y\},
$
$
\tIcal^{y}=\{i \in [n] : z_{i} \notin A_{y} ,y_{i}=y\},
$
$
\Ical_k ^{y}=\{i \in [n] : z_i = a_{k}^{y},y_{i}=y\}, 
$
$\tYcal = \{y \in \Ycal: |\tIcal^{y}| \neq 0\}$,
$ \frac{1}{|\tIcal^{y}|}\sum_{i\in\tIcal^{y} } \ell_{l}(z_{i},y) q \triangleq 0 \text{ for any $q$ if }  |\tIcal^{y}|=0,
$
and
$
\frac{1}{|\Ical_k^{y}|}\sum_{i\in\Ical_k^y } \ell_{l}(z_{i},y) q\triangleq 0 \text{ for any $q$ if } |\Ical_k^{y}|=0.
$ Here, for example, $Z$, $a_k^y$, $A_y$, $|\Ical_k^{y}|$, and $|\tIcal^{y}|$  depend on the training dataset $s$ through the function $\phi^s_l$ due to their definitions.

Using these, we can decompose the   expected loss as in the following lemma:
\begin{lemma} \label{lemma:paper:1}
The following holds (deterministically):
\begin{align}  \label{eq:1}
\EE_{X,Y}[\ell(f^{s}(X),Y)] - \frac{1}{n} \sum_{i=1}^n \ell(f^{s}(x_{i}), y_i)
  & = \sum_{y \in \tYcal}\frac{1}{|\tIcal^{y}|}\sum_{i\in\tIcal^{y} } \ell_{l}(z_{i},y_{})\left(\PP(Y=y,Z\notin A_{y})-\frac{|\tIcal^{y}|}{n}  \right)
\\ \nonumber & \quad + \sum_{y \in \Ycal}\sum_{k=1}^{T_{y}} \ell_{l}(a_k^{y},y) \left(\PP_{}(Y=y,Z=a_k^{y})-\frac{|\Ical_k^{y}|}{n} \right)  
\\ \nonumber & \quad  +   \sum_{y \in \Ycal}\PP(Y=y,Z\notin A_{y})\EE_{Z,Y}[\ell_{l}(Z,Y)|Z\notin A_{y},Y=y]
\\  \nonumber & \quad -  \sum_{y \in \tYcal}\PP(Y=y,Z\notin A_{y}) \frac{1}{|\tIcal^{y}|}\sum_{i\in\tIcal^{y} } \ell_{l}(z_{i},y_{}). 
 \end{align}
\end{lemma}
\begin{proof}
we can decompose the  expected loss by using conditionals as
\begin{align*}
\EE_{Z,Y}[\ell_{l}(Z,Y)]
 &= \sum_{y \in \Ycal} \PP(Y=y) \EE_{Z,Y}[\ell_{l}(Z,Y)|Y=y].
\end{align*}
Furthermore, we can decompose the conditional expectation as\begin{align*}
 \EE_{Z,Y}[\ell_{l}(Z,Y)|Y=y]
&= \PP(Z\notin A_{y}|Y=y)\EE_{Z,Y}[\ell_{l}(Z,Y)|Z\notin A_{y},Y=y]
\\ & \quad +\PP(Z\in A_{y} |Y=y) \EE_{Z,Y}[\ell_{l}(Z,Y)|Z\in A_{y},Y=y]   
\\ &=
 \PP(Z\notin A_{y}|Y=y)\EE_{Z,Y}[\ell_{l}(Z,Y)|Z\notin A_{y},Y=y]
 \\ & \quad +\sum_{k=1}^{T_y}\PP_{}(Z=a_k^{y}|Y=y) \EE_{Z,Y}[\ell_{l}(Z,Y)|Z=a_k^{y},Y=y]  
\\ &=
 \PP(Z\notin A_{y}|Y=y)\EE_{Z,Y}[\ell_{l}(Z,Y)|Z\notin A_{y},Y=y]
 \\ & \quad +\sum_{k=1}^{T_y}\PP_{}(Z=a_k^{y}|Y=y)\ell_{l}(a_k^y,y)  
\end{align*}
Summarising above, 
\begin{align*}
\EE_{Z,Y}[\ell_{l}(Z,Y)]
 &= \sum_{y \in \Ycal} \PP(Y=y,Z\notin A_{y})\EE_{Z,Y}[\ell_{l}(Z,Y)|Z\notin A_{y},Y=y]
\\ & \quad + \sum_{y \in \Ycal}\sum_{k=1}^{T_y}\PP_{}(Y=y,Z=a_k^{y})\ell_{l}(a_k^y,y).
\end{align*}
Similarly, we can decompose the   training loss as
\begin{align*}
\frac{1}{n} \sum_{i=1}^n \ell_{l}(z_{i},y_{i})  &= \frac{1}{n}  \sum_{y \in \Ycal} \sum_{i\in \Ical_y} \ell_{l}(z_{i},y_{}) 
\\ &=\frac{1}{n} \sum_{y \in \Ycal}\left(\sum_{i\in\tIcal^{y} } \ell_{l}(z_{i},y_{})+ \sum_{k=1}^{T_y}\sum_{i\in\Ical_k^y } \ell_{l}(z_{i},y_{}) \right)  
\\ & = \sum_{y \in \tYcal}\frac{|\tIcal^{y}|}{n} \frac{1}{|\tIcal^{y}|}\sum_{i\in\tIcal^{y} } \ell_{l}(z_{i},y_{})+ \sum_{y \in \Ycal}\sum_{k=1}^{T_y}\frac{|\Ical_k^{y}|}{n}  \ell_{l}(a_{k},y_{})
\end{align*}
Using these, we now decompose the expected loss as follows:
\begin{align*}
\EE_{Z,Y}[\ell_{l}(Z,Y)] - \frac{1}{n} \sum_{i=1}^n \ell_{l}(z_{i},y_{i})
& =   \sum_{y \in \Ycal} \PP(Y=y,Z\notin A_{y})\EE_{Z,Y}[\ell_{l}(Z,Y)|Z\notin A_{y},Y=y]
\\ & \quad + \sum_{y \in \Ycal}\sum_{k=1}^{T_y}\PP_{}(Y=y,Z=a_k^{y})\ell_{l}(a_k^y,y)-\frac{1}{n} \sum_{i=1}^n \ell_{l}(z_{i},y_{i})
\\ & \quad \pm   \sum_{y \in \tYcal}\PP(Y=y,Z\notin A_{y})\frac{1}{|\tIcal^{y}|}\sum_{i\in\tIcal^{y} } \ell_{l}(z_{i},y_{})\pm \sum_{y \in \Ycal}\sum_{k=1}^{T_y}\frac{|\Ical_k^{y}|}{n}  \ell_{l}(a_{k},y_{})
\end{align*}    
By rearranging,  
\begin{align*}
&\EE_{Z,Y}[\ell_{l}(Z,Y)] - \frac{1}{n} \sum_{i=1}^n \ell_{l}(z_{i},y_{i}) \\ &=   \sum_{y \in \Ycal}\PP(Y=y,Z\notin A_{y}) \EE_{Z,Y}[\ell_{l}(Z,Y)|Z\notin A_{y},Y=y]
\\ & \quad -\sum_{y \in \tYcal}\PP(Y=y,Z\notin A_{y})\left(\frac{1}{|\tIcal^{y}|}\sum_{i\in\tIcal^{y} } \ell_{l}(z_{i},y_{}) \right)
\\ & \quad + \sum_{y \in \Ycal}\sum_{k=1}^{T_{y}} \ell_{l}(a_k^y,y) \left(\PP_{}(Y=y,Z=a_k^{y})-\frac{|\Ical_k^{y}|}{n} \right)  
 \\ & \quad +   \sum_{y \in \tYcal}\PP(Y=y,Z\notin A_{y})\frac{1}{|\tIcal^{y}|}\sum_{i\in\tIcal^{y} } \ell_{l}(z_{i},y_{})+ \sum_{y \in \Ycal}\sum_{k=1}^{T_y}\frac{|\Ical_k^{y}|}{n}  \ell_{l}(a_{k},y_{})- \frac{1}{n} \sum_{i=1}^n \ell_{l}(z_{i},y_{i})
\\ &=   \sum_{y \in \Ycal}\PP(Y=y,Z\notin A_{y}) \EE_{Z,Y}[\ell_{l}(Z,Y)|Z\notin A_{y},Y=y]
\\ & \quad -\sum_{y \in \tYcal}\PP(Y=y,Z\notin A_{y})\left(\frac{1}{|\tIcal^{y}|}\sum_{i\in\tIcal^{y} } \ell_{l}(z_{i},y_{}) \right)
\\ & \quad + \sum_{y \in \Ycal}\sum_{k=1}^{T_{y}} \ell_{l}(a_k^y,y) \left(\PP_{}(Y=y,Z=a_k^{y})-\frac{|\Ical_k^{y}|}{n} \right)  
 \\ & \quad +   \sum_{y \in \tYcal}\PP(Y=y,Z\notin A_{y})\frac{1}{|\tIcal^{y}|}\sum_{i\in\tIcal^{y} } \ell_{l}(z_{i},y_{})+ \sum_{y \in \Ycal}\sum_{k=1}^{T_y}\frac{|\Ical_k^{y}|}{n}  \ell_{l}(a_{k},y_{})
 \\ & \quad -  \sum_{y \in \tYcal}\frac{|\tIcal^{y}|}{n} \frac{1}{|\tIcal^{y}|}\sum_{i\in\tIcal^{y} } \ell_{l}(z_{i},y_{})- \sum_{y \in \Ycal}\sum_{k=1}^{T_y}\frac{|\Ical_k^{y}|}{n}  \ell_{l}(a_{k},y_{})
 \end{align*}
By combining the relevant terms,  
\begin{align*} 
&\EE_{Z,Y}[\ell_{l}(Z,Y)] - \frac{1}{n} \sum_{i=1}^n \ell_{l}(z_{i},y_{i})
\\ &=   \sum_{y \in \Ycal}\PP(Y=y,Z\notin A_{y}) \EE_{Z,Y}[\ell_{l}(Z,Y)|Z\notin A_{y},Y=y]
\\ & \quad -\sum_{y \in \tYcal}\PP(Y=y,Z\notin A_{y})\left(\frac{1}{|\tIcal^{y}|}\sum_{i\in\tIcal^{y} } \ell_{l}(z_{i},y_{}) \right)
\\ & \quad + \sum_{y \in \Ycal}\sum_{k=1}^{T_{y}} \ell_{l}(a_k^y,y) \left(\PP_{}(Y=y,Z=a_k^{y})-\frac{|\Ical_k^{y}|}{n} \right)  
 \\ & \quad +\sum_{y \in \tYcal}\frac{1}{|\tIcal^{y}|}\sum_{i\in\tIcal^{y} } \ell_{l}(z_{i},y_{})\left(\PP(Y=y,Z\notin A_{y})-\frac{|\tIcal^{y}|}{n}  \right)+ \sum_{y \in \Ycal}\sum_{k=1}^{T_y}\frac{|\Ical_k^{y}|}{n} \left(\ell_{l}(a_{k},y_{})-\ell_{l}(a_{k},y_{}) \right)
\\ &=   \sum_{y \in \Ycal}\PP(Y=y,Z\notin A_{y}) \EE_{Z,Y}[\ell_{l}(Z,Y)|Z\notin A_{y},Y=y]
\\ & \quad -\sum_{y \in \tYcal}\PP(Y=y,Z\notin A_{y})\left(\frac{1}{|\tIcal^{y}|}\sum_{i\in\tIcal^{y} } \ell_{l}(z_{i},y_{}) \right)
\\ & \quad + \sum_{y \in \Ycal}\sum_{k=1}^{T_{y}} \ell_{l}(a_k^y,y) \left(\PP_{}(Y=y,Z=a_k^{y})-\frac{|\Ical_k^{y}|}{n} \right)  
 \\ & \quad +\sum_{y \in \tYcal}\frac{1}{|\tIcal^{y}|}\sum_{i\in\tIcal^{y} } \ell_{l}(z_{i},y_{})\left(\PP(Y=y,Z\notin A_{y})-\frac{|\tIcal^{y}|}{n}  \right)
\end{align*}
This implies the desired statement.
\end{proof}

\subsubsection{Bounding the third and forth  terms in the decomposition}
Define 
$$
R _{y}=\EE_{Z,Y}[\ell_{l}(Z,Y)|Z\notin A_{y},Y=y]. 
$$Then, the following lemma  bounds the third and forth terms in the decomposition of  \eqref{eq:1} from the previous subsection:
\begin{lemma} \label{lemma:paper:2}
For any $\gamma>0$, the following holds: 
\begin{align} \label{eq:4}
 \sum_{y \in \Ycal} \PP(Y=y)\frac{\gamma R_{y}}{\sqrt{n}} & \ge \nonumber     \sum_{y \in \Ycal}\PP(Y=y,Z\notin A_{y})\EE_{Z,Y}[\ell_{l}(Z,Y)|Z\notin A_{y},Y=y]
\\  & \quad -  \sum_{y \in \tYcal}\PP(Y=y,Z\notin A_{y}) \frac{1}{|\tIcal^{y}|}\sum_{i\in\tIcal^{y} } \ell_{l}(z_{i},y_{}).
\end{align}
\end{lemma}
\begin{proof}
Recalling the definition of $A_{y}=\Zcal_{\gamma,y}^{s}$, the third term can be written as 
\begin{align*}
 \sum_{y \in \Ycal}\PP(Y=y,Z\notin A_{y})R _{y}
 =\sum_{y \in \Ycal}\PP(Y=y)\PP (Z\notin \Zcal_{\gamma,y}^{s}|Y=y)R _{y}. 
\end{align*}
Then, using Lemma \ref{lemma:5},
for any $\gamma>0$, $$
\PP_{} (Z\notin \Zcal_{\gamma,y}^{s}|Y=y)\le \frac{\gamma}{\sqrt{n}}.
$$
Since $\sum_{y \in \tYcal}\PP(Y=y,Z\notin A_{y}) \frac{1}{|\tIcal^{y}|}\sum_{i\in\tIcal^{y} } \ell_{l}(z_{i},y_{})\ge0$, combining these implies the desired statement.
\end{proof}

\subsubsection{Bounding the first and second term in the decomposition}

Let $\Phi$ be fixed such that $\Phi$ is independent of $s$, while $\Phi$ can depend on the underlying data distribution. The following lemma  probabilistically bounds the first and second term in the  decomposition of  \eqref{eq:1}:  

\begin{lemma} \label{lemma:paper:3}
If $\phi^s_l \in \Phi$,  for any $\gamma>0$ and $\delta>0$, with probability at least $1-\delta$,  the following holds:  
\begin{align} \label{eq:9}
  &\sum_{y \in \tYcal}\frac{1}{|\tIcal^{y}|}\sum_{i\in\tIcal^{y} } \ell_{l}(z_{i},y_{})\left(\PP(Y=y,Z\notin A_{y})-\frac{|\tIcal^{y}|}{n}  \right)
 \\ \nonumber & \le\left(\sum_{y \in \tYcal}\sqrt{\PP_{}( Z \notin\Zcal^s_{\gamma,y} ,Y=y)}\frac{\sum_{i\in\tIcal^{y} } \ell_{l}(z_{i},y)}{|\tIcal^{y}|} \right) \sqrt{\frac{2\ln(2|\Phi|| \Ycal|/\delta)}{n}}, \text{\textit{ and,}}
\\  \nonumber & \sum_{y \in \Ycal}\sum_{k=1}^{T_{y}} \ell_{l}(a_k^{y},y) \left(\PP(Y=y,Z=a_k^{y})-\frac{|\Ical_k^{y}|}{n} \right)
\\ \nonumber & \le  \sum_{y \in \Ycal} \left(\sum_{k=1}^{T_{y}}\ell_{l}(a_k^{y},y) \sqrt{\PP_{}(Z=a_{k}^{y},Y=y)}\right)  \sqrt{\frac{2\left( I_{}(X_{y};Z_{y})+G_2^{y}  \right)\ln(2)+2\ln(2|\Phi|| \Ycal|/\delta)}{n}}.
\end{align}
where 
$$
G_2 ^{y}=c_{l}^{y}(\phi_{l}^{s})\sqrt{\frac{m\ln(\sqrt{n}/\gamma)}{2}}+H_{}(Z_{y}|X_{y}). $$
\end{lemma}
\begin{proof}
Let $\gamma>0$ fixed.
Define $$
\Xcal_{y} =\left\{\chi_{y}(\xi^{(y)}):\ \xi^{(y)}\in \hXi_y  \right\},
$$
and$$
\hA_y(\phi_{l})=\left\{x \in \Xcal_y: - \log \PP(Z_{y}=\phi_{l}(x)_{})- H(Z_{y})  \le c_{l}^{y}(\phi_l) \sqrt{\frac{m\ln(\sqrt{n}/\gamma)}{2}}  \right\}.
$$ 

For each $\phi_l$,  write the element of $\hA_{y}(\phi_{l})$ by $\hA_{y}(\phi_{l})=\{\ha_1^{y}(\phi_{l}),\dots,\ha_{\hT_{y}(\phi_l)}^{y}(\phi_{l})\}$ (with a fixed order) where $\hT_{y}(\phi_l) =|\hA_{y}(\phi_{l})|.$ Moreover, we define
$$
\hIcal_k ^{y}(\phi_l)=\begin{cases}\{i \in [n] : \phi_l(x_i )= \ha_{k}^{y}(\phi_l),y_{i}=y\} & \text{if } k \in [\hT_{y}(\phi_l)] \\
\{i \in [n] : \phi_l(x_i ) \notin \hA_{y}(\phi_{l}),y_{i}=y\} & \text{if } k=\hT_{y}(\phi_l)+1.
\end{cases}
$$
These are defined such that the previously defined notations are recovered when we set $\phi_{l}=\phi_{l}^{s}$ as 
\begin{align} \label{eq:7}
& \Zcal^s_{\gamma,y} =\hA_y(\phi_{l}^{s}), 
\quad  A_{y}= \hA_{y}(\phi_{l}^{s}),
\\ \nonumber & (a_{1}^{y},\dots,a_T^{y})=(\ha_1^{y}(\phi_{l}^{s}),\dots,\ha_{\hT_{y}(\phi_l^{s})}^{y}(\phi_{l}^{s})),
\quad \Ical_k ^{y}=\hIcal_k ^{y}(\phi_l^{s}),
\\ \nonumber & \tIcal^{y}=\hIcal_{\hT_{y}(\phi_l^{s})+1} ^{y}(\phi_l^{s}),
\quad  T_{y} =\hT_{y}(\phi_l^{s}). 
\end{align}

We begin with bounding terms for a fixed encoder, before extending it to the case of encoders learned from the training set. Let  $\phi_{l}$ fixed and define  
$$
p_k^{y}= \begin{cases}\PP((\phi_l\circ X)= \ha_{k}^{y}(\phi_l),Y=y) & \text{if } k\in [\hT_{y}(\phi_l)] \\
\PP((\phi_l\circ X)\notin \hA_{y}(\phi_{l}),Y=y) & \text{if } k = \hT_{y}(\phi_l)+1.\end{cases} $$
Let $y \in \Ycal$ and $k \in [\hT_{y}(\phi_l)+1]$. Then, we first prove  the following statement: for any $\delta>0$, with probability at least $1-\delta$,
\begin{align} \label{lemma:2} 
  p_k^{y}- \frac{|\hIcal_k ^{y}(\phi_l)|}{n} \le \sqrt{\frac{2p_k ^{y}\ln(1/\delta)}{n}}.
\end{align}
To prove this statement,
fix $y \in \Ycal$ and $k \in [\hT_{y}(\phi_l)+1]$. Let us write $\hIcal_k=\hIcal_k ^{y}(\phi_l)$ and $p_k=p_k^y$. If $p_k=0$, then the desired statement holds trivially because $p_k- \frac{|\hIcal_k|}{n}=- \frac{|\hIcal_k|}{n} \le \sqrt{\frac{2p_k \ln(1/\delta)}{n}}$ where $\frac{|\hIcal_k|}{n} =0$ and $\sqrt{\frac{2p_i \ln(1/\delta)}{n}}=0 $.
Thus, for
                the rest, we consider the case where $p_k\neq 0$. We notice that $(|\hIcal_1|,\dots,|\hIcal_{T+1}|)$ follows the multinomial distribution with parameter $n$ and $(p_{1},\dots,p_{T+1})$. Thus, we invoke Lemma 3 of \citep{kawaguchi2022robust} with $\bar a_i=1$ and $\bar a_j=0$ for all $j \neq i$ (which satisfies $\sum_{i=1}^K \bar a_i p_i\neq 0$ since $p_{i}\neq 0$), yielding  that for any $M>0$,
                $$
                \PP\left( p_k - \frac{|\hIcal_k|}{n} >M   \right) \le \exp\left(-\frac{nM^{2}}{2p_{k}} \right).
                $$ 
                By setting $M= \sqrt{\frac{2 p_i\ln(1/\delta)}{n}}$, 
                $$
                \PP\left( p_k - \frac{|\hIcal_k|}{n}> \sqrt{\frac{2 p_k\ln(1/\delta)}{n}}   \right) \le \delta.
$$            
This proves the statement of \eqref{lemma:2}.
Using \eqref{lemma:2},  we can bound the first and second terms for a fixed $\phi_{l}$ as follows.
For the first term with a fixed $\phi_l$, using  \eqref{lemma:2}, by taking union bounds over all $y \in \Ycal$,  we have that                 for any $\delta>0$, with probability at least $1-\delta$, the following holds for all $y \in \Ycal$:
\begin{align}
& \PP_{}( \phi_l(X ) \notin \hA_{y}(\phi_{l}),Y=y)-\frac{|\hIcal_{\hT_{y}(\phi_l)+1} ^{y}(\phi_l)|}{n} 
\\ \nonumber &\le \sqrt{\frac{2\PP_{}( \phi_l(X ) \notin \hA_{y}(\phi_{l}),Y=y)\ln(| \Ycal|/\delta)}{n}}.
\end{align}

For the second term with a fixed $\phi_l$, using  \eqref{lemma:2}, by taking union bounds over all  $y\in \Ycal$ and  all $k\in [\hT_{y}(\phi_l)]$, we have that for any $\delta>0$, with probability at least $1-\delta$, the following holds for  all $y \in \Ycal$ and all $k \in[\hT_{y}(\phi_l)]$, 
\begin{align*}
& \PP_{}(\phi_l(X)= \ha_{k}^{y}(\phi_l),Y=y)-\frac{|\hIcal_k ^{y}(\phi_l)|}{n} 
\\& \le\sqrt{\PP_{}(\phi_l(X)= \ha_{k}^{y}(\phi_l),Y=y)} \sqrt{\frac{2\ln(| \Ycal|\hT_{y}(\phi_l)/\delta)}{n}}.
\end{align*}

We now extend the results for the case of encoders learned from the training set; i.e., $\phi_l$ is no longer fixed. By taking union bounds with the previous two bounds, we have that for any $\delta>0$, with probability at least $1-\delta$,  the following holds for all $\phi_l \in \Phi$:
\begin{align*}
&  \PP_{}( \phi_l(X ) \notin \hA_{y}(\phi_{l}),Y=y)-\frac{|\hIcal_{\hT_{y}(\phi_l)+1} ^{y}(\phi_l)|}{n} 
\\ &\le \sqrt{\frac{2\PP_{}( \phi_l(X ) \notin \hA_{y}(\phi_{l}),Y=y)\ln(2|\Phi|| \Ycal|/\delta)}{n}}, \end{align*}
\textit{and  for all $k \in[\hT_{y}(\phi_l)]$, }
\begin{align*}
& \PP_{}(\phi_l(X)= \ha_{k}^{y}(\phi_l),Y=y)-\frac{|\hIcal_k ^{y}(\phi_l)|}{n} 
\\ &\le\sqrt{\PP_{}(\phi_l(X)= \ha_{k}^{y}(\phi_l),Y=y)} \sqrt{\frac{2\ln(2|\Phi|| \Ycal|\hT_{y}(\phi_l)/\delta)}{n}}.
\end{align*}
Thus, if $\phi^s_l \in \Phi$, then we have that for any $\delta>0$, with probability at least $1-\delta$,  the following holds:
\begin{align*}
&   \PP_{}( \phi_l^{s}(X ) \notin \hA_{y}(\phi_{l}^{s}),Y=y)-\frac{|\hIcal_{\hT_{y}(\phi_l^{s})+1} ^{y}(\phi_l^{s})|}{n}
\\ &\ \le \sqrt{\frac{2\PP( \phi_l^{s}(X ) \notin \hA_{y}(\phi_{l}^{s}),Y=y)\ln(2|\Phi|| \Ycal|/\delta)}{n}}. 
\end{align*}
\textit{and  for all $k \in[\hT_{y}(\phi_l^{s})]$, }
\begin{align*}
&\PP(\phi_l^{s}(X)= \ha_{k}^{y}(\phi_l^{s}),Y=y)-\frac{|\hIcal_k ^{y}(\phi_l^{s})|}{n} \\ & \le\sqrt{\PP_{}(\phi_l^{s}(X)= \ha_{k}^{y}(\phi_l^{s}),Y=y)} \sqrt{\frac{2\ln(2|\Phi|| \Ycal|\hT_{y}(\phi_l^{s})/\delta)}{n}}.
\end{align*}
By using \eqref{eq:7}, this means that  if $\phi^s_l \in \Phi$,  for any $\delta>0$, with probability at least $1-\delta$,  the following holds:
\begin{align*}
 \PP_{}( Z \notin\Zcal^s_{\gamma,y} ,Y=y)-\frac{|\tIcal^{y}|}{n} \le \sqrt{\frac{2\PP( Z \notin\Zcal^s_{\gamma,y} ,Y=y)\ln(2|\Phi|| \Ycal|/\delta)}{n}}, 
\end{align*}
\textit{and for all $k \in[T_{y}]$, }
\begin{align*}
 \PP_{}(Z=a_{k}^{y},Y=y)-\frac{|\Ical_k ^{y}|}{n}\le\sqrt{\PP(Z=a_{k}^{y},Y=y)} \sqrt{\frac{2\ln(2|\Phi|| \Ycal|T_{y}/\delta)}{n}}.
\end{align*}
Since $\ell_{l}(z_{i},y)\ge 0$ and  $\ell_{l}(a_k^y,y)\ge 0$, this implies that  if $\phi^s_l \in \Phi$,  for any $\delta>0$, with probability at least $1-\delta$,  the following holds:
\begin{align*}
& \sum_{y \in \tYcal}\frac{1}{|\tIcal^{y}|}\sum_{i\in\tIcal^{y} } \ell_{l}(z_{i},y_{})\left(\PP(Y=y,Z\notin A_{y})-\frac{|\tIcal^{y}|}{n}  \right)
\\ &\le\left(\sum_{y \in \tYcal}\sqrt{\PP_{}( Z \notin\Zcal^s_{\gamma,y} ,Y=y)}\frac{\sum_{i\in\tIcal^{y} } \ell_{l}(z_{i},y)}{|\tIcal^{y}|} \right) \sqrt{\frac{2\ln(2|\Phi|| \Ycal|/\delta)}{n}}, 
\end{align*}
\textit{and for all $k \in[T_{y}]$, }
\begin{align*}
&\sum_{y \in \Ycal}\sum_{k=1}^{T_{y}} \ell_{l}(a_k^{y},y) \left(\PP_{}(Y=y,Z=a_k^{y})-\frac{|\Ical_k^{y}|}{n} \right)
\\ &\le \sum_{y \in \Ycal} \left(\sum_{k=1}^{T_{y}}\ell_{l}(a_k^{y},y) \sqrt{\PP_{}(Z=a_{k}^{y},Y=y)}\right) \sqrt{\frac{2\ln(2|\Phi|| \Ycal|T_{y}/\delta)}{n}}.
\end{align*}
Here, using Lemma \ref{lemma:8}, we have that $T_{y} =|\Zcal_{\gamma,y} ^{s}|\le2^{H_{}(Z_{y})+c_{l}^{y}(\phi_l^{s})\sqrt{\frac{m\ln(\sqrt{n}/\gamma)}{2}}}$. Thus, 
\begin{align*}
 \sqrt{\frac{2\ln(2|\Phi|| \Ycal|T_{y}/\delta)}{n}}&= \sqrt{\frac{2\ln(T_{y})+2\ln(2|\Phi|| \Ycal|/\delta)}{n}}
\\  & \le  \sqrt{\frac{2\left(H_{}(Z_{y})+c_{l}^{y}(\phi_l^{s})\sqrt{\frac{m\ln(\sqrt{n}/\gamma)}{2}} \right)\ln(2)+2\ln(2|\Phi|| \Ycal|/\delta)}{n}}
\end{align*}
Finally, since $H(Z_{y}) = I_{}(X_{y};Z_{y})+H_{}(Z_{y}|X_{y})$, we have that 
$$
H_{}(Z_{y})+c_{l}^{y}(\phi_l^{s})\sqrt{\frac{m\ln(\sqrt{n}/\gamma)}{2}}= I_{}(X_{y};Z_{y})+G_2^{y}.
$$
Combining these, we have proven the desired statement of this lemma. 
\end{proof}

\subsubsection{Combine Lemmas}

By combining Lemmas \ref{lemma:paper:1}, \ref{lemma:paper:2}, and \ref{lemma:paper:3}, we have proven the following statement:

\begin{restatable}{lemma}{lemmaa}
\label{lemma:3}
Let $l \in \{1,\dots,D\}$. If $\phi^s_l \in \Phi$, then  for any $\gamma>0$ and $\delta>0$, with probability at least $1-\delta$, the following holds:\begin{align*}
&\EE_{X,Y}[\ell(( g_{l} ^{s}\circ \phi_{l}^{s})(X),Y)] - \frac{1}{n} \sum_{i=1}^n \ell(( g_{l} ^{s}\circ \phi_{l}^{s})(x_{i}), y_i)
\\ &\le G_3   \sqrt{ \frac{ I(X;Z|Y)\ln(2)+G_2 \ln(2)+\ln(2|\Phi|| \Ycal|/\delta)}{n}} + \frac{G_1 (\ln|\Phi|)}{\sqrt{n}}, 
\end{align*}
where
\begin{align*}
&  G_1 (q)=\frac{\Lcal(f^s)\sqrt{2\gamma|\Ycal|}}{n^{1/4}} \sqrt{\ln (q)+\ln(2| \Ycal|/\delta) } 
+\gamma \Rcal(f^s), 
\\ & G_2 =\EE_y[c_{l}^{y}(\phi_{l}^{s})]\sqrt{\frac{m\ln(\sqrt{n}/\gamma)}{2}}+H_{}(Z_{}|X_{},Y),
\\ &G_3 = \max_{y \in \Ycal}  \sum_{k=1}^{T_{y}}\ell_{l}(a_k^{y},y) \sqrt{2|\Ycal|\PP_{}(Z=a_{k}^{y}|Y=y)}.
\end{align*}
\end{restatable}
\begin{proof}
Define the radius of the expected loss $R$ by
\begin{align}
R = \EE_y [R_y ]= \EE_y \left[ \EE_{Z,Y}[\ell_{l}(Z,Y)|Z\notin A_{y},Y=y]\right],
\end{align}
and the maximum over $y$ of the average training loss  per $y$ by 
\begin{align}
\hL(f^{s})=\max_{y\in\tYcal} \frac{1}{|\tIcal^{y}|}\sum_{i\in\tIcal^{y} } \ell_{l}(z_{i},y)=\max_{y\in\tYcal} \frac{1}{|\tIcal^{y}|}\sum_{i\in\tIcal^{y} } \ell(f^{s}(x_{i}),y).
\end{align}
Let $l \in \{1,\dots,D\}$. By combining  Lemmas \ref{lemma:paper:1}, \ref{lemma:paper:2}, and \ref{lemma:paper:3}, if $\phi^s_l \in \Phi$, then  for any $\gamma>0$ and $\delta>0$, with probability at least $1-\delta$, the following holds:
\begin{align*}
&\EE_{X,Y}[\ell(( g_{l} ^{s}\circ \phi_{l}^{s})(X),Y)] - \frac{1}{n} \sum_{i=1}^n \ell(( g_{l} ^{s}\circ \phi_{l}^{s})(x_{i}), y_i)
\\ &\le \sqrt{2}  \sum_{y \in \Ycal} \left(\sum_{k=1}^{T_{y}}\ell_{l}(a_k^{y},y) \sqrt{\PP_{}(Z=a_{k}^{y},Y=y)}\right)  \sqrt{\frac{\left( I_{}(X_{y};Z_{y})+G_2 ^{y}\right)\ln(2)+\ln(2|\Phi|| \Ycal|/\delta)}{n}} \\ & \quad +\left(\sum_{y \in \tYcal}\sqrt{\PP_{}( Z \notin\Zcal^s_{\gamma,y} ,Y=y)}\frac{\sum_{i\in\tIcal^{y} } \ell_{l}(z_{i},y)}{|\tIcal^{y}|} \right) \sqrt{\frac{2\ln(2|\Phi|| \Ycal|/\delta)}{n}} +\sum_{y \in \Ycal} \PP_{}(Y=y)\frac{\gamma R_{y}}{\sqrt{n}}.
\end{align*} 
Define 
$$
\tilde G_3 =\max_{y \in \Ycal} \sqrt{2} \sum_{k=1}^{T_{y}}\ell_{l}(a_k^{y},y) \sqrt{\PP_{}(Z=a_{k}^{y}|Y=y)}.
$$
Then, we have 
\begin{align*}
&\sqrt{2}\left(\sum_{k=1}^{T_{y}}\ell_{l}(a_k^{y},y) \sqrt{\PP_{}(Z=a_{k}^{y},Y=y)} \right)
\\  &= \sqrt{\PP_{}(Y=y)}\sqrt{2}\left(\sum_{k=1}^{T_{y}}\ell_{l}(a_k^{y},y) \sqrt{\PP_{}(Z=a_{k}^{y}|Y=y)} \right) 
\\ & \le\tilde G_3   \sqrt{\PP_{}(Y=y)}   
\end{align*}
Using this and Jensen's inequality, we have that 
\begin{align*}
&\sqrt{2}  \sum_{y \in \Ycal} \left(\sum_{k=1}^{T_{y}}\ell_{l}(a_k^{y},y) \sqrt{\PP_{}(Z=a_{k}^{y},Y=y)}\right)  \sqrt{\frac{\left( I_{}(X_{y};Z_{y})+G_2 ^{y}\right)\ln(2)+\ln(2|\Phi|| \Ycal|/\delta)}{n}}
\\ &\le \tilde G_3  \sum_{y \in \Ycal} \frac{|\Ycal|}{|\Ycal|} \sqrt{\PP(Y=y)} \sqrt{\frac{\left( I_{}(X_{y};Z_{y})+G_2 ^{y}\right)\ln(2)+\ln(2|\Phi|| \Ycal|/\delta)}{n}} 
\\ & \le \tilde G_3  |\Ycal|  \sqrt{\sum_{y \in \Ycal} \frac{1}{|\Ycal|} \frac{\PP(Y=y)\left( I_{}(X_{y};Z_{y})+G_2 ^{y}\right)\ln(2)+\PP(Y=y)\ln(2|\Phi|| \Ycal|/\delta)}{n}} 
\\ & =\tilde G_3  \sqrt{|\Ycal|}  \sqrt{ \frac{\sum_{y \in \Ycal} \PP(Y=y)\left( I_{}(X_{y};Z_{y})+G_2 ^{y}\right)\ln(2)+\sum_{y \in \Ycal} \PP(Y=y)\ln(2|\Phi|| \Ycal|/\delta)}{n}} 
\\ & =\tilde G_3  \sqrt{|\Ycal|}  \sqrt{ \frac{\left( I_{}(X;Z|Y)+G_2 \right)\ln(2)+\ln(2|\Phi|| \Ycal|/\delta)}{n}} 
\end{align*}
where
$$
G_2 =\sum_{y \in \Ycal} \PP(Y=y)G_{2}^y=\sum_{y \in \Ycal} \PP(Y=y)\left(c_{l}^{y}(\phi_{l}^{s})\sqrt{\frac{m\ln(\sqrt{n}/\gamma)}{2}}+H_{}(Z_{y}|X_{y})\right).
$$
Moreover, 
$$
\sum_{y \in \Ycal} \PP_{}(Y=y)\frac{\gamma R_{y}}{\sqrt{n}} = \frac{\gamma _{}}{\sqrt{n}} \sum_{y \in \Ycal} \PP_{}(Y=y)R_{y}=\frac{\gamma R_{}}{\sqrt{n}}.
$$
Using Lemma \ref{lemma:5} and Jensen's inequality,
since $
\PP_{Z} (Z\notin \Zcal_{\gamma,y}^{s}|Y=y)\le \frac{\gamma}{\sqrt{n}},
$
\begin{align*}
&\sum_{y \in \tYcal}\sqrt{\PP_{}( Z \notin\Zcal^s_{\gamma,y} ,Y=y)}\frac{\sum_{i\in\tIcal^{y} } \ell_{l}(z_{i},y)}{|\tIcal^{y}|} 
\\ &= \sum_{y \in \tYcal}\sqrt{\PP_{}( Z \notin\Zcal^s_{\gamma,y} |Y=y)} \sqrt{\PP_{}(Y=y)}\frac{\sum_{i\in\tIcal^{y} } \ell_{l}(z_{i},y)}{|\tIcal^{y}|} 
\\ &\le \hL(f^s)\frac{\sqrt{\gamma}}{n^{1/4}}   \sum_{y \in \Ycal} \frac{|\Ycal|}{|\Ycal|} \sqrt{\PP_{}(Y=y)} \\ &\le \hL(f^s) \frac{\sqrt{\gamma}}{n^{1/4}}  |\Ycal|   \sqrt{\sum_{y \in \Ycal} \frac{1}{|\Ycal|}\PP_{}(Y=y)}
  \\ & = \hL(f^s) \frac{\sqrt{\gamma|\Ycal|}}{n^{1/4}}    
\end{align*}
Thus,
since  $R \le \Rcal(f^s)$ and $\hL(f^s) \le \Lcal(f^s)$,\begin{align*}
\frac{ G_1(\ln |\Phi|)}{\sqrt{n}} &\ge \left(\sum_{y \in \tYcal}\sqrt{\PP_{}( Z \notin\Zcal^s_{\gamma,y} ,Y=y)}\frac{\sum_{i\in\tIcal^{y} } \ell_{l}(z_{i},y)}{|\tIcal^{y}|} \right) \sqrt{\frac{2\ln(2|\Phi|| \Ycal|/\delta)}{n}} 
\\ & \quad +\sum_{y \in \Ycal} \PP_{}(Y=y)\frac{\gamma R_{y}}{\sqrt{n}}, 
\end{align*}
where
$$
G_1 (q)=\frac{\Lcal(f^s)\sqrt{2\gamma|\Ycal|}}{n^{1/4}} \sqrt{q+\ln(2| \Ycal|/\delta) } 
+\gamma \Rcal(f^s).
$$
Combining these imply the desired statement.
\end{proof}

\subsection{Completing the Proof of Theorem \ref{thm:1} with Key Lemmas} \label{app:2}

Recall that we have proven the following lemma in the previous subsection:

\lemmaa*

Theorem \ref{thm:2} directly follows from Lemma \ref{lemma:3}; i.e., we complete the proof of Theorem \ref{thm:2} using Lemma \ref{lemma:3}. Since  $\phi^s_l$ is fixed independently of the training dataset $s$ in Theorem \ref{thm:1}, we can invoke Lemma \ref{lemma:3} with  $\Phi=\{\phi_l^s\}$, with which $|\Phi|=1$ and $\phi^s_l \in \Phi$. Thus, by noticing that $f^s = g_{l} ^{s}\circ \phi_{l}^s$ for any $l\in \{1,\dots,D\}$, Lemma \ref{lemma:3} implies the desired statement.
\qed

\subsection{Completing the Proof of Theorem \ref{thm:2}  with Key Lemmas} \label{app:3}
We complete the proof of Theorem \ref{thm:2} by extending Lemma \ref{lemma:3} in the following.

\subsubsection{Finding a likely space of  encoder} \label{sec:2}

Fix $l \in \{1,\dots,D\}$ throughout this section. 
Let $\lambda =\lambda_l$ and $C_{\lambda}=C_{\lambda,l}$. Recall  that $\Acal_{l}(s)\in \Mcal_l $ and $|\Mcal_l| <\infty$. For simplicity of notation, we define the random variable $A_s$ by $A_s=\phi^{S}_l$. For any $q\in \Mcal_l$, we denote 
\begin{align}
    p(q)=\ \PP(A_s=q).
\end{align}
The entropy of the random variable $A_s$ is given by
$$
\EE_{A_s}\left[- \log p(A_s) \right]= H(A_s).
$$
 Define the typical subset
$$
\Phi_{\epsilon}^{l}= \left\{ \phi_{l}\in\Mcal_l: - \log \PP(A_s=\phi_{l})- H(A_s) \le \epsilon \right\}. 
$$
The following proposition shows that the probability of going outside of the typical subset $\Phi_\epsilon^l$ is bounded by $\delta$ when we take $\epsilon=(1/\lambda)\log(C_\lambda /\delta)$: 
\begin{lemma} \label{lemma:1}
For any $\lambda>0$, if we take $\epsilon=(1/\lambda)\ln(C_\lambda /\delta)$, then we have
\begin{align}
    \PP(\phi_{l}^{S}\not\in \Phi_\epsilon^l)\leq \delta,
\end{align}
and 
\begin{align}\label{e:sizebound}
|\Phi_{\epsilon}^l| \leq 2^{H(\phi_{l}^{S})+\epsilon}=2^{H(\phi_{l}^{S})+ \frac{1}{\lambda}\log \frac{C_\lambda }{\delta}}. 
\end{align}
\end{lemma}
\begin{proof}
By the definition of the set $\Phi_\epsilon^l$, we have
\begin{align}
    \PP(A_s\not\in \Phi_\epsilon^l)
    &=\PP(q\in \Mcal_l: -\log p(q)\geq H(A_s)+\epsilon)\\
    &=\PP(q\in \Mcal_l: -\lambda\log p(q)\geq \lambda H(A_s)+\lambda\epsilon)\\
    &=\PP(q\in \Mcal_l: p^{-\lambda}(q)\geq e^{\lambda H(A_s)+\lambda\epsilon})\\
    &\leq e^{-\lambda H(A_s)-\lambda\epsilon}\sum_{q\in \Mcal_l} p^{-\lambda}(q)p(q)\\
    &=\frac{C_\lambda}{e^{\lambda \epsilon}}=\delta.
\end{align}

Now we compute the size of $\Phi_{\epsilon}^{l}$. From the definition of $\Phi_{\epsilon}^l$,
we have
\begin{align*}
  - \log p(\phi_{l})- H(A_s)\le \epsilon &\Longleftrightarrow  - \log p(\phi_{l}) \le H(A_s)+\epsilon 
\\  &\Longleftrightarrow
-H(A_s)-\epsilon\le \log p(\phi_{l})  
\\  &\Longleftrightarrow
2^{-H(A_s)-\epsilon} \le p(\phi_{l}) .
\end{align*}
Using $2^{-H(A_s)-\epsilon} \le p(\phi_{l})$, 
$$
1\ge \PP_{s}(A_{s}\in \Phi_{\epsilon}^{l}) = \sum_{\phi_{l} \in\Phi_{\epsilon}^l}  \PP(A_{s}= \phi_{l}) \ge \sum_{\phi_{l} \in\Phi_{\epsilon}^l}  2^{-H(A_s)-\epsilon} =|\Phi_{\epsilon}^l|2^{-H(A_s)-\epsilon} .
$$
This implies that using $\epsilon= (1/\lambda)\ln(C_\lambda /\delta)$, $$
|\Phi_{\epsilon}^l| \le 2^{H(A_s)+ \frac{1}{\lambda}\ln \frac{C_\lambda }{\delta}}. 
$$
\end{proof}

\subsubsection{Result with  fixed Layer Index}
Combining Lemmas \ref{lemma:3}
and  \ref{lemma:1} implies the following lemma, which is a main result for a fixed layer index $l$:
\begin{lemma} \label{lemma:6}
Let $l \in \{1,\dots,D\}$.
  Then, for any $\gamma_{}>0$ and any $\delta>0$, with probability at least $1-\delta$, the following holds:\begin{align}
&\EE_{X,Y}[\ell(f^{s}(X),Y)] - \frac{1}{n} \sum_{i=1}^n \ell(f^{s}(x_{i}), y_i)
\\ \nonumber &\le   G_3 \sqrt{ \frac{\left( I_{}(X;Z|Y)+I(\phi_{l}^{S}; S)+G_2 +G_{4}\right)\ln(2)+\ln(4| \Ycal|/\delta)}{n}} + \frac{G_1(\tilde q)}{\sqrt{n}}.
\end{align}
where $\tilde q=(I_{}(\phi_{l}^{S}; S)+ G_{4}^{})\ln (2)+\ln(2)$,
\begin{align*}
&G_1 (\tilde q)=\frac{\Lcal(f^s)\sqrt{2\gamma|\Ycal|}}{n^{1/4}} \sqrt{\tilde q+\ln(2| \Ycal|/\delta) } 
+\gamma \Rcal(f^s),
\\ & G_2  =\EE_y[c_{l}^{y}(\phi_{l}^{s})]\sqrt{\frac{m\ln(\sqrt{n}/\gamma)}{2}}+H_{}(Z_{}|X_{},Y),
\\ & G_3  =\max_{y \in \Ycal}  \sum_{k=1}^{T_{y}}\ell_{l}(a_k^{y},y) \sqrt{2|\Ycal|\PP_{}(Z=a_{k}^{y}|Y=y)},
\\ &G_4 =\frac{1}{\lambda}\ln \frac{C_\lambda }{\delta} +H_{}(\phi_{l}^{S}|S).
\end{align*}
\end{lemma}
\begin{proof}
Fix $l \in \{1,\dots,D\}$. Let $\lambda>0$ and $\epsilon=(1/\lambda)\ln(C_\lambda /\delta)$. Using Lemma \ref{lemma:3}, if $\phi^s_l \in \Phi_{\epsilon}^l$, then  for any $\gamma>0$ and $\delta>0$, with probability at least $1-\delta$, the following holds:
\begin{align} \label{eq:5}
&\EE_{X,Y}[\ell(( g_{l} ^{s}\circ \phi_{l}^{s})(X),Y)] - \frac{1}{n} \sum_{i=1}^n \ell(( g_{l} ^{s}\circ \phi_{l}^{s})(x_{i}), y_i)
\\ \nonumber &\le  G_3   \sqrt{ \frac{\left( I_{}(X;Z|Y)+G_2 \right)\ln(2)+\ln(2|\Phi_{\epsilon}^l|| \Ycal|/\delta)}{n}} + \frac{G_1(\ln |\Phi_{\epsilon}^l|)}{\sqrt{n}}.
\end{align}
From Lemma \ref{lemma:1},
$$
\PP(A_s\not\in \Phi_\epsilon^l)\leq \delta
$$
$$
\PP_{s}(\phi_{l}^{s}\notin \Phi_{\epsilon}^{l})\le \bdelta.
$$
Thus, since $\PP(A \cap B) \le \PP(B)$ and $\PP(A \cap B)=\PP(A)\PP(A \mid B)$, we have that 
\begin{align*}
& \PP_{S}(\text{Inequality \eqref{eq:5}  holds})
\\ &\ge \PP_{S}(\phi_{l}^{S}\in \Phi_{\epsilon}^l \ \bigcap \ \text{Inequality \eqref{eq:5}  holds})
\\ & =\PP_{S}(\phi_{l}^{S}\in \Phi_{\epsilon}^{l})\PP_{S}(\text{Inequality \eqref{eq:5}  holds}\mid \phi_{l}^{S}\in \Phi_{\epsilon}^{l})
\\ & \ge \PP_{S}(\phi_{l}^{S}\in \Phi_{\epsilon}^{l})(1-\delta)
\\ & \ge(1-\delta)(1-\delta)=1-2\delta+\delta^{2}\ge1-2\delta.
\end{align*} 
Therefore, by setting $\delta=\frac{\delta'}{2}$, we have that for any $\delta'>0$, 
$$
\PP_{S}(\text{Eq \eqref{eq:5}  holds}) \ge1-\delta'. 
$$
In other words, for any $\gamma>0$ and $\delta>0$, with probability at least $1-\delta$, the following holds:
\begin{align}
&\EE_{X,Y}[\ell(( g_{l} ^{s}\circ \phi_{l}^{s})(X),Y)] - \frac{1}{n} \sum_{i=1}^n \ell(( g_{l} ^{s}\circ \phi_{l}^{s})(x_{i}), y_i)
\\ \nonumber &\le  G_3   \sqrt{ \frac{\left( I_{}(X;Z|Y)+G_2 \right)\ln(2)+\ln(4|\Phi_{\epsilon}^l|| \Ycal|/\delta)}{n}} + \frac{G_1(\ln 2|\Phi_{\epsilon}^l|)}{\sqrt{n}}.
\end{align}

From Lemma \ref{lemma:1},
 we have  $|\Phi_{\epsilon}^l| \leq 2^{H(\phi_{l}^{S})+ \frac{1}{\lambda}\ln \frac{C_\lambda }{\delta}}$ and thus 
$$
\ln(4|\Phi_{\epsilon} ^{l}|| \Ycal|/\delta)=\ln(|\Phi_{\epsilon} ^{l}|)+\ln(4| \Ycal|/\delta)\le\left( H(\phi_{l}^{S})+ \frac{1}{\lambda}\ln \frac{C_\lambda }{\delta}\right)\ln(2)+\ln(4| \Ycal|/\delta). 
$$
From the definition of the entropy, conditional entropy, and mutual information, we have that 
$$
H(\phi_{l}^{S})= I_{}(\phi_{l}^{S}; S)+H_{}(\phi_{l}^{S}|S).
$$
Using this, 
$$
H(\phi_{l}^{S})+ \frac{1}{\lambda}\ln \frac{C_\lambda }{\delta} =I_{}(\phi_{l}^{S}; S)+G_{4}. 
$$
By combining these and noticing that $f^s = g_{l} ^{s}\circ \phi_{l}^s$ for any $l\in \{1,\dots,D\}$, we have that for any $\gamma>0$ and $\delta>0$, with probability at least $1-\delta$, the following holds:
\begin{align}
&\EE_{X,Y}[\ell(f^{s}(X),Y)] - \frac{1}{n} \sum_{i=1}^n \ell(f^{s}(x_{i}), y_i)
\\ \nonumber &\le  G_3    \sqrt{ \frac{\left( I_{}(X;Z|Y)+I(\phi_{l}^{S}; S)+G_2 +G_{4}\right)\ln(2)+\ln(4| \Ycal|/\delta)}{n}} + \frac{G_1(\tilde q)}{\sqrt{n}}.
\end{align}
\end{proof}

\subsubsection{Completing the Proof}
We complete the proof of Theorem \ref{thm:2} using Lemma \ref{lemma:6}.  Let $\gamma_{l}>0$ and $\lambda_{l}>0$ for all $l \in \{1,2,\dots,D+1\}$. Recall that  $f^s = g_{l} ^{s}\circ \phi_{l}^s$ for any $l\in \{1,\dots,D\}$. Thus, by making the dependence of the layer index $l$ explicit,  Lemma \ref{lemma:6} states that  for any $\delta>0$ and (fixed) $l\in \{1,\dots,D\}$, with probability at least $1-\delta$, 
\begin{align} \label{eq:new:3}
&\EE_{X,Y}[\ell(f^{s}(X),Y)] - \frac{1}{n} \sum_{i=1}^n \ell(f^{s}(x_{i}), y_i)
\\ \nonumber &\le  G_3^l \sqrt{ \frac{\left( I_{}(X;Z_{l}^s|Y)+I(\phi_{l}^{S}; S)+G_2 ^{l}+G_{4}^l\right)\ln(2)+\ln(4| \Ycal|/\delta)}{n}} + \frac{G_1^l(\tilde q)}{\sqrt{n}},
\end{align}
where  $\tilde q=(I_{}(\phi_{l}^{S}; S)+ G_{4}^{})\ln (2)+\ln(2)$,
\begin{align*}
& G_1^{l}(q) =\frac{\Lcal(f^s)\sqrt{2\gamma_{l}|\Ycal|}}{n^{1/4}} \sqrt{q+\ln (2| \Ycal|/\delta)} 
+\gamma_{l} \Rcal(f^s),
\\ & G_2^l  =\EE_y[c_{l}^{y}(\phi_{l}^{s})]\sqrt{\frac{m\ln(\sqrt{n}/\gamma_{l})}{2}}+H_{}(Z_{l}^s|X_{},Y).
\\ & G_3^{l}  =\max_{y \in \Ycal}  \sum_{k=1}^{T_{y}}\ell_{l}(a_k^{y},y) \sqrt{2|\Ycal|\PP_{}(Z=a_{k}^{y}|Y=y)}, 
\\ & \tilde G_4^{l} =\frac{1}{\lambda_{l}}\ln \frac{C_{\lambda_{l},l} }{\delta} +H_{}(\phi_{l}^{S}|S).
\end{align*}
We now consider the case of $l=D+1$. Let $l=D+1$ and $\lambda_{D+1}>0$.  Fix $f = \phi_{D+1} \in \Phi_{\epsilon}^{D+1}$ with  $\epsilon=(1/\lambda)\ln(C_{\lambda_{D+1},D+1} /\delta)$. Then, by using Hoeffding's inequality, for any $\delta>0$, with probability at least $1-\delta$,
$$
\EE_{X,Y}[\ell(f(X),Y)] - \frac{1}{n} \sum_{i=1}^n \ell(f(x_{i}), y_i)\le \Rcal(f) \sqrt{\frac{\ln(1/\delta)}{2n}}.
$$ 
By taking union bounds over elements of $\Phi_{\epsilon}^{D+1}$, this implies that for any $\delta>0$, with probability at least $1-\delta$,
 the following holds for all $f \in \Phi_{\epsilon}^{D+1}$,
$$
\EE_{X,Y}[\ell(f(X),Y)] - \frac{1}{n} \sum_{i=1}^n \ell(f(x_{i}), y_i)\le \Rcal(f) \sqrt{\frac{\ln(|\Phi_{\epsilon}^{D+1}|/\delta)}{2n}}.
$$   
This implies that for any $\delta>0$, if $\phi_{D+1}^s\in \Phi_{\epsilon}^{D+1}$, then with probability at least $1-\delta$,
\begin{align} \label{eq:new:1}
\EE_{X,Y}[\ell(f^{s}(X),Y)] - \frac{1}{n} \sum_{i=1}^n \ell(f^{s}(x_{i}), y_i)\le \Rcal(f^{s}) \sqrt{\frac{\ln(|\Phi_{\epsilon}^{D+1}|/\delta)}{2n}}.
\end{align}

Here, from Lemma \ref{lemma:1}, we have that 
\begin{align*}
    \PP(\phi_{D+1}^{S}\not\in \Phi_\epsilon^{D+1})\leq \delta.
\end{align*} 
Since $\PP(A \cap B) \le \PP(B)$ and $\PP(A \cap B)=\PP(A)\PP(A \mid B)$, we have that 
\begin{align*}
& \PP_{S}(\text{Inequality \eqref{eq:new:1}  holds})
\\ &\ge \PP_{S}(\phi_{{D+1}}^{S}\in \Phi_{\epsilon}^{D+1} \ \bigcap \ \text{Inequality \eqref{eq:new:1}  holds})
\\ & =\PP_{S}(\phi_{{D+1}}^{S}\in \Phi_{\epsilon}^{{D+1}})\PP_{S}(\text{Inequality \eqref{eq:new:1}  holds}\mid \phi_{{D+1}}^{S}\in \Phi_{\epsilon}^{{D+1}})
\\ & \ge \PP_{S}(\phi_{{D+1}}^{S}\in \Phi_{\epsilon}^{{D+1}})(1-\delta)
\\ & \ge(1-\delta)(1-\delta)
\\ &\ge1-2\delta.
\end{align*} 
Therefore, by setting $\delta=\frac{\delta'}{2}$, we have that for any $\delta'>0$, 
$$
\PP_{S}(\text{Eq \eqref{eq:new:1}  holds}) \ge1-\delta'. 
$$In other words,  for any $\delta'>0$,  with probability at least $1-\delta'$,
\begin{align} \label{eq:new:2}
\EE_{X,Y}[\ell(f^{s}(X),Y)] - \frac{1}{n} \sum_{i=1}^n \ell(f^{s}(x_{i}), y_i)\le\Rcal(f^{s}) \sqrt{\frac{\ln(2|\Phi_{\epsilon}^{D+1}|/\delta')}{2n}}.
\end{align}
Here, from Lemma \ref{lemma:1}, we have that 
\begin{align*}
|\Phi_{\epsilon}^{D+1}| \leq 2^{H(\phi_{D+1}^{S})+ \frac{1}{\lambda_{D+1}}\log \frac{C_{\lambda_{D+1},D+1} }{\delta}}. 
\end{align*} 
Substituting this,
\begin{align*}
\ln(2|\Phi_{\epsilon}^{D+1}|/\delta')&=\ln(|\Phi_{\epsilon}^{D+1}|)+\ln(2/\delta')
\\ & \le\left(H(\phi_{D+1}^{S})+ \frac{1}{\lambda_{D+1}}\log \frac{C_{\lambda_{D+1},D+1} }{\delta} \right)\ln(2)+\ln(2/\delta') 
\end{align*}
Using $H(\phi_{D+1}^{S})= I_{}(\phi_{D+1}^{S}; S)+H_{}(\phi_{D+1}^{S}|S)$, $$
H(\phi_{D+1}^{S})+ \frac{1}{\lambda_{D+1}}\log \frac{C_{\lambda_{D+1},D+1} }{\delta}=I_{}(\phi_{D+1}^{S}; S)+\tilde G_{4}^{D+1}.
$$
Substituting these into \eqref{eq:new:2}, we have that  for any $\delta>0$,  with probability at least $1-\delta$,
\begin{align} \label{eq:new:4}
\EE_{X,Y}[\ell(f^{s}(X),Y)] - \frac{1}{n} \sum_{i=1}^n \ell(f^{s}(x_{i}), y_i)\le\Rcal(f^{s}) \sqrt{\frac{\left(I_{}(\phi_{D+1}^{S}; S)+\tilde G_{4}^{D+1} \right)\ln(2)+\ln(2/\delta)}{2n}}.
\end{align}
By combining \eqref{eq:new:3} and \eqref{eq:new:4} with union bounds over $\bD$, we have that for any $\delta>0$ and $\bD\subseteq \{1,2,\dots,D+1\}$, with probability at least $1-\delta$, the following holds for all $l \in \bD$: 
\begin{align}
&\EE_{X,Y}[\ell(f^{s}(X),Y)] - \frac{1}{n} \sum_{i=1}^n \ell(f^{s}(x_{i}), y_i)
\\ \nonumber &\le  \one\{l\neq D_{+}\}\left(G_3^l \sqrt{ \frac{\left( I_{}(X;Z_{l}^s|Y)+I(\phi_{l}^{S}; S)+G_2 ^{l}+G_{4}^l\right)\ln(2)+\ln(4| \Ycal||\bD|/\delta)}{n}} + \frac{G_1^l}{\sqrt{n}}\right) \\ \nonumber & \quad +\one\{l=D_{+}\}\Rcal(f^{s}) \sqrt{\frac{\left(I_{}(\phi_{D+1}^{S}; S)+G_{4}^{D+1} \right)\ln(2)+\ln(2/\delta)}{2n}},
\end{align}
where $D_{+}=D+1$,
\begin{align*}
G_4^{l} &=\frac{1}{\lambda_{l}}\ln \frac{C_{\lambda_{l},l}|\bD| }{\delta} +H_{}(\phi_{l}^{S}|S).
\end{align*}
Since the right-hand side of this inequality holds for all $l \in \bD$ and the left-hand side does not depend on $l$, this implies that 
for any $\delta>0$ and  $\bD\subseteq \{1,2,\dots,D+1\}$, with probability at least $1-\delta$, 
\begin{align} 
&\EE_{X,Y}[\ell(f^{s}(X),Y)] - \frac{1}{n} \sum_{i=1}^n \ell(f^{s}(x_{i}), y_i)
\le \min_{l \in \bD} Q_l ,
\\ \nonumber \text{ where } &Q_l =\begin{cases}G_3^l \sqrt{ \frac{\left( I_{}(X;Z_{l}^s|Y)+I(\phi_{l}^{S}; S) \right) \ln(2)+ \widehat \Gcal_2^l}{n}} + \frac{G_1^l(\zeta)}{\sqrt{n}} & \text{if } l \le D \\
\Rcal(f^{s}) \sqrt{\frac{I_{}(\phi_{l}^{S}; S)\ln(2)  + \check \Gcal_2^l }{2n}} & \text{if } l = D+1, \\
\end{cases}
\end{align}
where   $\zeta=(I_{}(\phi_{l}^{S}; S)+G_{4}^{l})\ln (2)+\ln(2|\bD|)$, $\widehat \Gcal_2 ^l=\left(G_2 ^{l}+G_{4}^l\right)\ln(2)+\ln(4| \Ycal||\bD|/\delta)$, $\check \Gcal_2^l=G_{4}^{l} \ln(2)+\ln(2/\delta)$, and $G_4^l =\frac{1}{\lambda_{l}}\ln \frac{C_{\lambda_{l},l}|\bD| }{\delta} +H_{}(\phi_{l}^{S}|S)$. 

\qed

\subsection{Proof of Remark \ref{remark:1} } \label{app:6}
The desired statement follows from 
$$
I(\theta_{l}^{S};S)+H(\theta_{l}^{S}|S)=H(\theta_{l}^{S})\ge H(\phi_{l,\theta^{S}_l})= H(\phi_{l}^{S})=I(\phi_{l}^{S};{S})+H(\phi_{l}^{S}|{S}), $$
where the inequality holds because all the randomness of  $\phi_{l,\theta^{S}_l}$ comes from the randomness of  $\theta_{l}^{S}=\Acal_l^\theta \circ {S}$ (where $\Acal_l^\theta$ is the version of $\Acal_l$ that outputs the parameter vector instead of the encoder function), and because one  $\phi_{l,\theta^{S}_l}$ corresponds to one or more   $\theta_{l}^{S}$; i.e., we have   $\phi_{l,\theta^s_l}=\phi_{l,\bar \theta^s_l}$ whenever    $\theta^s_l=\bar \theta^s_l$ and it is possible that $\phi_{l,\theta^s_l}=\phi_{l,\bar \theta^s_l}$ for     $\theta^s_l\neq\bar \theta^s_l$. In other words, the desired statement does not hold  only if  $\phi_{l,\theta^s_l}\neq \phi_{l,\bar \theta^s_l}$ for some    $\theta^s_l=\bar \theta^s_l$, which is not the case.

\qed

\subsection{Proof of Corollary \ref{coro:1}} \label{app:8}
\begin{proof}
Set  $\phi_{l}^{s} =\Ecal_{l}[\tphi_{l}^s] \circ\tphi_{l}^{s}$. Then, Theorems \ref{thm:1}--\ref{thm:2}  hold true for this choice of encoder $\phi_{l}^{s}$ since this does not violate any assumption of  Theorems \ref{thm:1}--\ref{thm:2}. Thus, Theorems \ref{thm:1}--\ref{thm:2} hold  with  \cref{eq:new:6}  and \cref{eq:11} in their original forms: i.e.,
\begin{align} \label{eq:new:9}
&\EE_{X,Y}[\ell(f^s (X),Y)] - \frac{1}{n} \sum_{i=1}^n \ell(f^s (x_{i}), y_i) \le \hQ_l, \ \ \text{and,} \\
\nonumber &\EE_{X,Y}[\ell(f^s (X),Y)] - \frac{1}{n} \sum_{i=1}^n \ell(f^s (x_{i}), y_i) \le\min_{l \in \bD} Q_l,
\end{align}
Since $\PP(|\ell((g_{l} ^{s}\circ \tphi_{l}^s)(X),Y)- \ell(( g_{l} ^{s}\circ\Ecal_{l}[\tphi_{l}^s]\circ  \tphi_{l}^{s})(X),Y)|\le C_{l} )=\PP(|\ell(\tf^{s}(X),Y)- \ell(f^s(X),Y)|\le C_{l} )=1$, we have that with probability one, 
\begin{align}
\ell(\tf^{s}(x_{i}),y_{i})= \ell(f^{s}(x_{i}),y_{i})+(\ell(\tf^{s}(x),y)- \ell(f^{s}(x_{i}),y_{i}))\le \ell(f^{s}(x_{i}),y_{i})+C_{l}.   
\end{align}
Thus, with probability one,
\begin{align} \label{eq:new:10}
\EE_{X,Y}[\ell(\tf^s (X),Y)] - \frac{1}{n} \sum_{i=1}^n \ell(\tf^s (x_{i}), y_i)\le\EE_{X,Y}[\ell(f^s (X),Y)] - \frac{1}{n} \sum_{i=1}^n \ell(f^s (x_{i}), y_i)+2C_{l}.  
\end{align}
Combining \cref{eq:new:9} and \cref{eq:new:10}
with union bounds concludes that Theorems \ref{thm:1}--\ref{thm:2} hold  when we
replace
\cref{eq:new:6}  and \cref{eq:11} by\begin{align}
&\EE_{X,Y}[\ell(\tf^s (X),Y)] - \frac{1}{n} \sum_{i=1}^n \ell(\tf^s (x_{i}), y_i) \le\hQ_{l}+2C_{l}, \ \ \text{and,} \\
\nonumber &\EE_{X,Y}[\ell(\tf^s (X),Y)] - \frac{1}{n} \sum_{i=1}^n \ell(\tf^s (x_{i}), y_i)\le\min_{l \in \bD} Q_l+2C_{l},
\end{align} 
Finally, the values of $\hQ_{l}$ and $Q_l$ are finite since  $|\Zcal_{l}^{s}| < \infty$ and  $|\Mcal_l | < \infty$; e.g., $|\Zcal_{l}^{s}| < \infty$ implies that $I(X;Z_{l}^s|Y)<\infty$. Thus, if $C_{\Ecal}<\infty$, we have $\hQ_{l}+2C_{l}<\infty$ and $\min_{l \in \bD} Q_l+2C_{l}<\infty$. \end{proof}

\subsection{Proof of Proposition \ref{prop:3}} \label{app:9}

\begin{proof} Let $l$ be fixed and $\phi^s=\phi^s_l$. For deterministic neural networks, the intermediate output $Z$ is a deterministic function of the input $X$, i.e. $Z=\phi^s(X)$. In this case the conditional mutual information between $X$ and $Z$ simplifies to the conditional entropy of $Z$:
\begin{align}
I(X,Z|Y)=H(Z|Y)=H(\phi^s(X)|Y). 
\end{align}
It has been proven in \citep{amjad2019learning} that if $X$ has absolutely continuous component, which has continous density on a compact set, and the activation is bi-Lipschitz or continuous differentiable with strictly positive derivative, then the entropy of $\phi^s(X)$ is infinite.

In the following, we first give some simple examples with ReLU activation where the entropy of $\phi^s(X)$ is finite for an initial intuition, and then we  generalize the examples for more practical settings. Finally, we discuss  generality and practicality of our construction. 

Consider an arbitrary (continuous or discrete) distribution such that the distribution of $X|Y$ consists of several components (which may correspond to further subclasses). For simplicity, we assume that there are two components $\mathcal C_1$ and $\mathcal C_2$.
We start with the linearly separable case where $\mathcal C_1$ and $\mathcal C_2$ are separated by a hyperplane $ax+b=0$ with margin at least $r$. In other words $ax+b\geq r$ for $x\in \mathcal C_1$, and $ax+b\leq -r$ for $x\in \mathcal C_2$. 
Then the following simple one layer network 
\begin{align}
\sigma(ax+b+c)-\sigma(ax+b-c),\quad 0<c\leq r
\end{align}
maps $x\in \mathcal C_1$ to $2c$ and $\mathcal C_2$ to $0$. Thus, the output $Z=\phi^s(X)$ follows a Bernoulli distribution, which has bounded entropy.
Thus, we have $I(X,Z|Y)<\infty$.

More generally, if $\mathcal C_1$ and $\mathcal C_2$ are separable, with margin at least $r$ using some metric $d(\mathcal C_1, \mathcal C_2)\geq r$, we can take a Lipschitz function $g$ w.r.t. this metric $d$ with lipschitz constant $1/r$ such that it equals $0$ on $C_1$ and equals $1$ on $C_2$. By the universal approximation theory, ReLU neural network can approximate arbitrary continuous function to arbitrary precision as we increase the network size. In particular, there exists a finite-size ReLU neural network $N$ such that $|N(x)-g(x)|\leq 1/8$. As a consequence, we have $N(x)\leq 1/8$ for $x\in \mathcal C_1$ and $N(x)\geq 7/8$ for $x\in \mathcal C_2$. We consider the following neural network:
\begin{align}
\sigma(N-1/2+c)-\sigma(N-1/2-c),\quad 0<c<3/8,
\end{align}
which  maps $x\in \mathcal C_1$ to $2c$ and $\mathcal C_2$ to $0$. Thus, the output $Z=\phi^s(X)$ follows a Bernoulli distribution, which has bounded entropy.
Thus, we have $I(X,Z|Y)<\infty$.
Since the distribution in this general example is arbitrary except for the separable components, there exists infinitely many such distributions. 

Finally, we observe that  these examples are  general and practical. First, the above proof works for any finite number of  separable  components instead of two components. Second,
it is also observed in practice that trained neural networks behave like these examples discussed above, which maps different class to different points; this is sometimes referred as a neural collapse phenomenon \citep{papyan2020prevalence}.
\end{proof} 

\reve

\subsection{Proof of Proposition \ref{prop:1}} \label{app:4}

We will use the following lemma to prove Proposition \ref{prop:1}:
\begin{lemma} \label{lemma:10}
Let $v_{1},\dots, v_T \in \RR$ such that $0\le v_k \le C e^{-(k/\beta)^\alpha}$ for some constants $\alpha\ge 1$ and $\beta,C>0$. Then, 
$$
\sum_{k=1}^T \sqrt{v_k} \le \sum_{k=1}^{\ceil{\tbeta}} \sqrt{v_k} + \frac{C \tbeta}{\alpha e}
$$
where  $\tbeta= 2^{1/\alpha}\beta$.
\end{lemma}
\begin{proof} Using the condition on $v_k$,
$$
\sqrt{v_k} \le \sqrt{C e^{-(k/\beta)^\alpha}} =\sqrt{C}\sqrt{ e^{-(k/\beta)^\alpha}}=\sqrt{C} e^{-\frac{k^{\alpha}}{2\beta^\alpha}} =\sqrt{C} e^{-(k/\tbeta)^\alpha}   
$$
Then,
$$
\sum_{k=1}^T \sqrt{v_k} = \sum_{k=1}^{\ceil{\tbeta}} \sqrt{v_k} + \sum_{k={\ceil\tbeta+1}}^{T} \sqrt{v_k} \le\sum_{k=1}^{\ceil{\tbeta}} \sqrt{v_k} + \sqrt{C } \sum_{k={\ceil\tbeta+1}}^{T} e^{-(k/\tbeta)^\alpha}
$$
We now bound the last term by using integral as$$
\sum_{k={\ceil\tbeta+1}}^{T} e^{-(k/\tbeta)^\alpha}\le \int_{\tbeta}^\infty e^{-(q/\tbeta)^\alpha}dq = \frac{\tbeta}{\alpha} \int_{(\tbeta/\tbeta)^\alpha}^\infty t^{\frac{1}{\alpha}-1} e^{-t} dt = \frac{\tbeta}{\alpha} \int_{1}^\infty t^{\frac{1}{\alpha}-1} e^{-t} dt   
$$
Here, since $\alpha \ge 1$ and $t \ge 1$ in  the integral, we have $t^{\frac{1}{\alpha}-1}\le 1$ in the integral. Thus, 
$$
  \int_{1}^\infty t^{\frac{1}{\alpha}-1} e^{-t} dt \le\int_{1}^\infty e^{-t} dt= e^{-1}. 
$$
By combining these, we have
$$
\sum_{k=1}^T \sqrt{v_k} \le\sum_{k=1}^{\ceil{\tbeta}} \sqrt{v_k} +\frac{C\tbeta }{\alpha e}  
$$
\end{proof}

Using Lemma \ref{lemma:10}, we complete the proof of Proposition \ref{prop:1} in the following:

\begin{proof}[Proof of Proposition \ref{prop:1}]
Let $y \in \Ycal$ and $l \in \{1,\dots,D\}$. To invoke Lemma  \ref{lemma:10}, we rearrange the expression of $G_3  ^{l}$ as 
\begin{align*}
G_3  ^{l}&=\max_{y \in \Ycal} \sum_{k=1}^{T_{y}^l}\ell( g_{l}^{s}(a_k^{l,y}),y) \sqrt{2|\Ycal|\PP_{}(Z_{l,y}^s=a_{k}^{l,y})} 
\\ & \le \sqrt{2|\Ycal|} \max_{y \in \Ycal} \sum_{k=1}^{T_{y}^l}  \sqrt{\ell( g_{l}^{s}(a_k^{l,y}),y)^{2}\PP_{}(Z_{l,y}^s=a_{k}^{l,y})} 
\end{align*}
Then, we invoke Lemma  \ref{lemma:10} with $v_k = v_k^{(y)}=\ell( g_{l}^{s}(a_k^{l,y}),y)^{2}\PP_{}(Z_{l,y}^s=a_{k}^{l,y})$, where we define $v_k^{(y)}=v_k(y)$. This implies that
$$
\sum_{k=1}^{T_{y}^l} \sqrt{\ell( g_{l}^{s}(a_k^{l,y}),y)^{2}\PP_{}(Z_{l,y}^s=a_{k}^{l,y})} \le \sum_{k=1}^{\ceil{\tbeta_y}} \sqrt{v_k^{(y)}} + \frac{C_{y}\tbeta_{y}}{\alpha_{y} e}, 
$$
where   $\tbeta_{y}= 2^{1/{\alpha_y}}\beta_{y}$.
Thus, 
$$
G_3^l \le \sqrt{2|\Ycal|}\max_{y \in \Ycal} \left(\sum_{k=1}^{\ceil{\tbeta_y}} \sqrt{v_k^{(y)}} + \frac{C _{y}\tbeta_y}{\alpha _{y}e}\right) 
$$
\end{proof}

\subsection{Proof of Proposition \ref{prop:new:1}} \label{app:5}

\begin{proof}
Let $l \in \{1,2,\dots,D+1\}$ and let us write $\lambda=\lambda_l$ and $C_\lambda=C_{\lambda_{l},l}$. We first note that  the value of $C_\lambda$  is always bounded as
\begin{align}
    C_\lambda\leq \sum_{q\in \Mcal_l}(\PP(\phi_{l}^{{S}}=q))^{1-\lambda}
    \leq |\Mcal_l|\left(\frac{1}{|\Mcal_l|}\sum_{q\in \Mcal_l} \PP(\phi_{l}^{{S}}=q)\right)^{1-\lambda}=|\Mcal_l|^\lambda,
\end{align}which is a very loose
bound and we will provide tighter bounds in below.
Before proceeding to the proof, we recall the following bounds. For $a>1$ we have
\begin{align}
   &\frac{1}{a-1}\leq \int_1^\infty \frac{d x}{x^a} \leq \sum_{i=1}^\infty \frac{1}{i^a}\leq 1+\int_1^\infty \frac{d x}{x^a}=\frac{a}{a-1},\\
  & \sum_{i=1}^\infty\frac{\ln(i)}{i^a}\leq \frac{\ln(2)}{2^a}+\frac{\ln(3)}{3^a}+\int_3^\infty \frac{\ln(x)d x}{x^a}
  =\frac{\ln(2)}{2^a}+\frac{\ln(3)}{3^a}+\frac{3^{1-a}((a-1)\ln(3)+1)}{(a-1)^2},\\
  & \sum_{i=1}^\infty \frac{\ln(i)}{i^a}
  \geq \frac{\ln(2)}{2^a}+\int_3^\infty \frac{\ln(x)d x}{x^a}\geq\frac{\ln(2)}{2^a}+\frac{3^{1-a}((a-1)\ln(3)+1)}{(a-1)^2}.
\end{align}
For $a<1$, we have
\begin{align}
    \frac{N^{1-a}-1}{1-a}=\int_1^N \frac{d x}{x^a}\leq \sum^{N}_{i=1}\frac{1}{i^a} \leq 1+\int_1^N \frac{d x}{x^a}=\frac{N^{1-a}-a}{1-a}\leq \frac{N^{1-a}}{1-a}.
\end{align}
In the first case, we have 
\begin{align}
    C_\lambda\leq \sum_{i=1}^Np_i^{1-\lambda}
    \leq \sum_{i=1}^N\frac{C^{1-\lambda}}{i^{\alpha(1-\lambda)}}\leq C^{1-\lambda}\frac{\alpha(1-\lambda)}{\alpha(1-\lambda)-1},
\end{align}
which is bounded and independent of $N$.
For the entropy, we notice that on $[0,1]$ the function $-p\ln p$ is non-negative, increasing on $[0,1/e]$ and decreasing on $[1/e,1]$.
\begin{align}
    H(A_s)&=\sum_{i=1}^N -p_i\ln(p_i)
    \leq \sum_{p_i>1/e}\frac{1}{e}+\sum_{p_i<1/e}\frac{C}{i^\alpha}\ln \frac{i^\alpha}{C}\\
    &\leq 1+\sum_{i\geq 1}\frac{C \alpha \ln i}{i^\alpha}
    \leq 1+C\alpha\left(\frac{\ln(2)}{2^\alpha}+\frac{\ln(3)}{3^\alpha}+\frac{3^{1-\alpha}((\alpha-1)\ln(3)+1)}{(\alpha-1)^2}\right).
\end{align}

In the second case, the normalization constant $Z$ diverges with $N$,
\begin{align}\label{e:upbb}
    Z=\sum_{i=1}^N \frac{c_i}{i^\alpha}\leq \sum_{i=1}^N \frac{C}{i^\alpha}\leq C\left(1+\int_1^N \frac{d x}{x^\alpha}\right)\leq C\left(\frac{N^{1-\alpha}}{1-\alpha}\right) .
\end{align}
And using $c_i\geq c$, we have a lower bound for $Z$
\begin{align}
    Z\geq c\left(\int_1^N \frac{d x}{x^\alpha}\right)\geq \frac{c (N^{1-\alpha}-1)}{1-\alpha}.
\end{align}
Thus $Z$ is of order $N^{1-\alpha}$, i.e. $Z=\Omega(N^{1-\alpha})$. 

We recall the formula of $C_\lambda$ from \eqref{e:Cla}
\begin{align}\label{e:lnCla}
    \ln C_\lambda=\ln\left(\sum_{i=1}^N p_i^{1-\lambda}\right)-\lambda H(A_s).
\end{align}
For the first term on the righthand side of \eqref{e:lnCla}, we have
\begin{align}\begin{split}\label{e:first}
    \ln\left(\sum_{i=1}^N p_i^{1-\lambda}\right)
    &=\ln\left(\sum_{i=1}^N \left(\frac{c_i}{Z i^\alpha}\right)^{1-\lambda}\right)
    =-(1-\lambda)\ln(Z)+\ln\left(\sum_{i=1}^N \frac{c_i^{1-\lambda}}{i^{(1-\lambda)\alpha}}\right)\\
    &=-(1-\lambda)\ln(N^{1-\alpha})+\ln(N^{1-(1-\lambda)\alpha})+\mathcal E_0
    =\lambda \ln N+\mathcal E_0,
\end{split}\end{align}
where 
\begin{align}
     \left|\mathcal E_0-\left(\ln(1-(1-\lambda)\alpha)-(1-\lambda)\ln(1-\alpha)\right)\right|\leq (1-\lambda)\ln(C/c)
\end{align}

Next we compute the entropy $H(A_s)$ and show it diverges as $\ln N$.
\begin{align}\begin{split}\label{e:HS}
    H(A_s)&=-\sum_{i=1}^N p_i \ln (p_i)
    =-\sum_{i=1}^N p_i \ln \frac{c_i}{i^\alpha Z}
    =\alpha\sum_{i=1}^N p_i\ln(i)+\ln(Z)-\sum_{i=1}^N p_i\ln(c_i)\\
    &=\alpha\sum_{i=1}^N p_i\ln(i)+(1-\alpha)\ln N+\mathcal E_1,
\end{split}\end{align}
where 
\begin{align}
   \ln (c/C)-\ln(1-\alpha)\leq  \mathcal E_1\leq \ln(C/c)-\ln(1-\alpha)
\end{align}
To compute the first term on the righthand side of \eqref{e:HS}, we introduce 
\begin{align}
    s_i=\frac{c_1}{1^\alpha}+\frac{c_2}{2^\alpha}+\cdots+\frac{c_i}{i^\alpha},\quad 1\leq i\leq N,
\end{align}
then $s_N=Z$. Next we can do a summation by part
\begin{align}\begin{split}\label{e:pi}
    \sum_{i=1}^N p_i\ln(i)
    &=\frac{1}{Z}\sum_{i=1}^N\frac{c_i}{i^\alpha} \ln(i)
    =\frac{1}{Z}\sum_{i=1}^{N}(s_{i}-s_{i-1}) \ln(i)\\
    &=\frac{1}{Z}\sum_{i=1}^{N-1}s_{i}( \ln(i)-\ln(i+1))+\frac{s_N \ln N}{Z}\\
    &=\frac{1}{Z}\sum_{i=1}^{N-1}s_{i}( \ln(i)-\ln(i+1))+\ln N
\end{split}\end{align}
The same as in \eqref{e:upbb}, we have $|s_i|\leq Ci^{1-\alpha}/(1-\alpha)$. Moreover $\ln(1+1/i)\leq 1/i$. Thus the first term on the righthand side of \eqref{e:pi} can be bounded as
\begin{align}\label{e:sum}
    \left|\frac{1}{Z}\sum_{i=1}^{N-1}s_{i}( \ln(i)-\ln(i+1))\right|
    \leq \frac{1}{Z}\sum_{i=1}^N C\left(\frac{i^{1-\alpha}}{1-\alpha}\right)\frac{1}{i}
    \leq  \frac{C}{c(1-\alpha)}
\end{align}
By plugging \eqref{e:pi} and \eqref{e:sum} into \eqref{e:HS}, we conclude the following bound on the entropy
\begin{align}\label{e:entropyexp}
    H(A_s)=\ln(N)+\mathcal E_2.
\end{align}
where 
\begin{align}
    \ln(c/C)-\ln(1-\alpha)-\frac{C}{c(1-\alpha)}\leq \mathcal E_2\leq \ln(C/c)-\ln(1-\alpha)+\frac{C}{c(1-\alpha)}
\end{align}
The two estimates \eqref{e:first} and \eqref{e:entropyexp} together imply that $C_\lambda=\mathcal E_0-\lambda \mathcal E_2$, and
\begin{align}
    \left|C_\lambda-\left(\ln(1-(1-\lambda)\alpha)-(1-2\lambda)\ln(1-\alpha)\right)\right|\leq (2-\lambda)\ln(C/c)+\frac{C}{c(1-\alpha)}
\end{align}
\end{proof}

\end{document}